\documentclass[nohyperref]{article}

\usepackage{microtype}
\usepackage{graphicx}
\usepackage{booktabs} 

\usepackage{nicefrac}       
\usepackage{adjustbox}
\usepackage{subfig}
\usepackage{amsmath, amsthm, amssymb}
\usepackage{enumitem}

\usepackage{graphbox,graphicx}

\theoremstyle{plain}
\newtheorem{thm}{\protect\theoremname}
\theoremstyle{definition}

\theoremstyle{plain}
\newtheorem{prop}[thm]{\protect\propositionname}
\theoremstyle{plain}
\newtheorem{conjecture}[thm]{\protect\conjecturename}

\newtheorem{lem}[thm]{\protect\lemmaname}
\theoremstyle{plain}
\newtheorem{assumption}[thm]{\protect\assumptionname}
\newtheorem{cor}[thm]{\protect\corollaryname}
\theoremstyle{plain}
\theoremstyle{remark}
\newtheorem{rem}[thm]{\protect\remarkname}

\providecommand{\lemmaname}{Lemma}
\providecommand{\assumptionname}{Assumption}
\providecommand{\conjecturename}{Conjecture}
\providecommand{\definitionname}{Definition}
\providecommand{\propositionname}{Proposition}
\providecommand{\theoremname}{Theorem}
\providecommand{\corollaryname}{Corollary}
\providecommand{\remarkname}{Remark}

\usepackage{hyperref}

\usepackage[accepted]{icml2022}

\usepackage[textsize=tiny]{todonotes}

\icmltitlerunning{Deep Linear Networks Dynamics}

\begin{document}

\twocolumn[
\icmltitle{Saddle-to-Saddle Dynamics in Deep Linear Networks: \\ Small Initialization Training, Symmetry, and Sparsity}


\icmlsetsymbol{equal}{*}

\begin{icmlauthorlist}
\icmlauthor{Arthur Jacot}{epfl}
\icmlauthor{Fran\c{c}ois Ged}{epfl}
\icmlauthor{Berfin \c{S}im\c{s}ek}{epfl}
\icmlauthor{Cl\'ement Hongler}{epfl}
\icmlauthor{Franck Gabriel}{epfl}
\end{icmlauthorlist}

\icmlaffiliation{epfl}{EPFL,Switzerland}

\icmlcorrespondingauthor{Arthur Jacot}{arthur.jacot@netopera.net}

\icmlkeywords{Machine Learning, ICML}

\vskip 0.3in
]



\printAffiliationsAndNotice{\icmlEqualContribution} 

\begin{abstract}
The dynamics of Deep Linear Networks (DLNs) is dramatically affected by the variance $\sigma^2$ of the parameters at initialization $\theta_0$. For DLNs of width $w$, we show a phase transition w.r.t. the scaling $\gamma$ of the variance $\sigma^2=w^{-\gamma}$ as $w\to\infty$: for large variance ($\gamma<1$), $\theta_0$ is very close to a global minimum but far from any saddle point, and for small variance ($\gamma>1$), $\theta_0$ is close to a saddle point and far from any global minimum. While the first case corresponds to the well-studied NTK regime, the second case is less understood. This motivates the study of  the case $\gamma \to +\infty$, where we conjecture a Saddle-to-Saddle dynamics: throughout training, gradient descent visits the neighborhoods of a sequence of saddles, each corresponding to linear maps of increasing rank, until reaching a sparse global minimum. We support this conjecture with a theorem for the dynamics between the first two saddles, as well as some numerical experiments.
\end{abstract}

\section{Introduction}

In spite of their widespread usage, the theoretical understanding
of Deep Neural Networks (DNNs) remains limited.
In contrast to more common statistical methods which are built (and proven) to recover the specific structure of the data, the development of DNNs techniques has been mostly driven by empirical results.
This has led to a great variety of models which perform consistently
well, but without a theory explaining why. In this paper, we provide a theoretical analysis of Deep Linear (Neural) Networks (DLNs), whose simplicity makes them particularly attractive as a first step towards the development of such a theory.

DLNs have a non-convex loss landscape and the behavior of training dynamics can
be subtle. For shallow networks, the convergence of gradient descent is guaranteed
by the fact that the saddles are strict and that all minima are global \cite{baldi1989neural, Kawaguchi_2016_no_local_min,lee_2016_conv_to_min,lee_2019_strict_saddle}.
In contrast, the deep case features non-strict saddles \cite{Kawaguchi_2016_no_local_min} and no general proof of convergence exists at the moment, though convergence to a global minimum can be guaranteed in some cases \cite{arora_2018_DLN_convergence, Eftekhari2020}.

A recent line of work focuses on the implicit bias of DLNs,
and consistently reveals some form of incremental learning and implicit sparsity as in \cite{Gissin2020}. Diagonal
networks are known to learn minimal $L_{1}$ solutions \cite{Moroshko_2020_implicit_bias_diag,woodworth_2020_diagnet_bias}.
With a specific initialization and the MSE loss, DLNs learn the singular
components of the signal one by one \cite{saxe_2014_exact,advani_2017_independent_diag,Saxe_2019_independent_diag, gidel_2019_independent_diag, arora_2019_matrix_factorization}.
Recently, it has been shown that with losses such as the cross-entropy and
the exponential loss, the parameters diverge towards infinity, but end up following the
direction of the max-margin classifier w.r.t. the $L_{p}$-Schatten (quasi)norm
\cite{gunasekar_2018_implicit_bias,gunasekar_2018_implicit_bias2,soudry2018implicit,Ji_2018_directional,Ji_2020_directional2,chizat_2020_implicit_bias,Lyu_2020_directional1,Moroshko_2020_implicit_bias_diag,yun2021_implicit_bias}.

In parallel, recent works have shown the existence of two regimes in large-width
DNNs: a kernel regime (also called NTK or lazy regime) where learning is described by the so-called Neural Tangent Kernel (NTK)  guaranteeing linear convergence \cite{jacot2018neural,Du2019,chizat2018note,exact_arora2019,lee2019wide,huang2019NTH},
and an active regime where the dynamics is nonlinear \cite{Chizat2018,Rotskoff2018,Mei2018,mei2019mean, chizat_2020_implicit_bias}. For DLNs, both regimes can be observed as well, with evidence that while the linear regime exhibits no sparsity, the active regime favors solutions with some kind of sparsity \cite{woodworth_2020_diagnet_bias,Moroshko_2020_implicit_bias_diag}.


\subsection{Contributions}
We study deep linear networks $x\mapsto A_\theta x$ of depth $L\geq 1$ and widths $n_0,\cdots,n_L$, that is $A_\theta= W_L \cdots W_1$ where $W_1,
\ldots, W_L$ are matrices such that $W_i\in\mathbb{R}^{n_i\times n_{i-1}}$ and $\theta$ is a vector that consists of all the (learnable) parameters of the DLN, i.e. the components of the matrices $W_1,\cdots,W_L$. For any general convex cost $C:\mathbb{R}^{n_L\times n_0}\to\mathbb{R}$ on matrices such that the zero matrix is not a global minimum, we investigate the gradient flow minimizing the loss $\mathcal L(\theta) = C(A_\theta)$.
To ease the notation, suppose that the hidden layers have the same size, that is $w=n_1=\cdots=n_{L-1}$ for some $w\in\mathbb{N}$.

The variance of the parameters at initialization has a profound effect on the training dynamics. If the parameters are initialized with variance $\sigma^2 = w^{-\gamma}$, where $w$ is the size of the hidden layers, we observe a phase transition in the infinite width limit as $w\to\infty$ and show in Theorem \ref{th:distances} that: \begin{itemize}
\item when $\gamma<1$, the random initialization $\theta_0$ is (with high probability) very close to a global minimum and very far from any saddle,
\item when $\gamma>1$, the initialization is very close to a saddle and far from any global minimum.
\end{itemize}

The case $\gamma<1$ corresponds to the NTK regime (or kernel/lazy regime, described in Section \ref{sec:NTK}) and the case $\gamma=1$ corresponds to the Mean-Field limit (or the Maximal Update parametrization of \cite{Greg_Yang_2019}). It appears that the case $\gamma>1$ has been much less studied in previous works.

To understand this regime, we investigate in Section \ref{sec:saddle} the case $\gamma\to+\infty$. More precisely, we fix the width of the network and let the variance at initialization go to zero. We show in Theorem \ref{thm:first-path} that the gradient flow trajectory asymptotically goes from the saddle at the origin $\vartheta^{0}=0$  to a rank-one saddle $\vartheta^1$, i.e. a saddle where the matrices $W_1,\ldots,W_L$ are of rank $1$.  The proof is based on a new description (Theorem \ref{thm:bijection-fast-escape-paths}), in the spirit of the  Hartman-Grobman theorem, of the so-called fast escape paths at the origin. This theorem may be of independent interest.

We propose the Conjecture \ref{conj:Conj-Saddle-Saddle}, backed by numerical experiments, describing the full gradient flow when the variance at initialization is very small, suggesting that it goes from saddle to saddle, visiting the neighborhoods of a sequence of critical points  $\vartheta^0, \ldots, \vartheta^{K}$  (the first $K$ ones being saddle points, the last one being either a global minimum or a point at infinity) corresponding to matrices of increasing ranks. This is consistent with \cite{Gissin2020} which shows that incremental learning occurs in a toy model of DLNs and that gradient-based optimization hence has an implicit bias towards simple (sparse) solutions.

In Section \ref{sec:greedyalgo}, we show how this Saddle-to-Saddle dynamics can be described using a greedy low-rank algorithm which bears similarities with that of \cite{li2020towards} and leads to a low-rank bias of the final learned function. This is in stark contrast to the NTK regime which features no low-rank bias.

\subsection{Related Works}
The existence of distinct regimes in the training dynamics of DNNs has been explored in previous works, both theoretically \cite{chizat2018note,Greg_Yang_2019} and empirically \cite{geiger2019disentangling}. The theoretical works \cite{chizat2018note,Greg_Yang_2019} have mostly focused on the transition from the NTK regime ($\gamma<1$) to the Mean-Field regime ($\gamma=1$). This paper is focused on the regime beyond the critical one ($\gamma>1$).

Our study of the Saddle-to-Saddle dynamics can also be understood as a generalization of the works  \cite{saxe_2014_exact,advani_2017_independent_diag,Saxe_2019_independent_diag, gidel_2019_independent_diag, arora_2019_matrix_factorization} which describe a similar plateau effect in a very specific setting and with a very carefully chosen initialization.

Shortly after the initial publication of this article, we came aware of the paper \cite{li2020towards} which provides a similar description to our Saddle-to-Saddle dynamics. For shallow networks, the results are almost equivalent, although the techniques are very different, especially when dealing with the fact that the escape directions (and escape paths) are unique only up to rotations. The paper \cite{li2020towards} uses a clever trick that allows them to both study the dynamics of the output matrix $A_{\theta(t)}$, without the need to keep track of the parameters,  and obtain a unicity property for the asymptotic dynamics. Instead, we focus on the dynamics of the parameters, give an identification of all optimal escape paths, and show that the path followed by the parameters' dynamics is unique up to symmetries of the network. Note also that, as in our paper, \cite{li2020towards} only proves the first step of the Saddle-to-Saddle regime: for the subsequent steps, it is assumed that the next saddle is not approached along a \`bad' direction (as we discuss in Section \ref{sec:sketch-of-proof}). For deep networks, our results are more general as they hold for more general initializations than in \cite{li2020towards}. Indeed, in order to avoid the non-uniqueness problem of the escape paths in the space of parameters, their analysis relies heavily on the assumption that the weights of the network are balanced at initialization, and thus during training. Because we do not rely on this trick, our analysis does not require a balanced initialization.

\section{Deep Linear Networks}
\subsection{Setup}
A DLN of depth $L$ and widths $n_{0},\dots,n_{L}$
is the composition of $L$ matrices
\[
A_{\theta}=W_{L}\cdots W_{1}
\]
where $W_{\ell}\in\mathbb{R}^{n_{\ell}\times n_{\ell-1}}$. The number of parameters is $P=\sum_{\ell=1}^{L}n_{\ell-1}n_{\ell}$ and we denote by $\theta=(W_{L},\dots,W_{1})\in\mathbb{R}^{P}$ the vector of parameters. The input dimension, resp. the output dimension is $n_{0}$, resp. $n_L$. All parameters are initialized
as i.i.d. $\mathcal{N}(0,\sigma^{2})$ Gaussian random variables.

We will focus on the so-called rectangular networks, in which the number
of neurons in all hidden layers is the same, i.e. $n_{1}=\dots=n_{L-1}=w$. Such rectangular network is called a $(L,w)$-DLN, and its number of parameters is denoted by $P=P_{(L,w)}=n_{0}w+(L-2)w^{2}+wn_{L}$. The proofs given in this article can be extended to the non-rectangular case, but this leads to more complex notations.

We study the dynamics of gradient descent on the loss $\mathcal L(\theta)=C(A_{\theta})$ for
a general differentiable and convex cost $C$ on $n_{L}\times n_{0}$ matrices. To ensure a non-trivial minimisation problem, we assume that the null matrix is not a global minimum of $C$ : in this case, the origin in the parameter space is a saddle of $\mathcal L$. Given a starting point $\theta_{0}\in\mathbb{\mathbb{R}}^{P}$,
we denote by $t\mapsto\Gamma(t,\theta_{0})$ the gradient flow path
on the cost $\mathcal L(\theta)$ starting from $\theta_{0}$, i.e. $\Gamma(0,\theta_0)=\theta_0$ and $\partial_t \Gamma(t,\theta_0)=-\nabla \mathcal L(\Gamma(t,\theta_0))$.

While our analysis applies to general  twice differentiable costs $C$, the typical costs used in practice are:

The \emph{Mean-Squared Error} (MSE) loss $C(A)=\frac{1}{N}\left\Vert AX-Y\right\Vert _{F}^{2}$
for some inputs $X\in\mathbb{R}^{n_{0}\times N}$ and labels $Y\in\mathbb{R}^{n_{L}\times N}$,
where $||\cdot||_{F}$ is the Frobenius norm.

The \emph{Matrix Completion} (MC) loss $C(A)=\frac{1}{N}\sum_{i=1}^{N}(A_{k_{i},m_{i}}-A_{k_{i},m_{i}}^{*})^{2}$
for some true matrix $A^{*}$ of which we observe only the $N$ entries
$A_{k_{1},m_{1}}^{*},\dots,A_{k_{N},m_{N}}^{*}$.

\begin{figure*}[!t]
  \centering
  \subfloat{
      \includegraphics[height=0.27\textwidth,align=t]{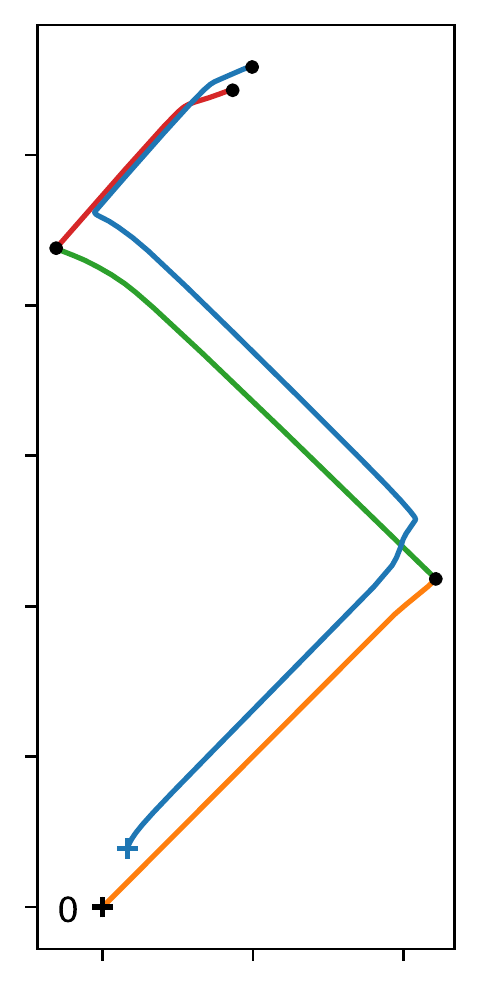}
      }
  \subfloat{
      \includegraphics[height=0.29\textwidth,align=t]{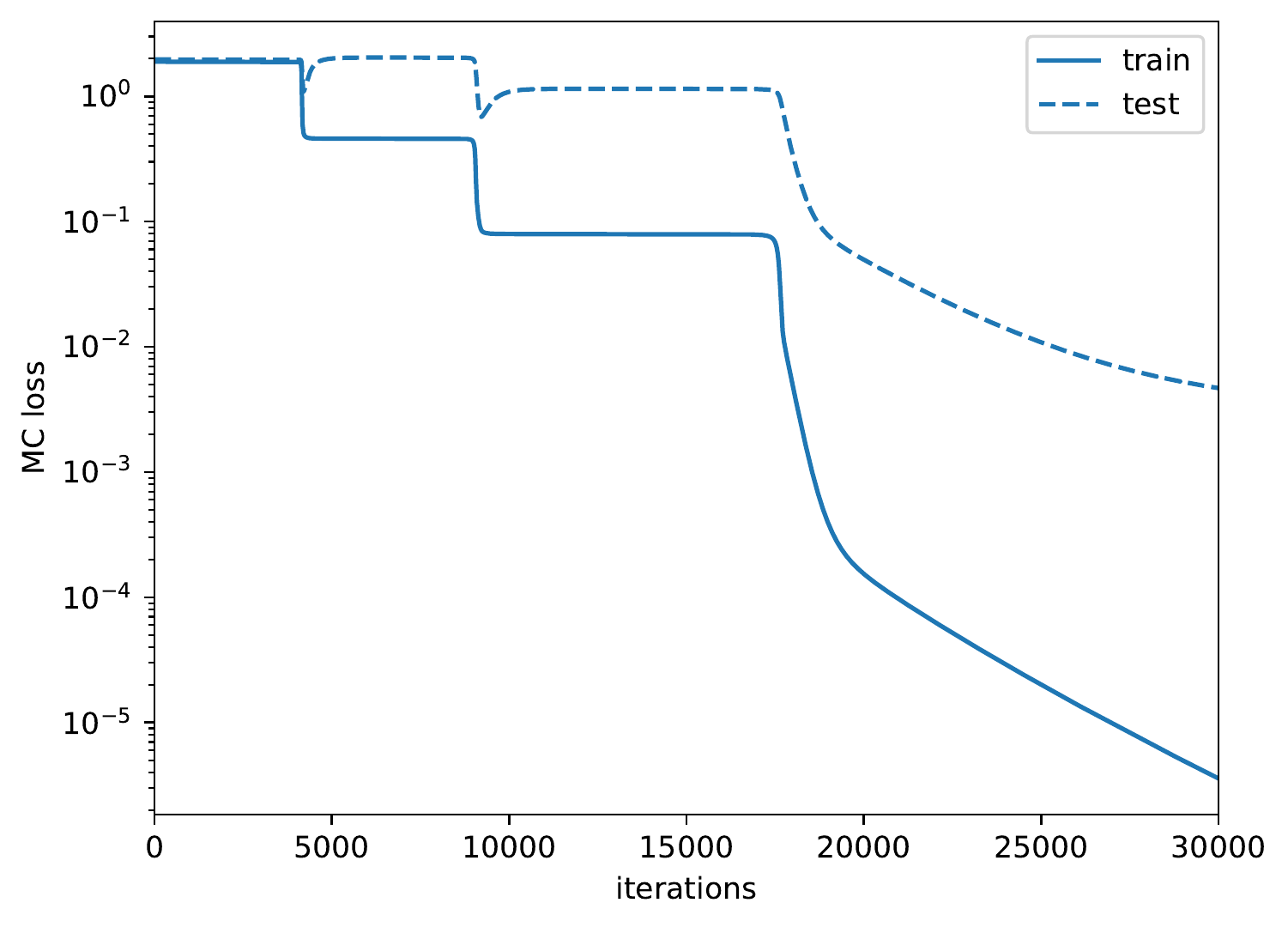}
      }
  \subfloat{
      \includegraphics[height=0.275\textwidth,align=t]{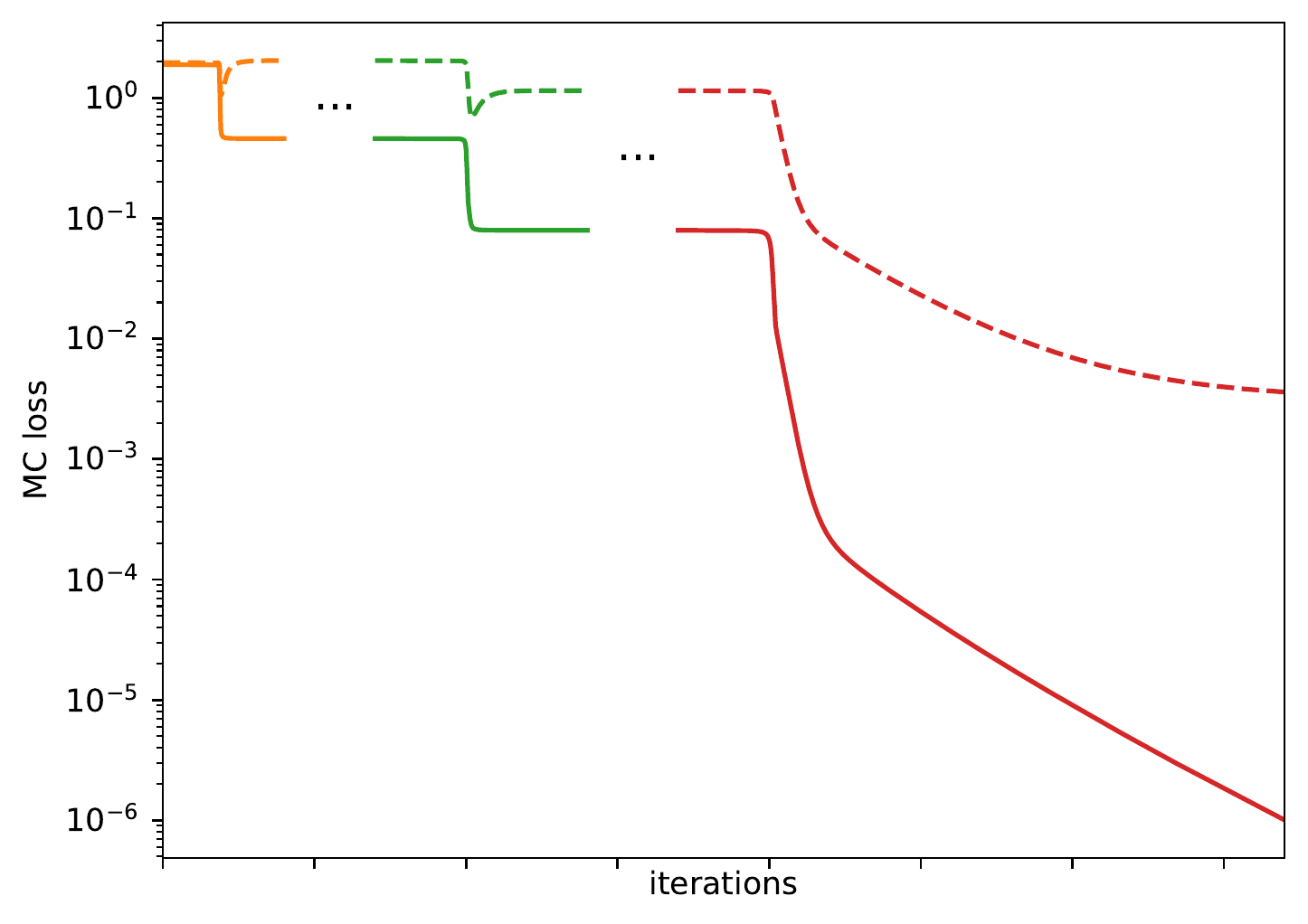}
      }
  \caption{\textit{Saddle-to-Saddle dynamics:}
  A DLN ($L=4,w=100$)  with a small initialization ($\gamma=2$) trained on a MC loss fitting a $10\times 10$ matrix of rank $3$. \textbf{Left:} Projection onto a plane of the gradient flow path $\theta_\alpha$ in parameter space (in blue) and of the sequence of 3 paths $\theta^1,\theta^2,\theta^3$ (in orange, green and red), described by Algorithm $\mathcal{A}_{\epsilon,T,\eta}$, starting from the origin (+) and passing through 2 saddles ($\cdot$) before converging. \textbf{Middle:} Train (solid) and test (dashed) MC costs through training. We observe three plateaus, corresponding to the three saddles visited. \textbf{Right:} The train (solid) and test (dashed) losses of the three paths  plotted sequentially, in the saddle-to-saddle limit; the dots represent an infinite amount of steps separating these paths.}
  \label{fig:linear_saddle_to_saddle_MC}
\end{figure*}

\subsection{Symmetries and Invariance}

A key tool in this paper is the use of two important symmetries of the parametrization map $\theta\mapsto A_{\theta}$ in
DLNs: rotations of hidden layers and inclusions in wider DLNs.

\begin{description}
\item [{Rotations:}] A $L-1$ tuple $R=(O_{1},\dots,O_{L-1})$ of orthogonal
$w\times w$ matrices is called a $w$-width network rotation, or
in short a rotation. A rotation $R$ acts on a parameter vector $\theta=(W_{L},\ldots,W_{1})$
as $R\theta=(W_{L}O_{L-1}^{T},O_{L-1}W_{L-1}O_{L-2}^{T},\dots,O_{1}W_{1})$.
The space of rotations is an important symmetry of DLN: indeed, for
any parameter $\theta,$ and any cost $C$, the two following important
properties hold:
\[
A_{R\theta}=A_{\theta},\quad\nabla_{\theta}C(A_{R\theta})=R\nabla_{\theta}C(A_{\theta}),
\]
where we considered $\nabla_{\theta}C(A_{\theta})\in \mathbb{R}^{P_{L,w}}$
as another vector of parameters.
These properties imply that if $\theta(t)=\Gamma(t, \theta_0)$ is a gradient flow path, then
so is $R\theta(t)=\Gamma(t,R\theta_0)$.
\item [{Inclusion:}] The inclusion $I^{(w\to w')}$ of a network of width
$w$ into a network of width $w'>w$ (by adding zero weights on the
new neurons) is defined as $I^{(w\to w')}(\theta)=(V_L,\dots,V_1)$ with
\[\arraycolsep=1.5pt
V_1 = \left(\begin{array}{c}
W_1 \\ 0
 \end{array}\right),
V_\ell = \left(\begin{array}{cc}
W_{\ell} & 0\\
0 & 0
\end{array}\right),
V_L = \left(\begin{array}{cc}
W_{L} & 0\end{array}\right).
\]
For any parameters $\theta$ and any cost $C$, we have $A_{I^{(w\to w')}\theta}=A_{\theta}$
and $\nabla C(A_{I^{(w\to w')}\theta})=I^{(w\to w')}\nabla C(A_{\theta})$:
the image of the inclusion map $I^{(w\to w')}$ (as well as any rotation
$R\mathrm{Im}I^{(w\to w')}$ thereof) is invariant under gradient
flow.
\end{description}

\section{Proximity of Critical Points at Initialization}

It has already been observed that in the infinite width limit, when
the width $w$ of the network grows to infinity, the scale at which
the variance $\sigma^{2}$ of the parameters at initialization scales
with the width can lead to very different behaviors \cite{chizat2018note,geiger2019disentangling,Greg_Yang_2019}.
Let us consider scaling of the variance $\sigma^{2}=w^{-\gamma}$
for $\gamma\geq1-\frac{1}{L}$. The reason we lower bound $\gamma$
is that any smaller $\gamma$ would lead to an explosion of the variance
of the matrix $A_{\theta}$ at initialization as the width $w$ grows.

Let $d_{\mathrm{m}}$ and $d_{\mathrm{s}}$ be the Euclidean distances between the initialization
$\theta$ and, respectively, the set of global minima and  the set of all saddles. For random variables $f(w),g(w)$ which depend on $w$, we write $f\asymp g$ if both $\nicefrac{f(w)}{g(w)}$ and $\nicefrac{g(w)}{f(w)}$ are stochastically bounded as $w\to\infty$.
The following theorem studies how $d_{\mathrm{m}}$ and $d_{\mathrm{s}}$  scale as $w\to\infty$:
\begin{thm}\label{th:distances}
Suppose that the set of matrices that minimize $C$ is non-empty, has Lebesgue measure zero, and does not contain the zero matrix. Let $\theta$ be i.i.d. centered Gaussian r.v. of variance
$\sigma^{2}=w^{-\gamma}$ where $1-\frac{1}{L}\leq\gamma<\infty$. Then:
\begin{enumerate}
\item if $1-\frac{1}{L}\leq\gamma<1$, we have $d_{\mathrm{m}}\asymp w^{-\frac{(1-\gamma)(L-1)}{2}}$
and $d_{\mathrm{s}}\asymp w^{\frac{1-\gamma}{2}}$,
\item if $\gamma=1$, we have $d_{\mathrm{m}},d_{\mathrm{s}} \asymp 1$,
\item if $\gamma>1$ we have $d_{\mathrm{m}}\asymp 1$ and $d_{\mathrm{s}}\asymp w^{-\frac{\gamma-1}{2}}$.
\end{enumerate}
\end{thm}
This theorem shows an important change of behavior between the
case $\gamma<1$ and $\gamma>1$. When $\gamma<1$, the network is
initialized very close to a global minimum and far from any saddle. When $\gamma>1$,
the parameters are initialized very close to a saddle but far
away from any global minimum. The critical case
$\gamma =1$ is the unique limit where both types of critical points are
at the same distance from the initialization.

Hence, the landscape of the loss near the initialization displays distinct features  in the three regimes highlighted in the previous theorem. In fact, the dynamics of the gradient descent also exhibits very distinctive characteristics in the different regimes. In Appendix \ref{subsec:equivalence-parametrization-NTK}, we show that the largest initialization, corresponding to the choice $\gamma=1-\frac{1}{L}$, is equivalent to the so-called NTK parametrization of \cite{jacot2018neural},
up to a rescaling of the learning rate. In the range $1-\frac{1}{L} < \gamma < 1$,  \cite{yang2020feature} obtain a similar, yet slightly different, kernel regime. The initialization $\gamma=1$
corresponds to the Mean-Field limit for shallow networks \cite{Chizat2018,Rotskoff2018}
or, more generally, to the Maximal Update parametrization \cite{yang2020feature} (see Appendix \ref{subsec:equivalence-parametrization-MUP}). The case $\gamma>1$ is however much less studied and is difficult to study since the initialization approaches
a saddle as $w\to\infty$. Thus, in this regime, the wider the network, the longer it takes to escape this nearby saddle and, in the limit as $w\to \infty$, nothing happens over a finite number of gradient steps. With the right time parametrization, we will observe interesting Saddle-to-Saddle dynamics in this regime, leading to some low-rank bias.

\section{NTK regime: $\gamma <1$} \label{sec:NTK}
The NTK for linear networks can be expressed easily using the tensor
\[
\Theta^{(L)}= \sum_{\theta} \partial_\theta A \otimes   \partial_\theta A,
\]
which entries are given by $[\Theta^{(L)}]_{i,k}^{j,l} = (\nabla_{\theta} (A_\theta)_{i,j})^T (\nabla_{\theta} (A_\theta)_{k,l})$, for $i,k=1,\ldots,n_L$ and $j,l=1,\ldots,n_0$.
For any $x$, $y$ in $\mathbb{R}^{n_0}$, the value of the NTK at $x$ and $y$ is $\Theta^{(L)} (x\otimes y)=\sum_{j,l}[\Theta^{(L)}]_{\cdot,\cdot}^{j,l} x_j y_l$.

When the parameters evolve
according to the gradient flow on $\mathcal{L}(\theta)=C(A_{\theta})$, the dynamics of $A_{\theta(t)}$ is:
\begin{align*}
\partial_{t}A_{\theta(t)}&=-\Theta^{(L)}\cdot\nabla_{A}C(A_{\theta(t)}) \\&=  -\sum_{k,l}[\Theta^{(L)}]_{\cdot,k}^{\cdot ,l} \frac{d}{dA_{k,l}}C(A_{\theta(t)}),
\end{align*}
where $\cdot$ denotes a contraction of the $k,l$  indices of $\Theta^{(L)} $ with the two indices of  $\nabla_{A}C(A_{\theta(t)})$.

At initialization, $\Theta^{(L)}$ concentrates around its expectation $\mathbb{E}\left[\Theta^{(L)}\right]=Lw^{(1-\gamma)(L-1)} \delta_{i,k}\delta_{j,l}$ as the width grows. It was first proven in \cite{jacot2018neural}
that for an initialization equivalent to the case $\gamma=1-\frac{1}{L}$
(see Appendix \ref{subsec:equivalence-parametrization-NTK} for more details), as $w \to \infty$ the NTK remains constant during training. Recent
results \cite{yang2020feature} have shown that the NTK is asymptotically
fixed for all $\gamma\in(1-\frac{1}{L},1)$. In this case, given the asymptotic behavior of the NTK, the evolution of $A_{\theta(t)}$
is the same (up to a change of learning rate) as the one obtained by performing directly a gradient flow on the cost $C$.

As a result, in the regime $\gamma<1$, if the cost $C$ is strictly convex (or satisfies the Polyak-Lojasievicz inequality \cite{liu2020_NTK-PL-inequ}), the loss decays exponentially fast. Besides, the depth of the network has no effect in the infinite width limit (except for a change of learning rate) and the DLN structure adds no specific bias to the global minimum learned with gradient descent. In particular, this regime leads to no low-rank bias.

\begin{figure*}[!t]
  \centering
  \subfloat{
      \includegraphics[width=0.33\textwidth]{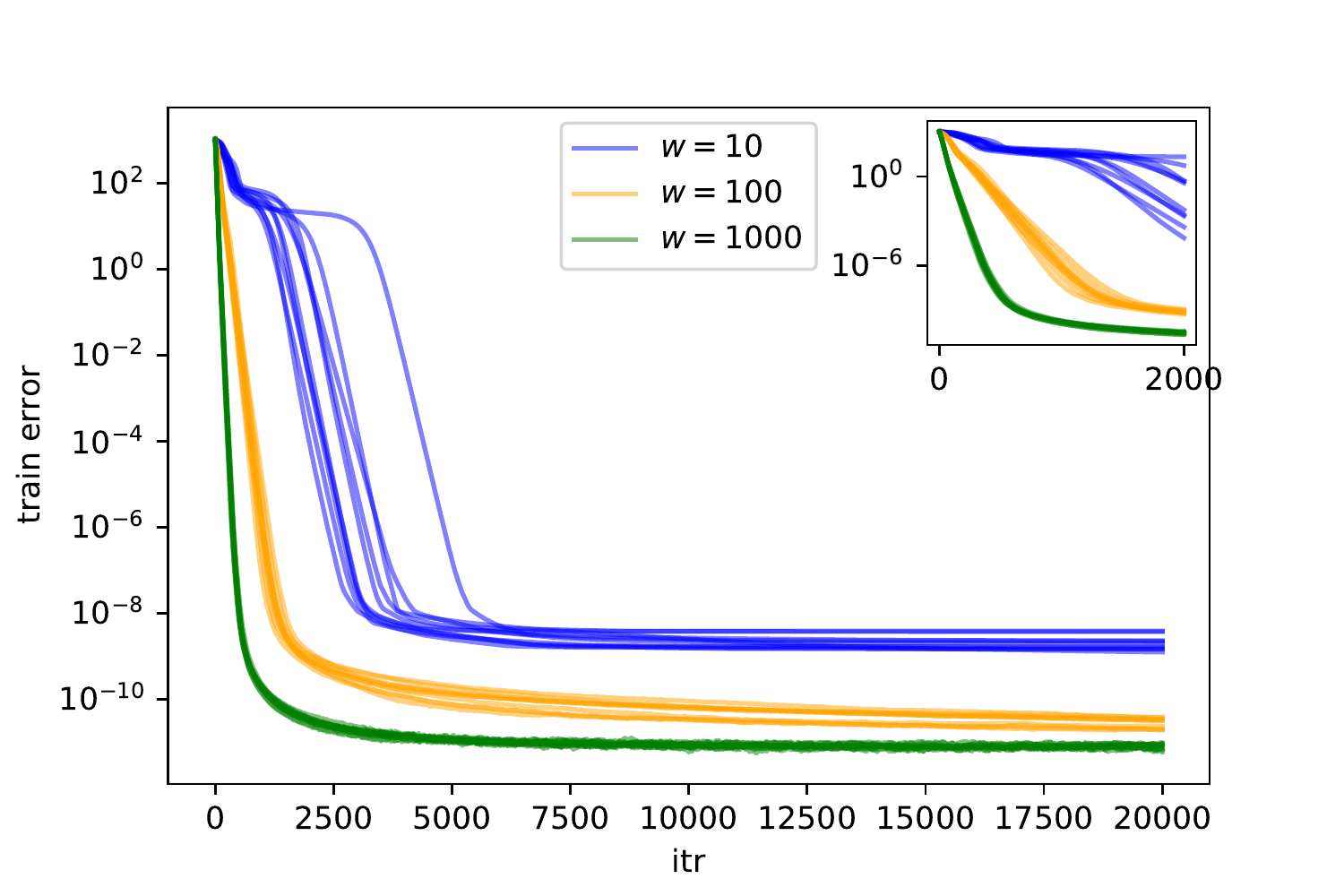}
  	\label{fig:regimes-of-training-NTK}
      } \hspace{-0.6cm}
  \subfloat{
      \includegraphics[width=0.33\textwidth]{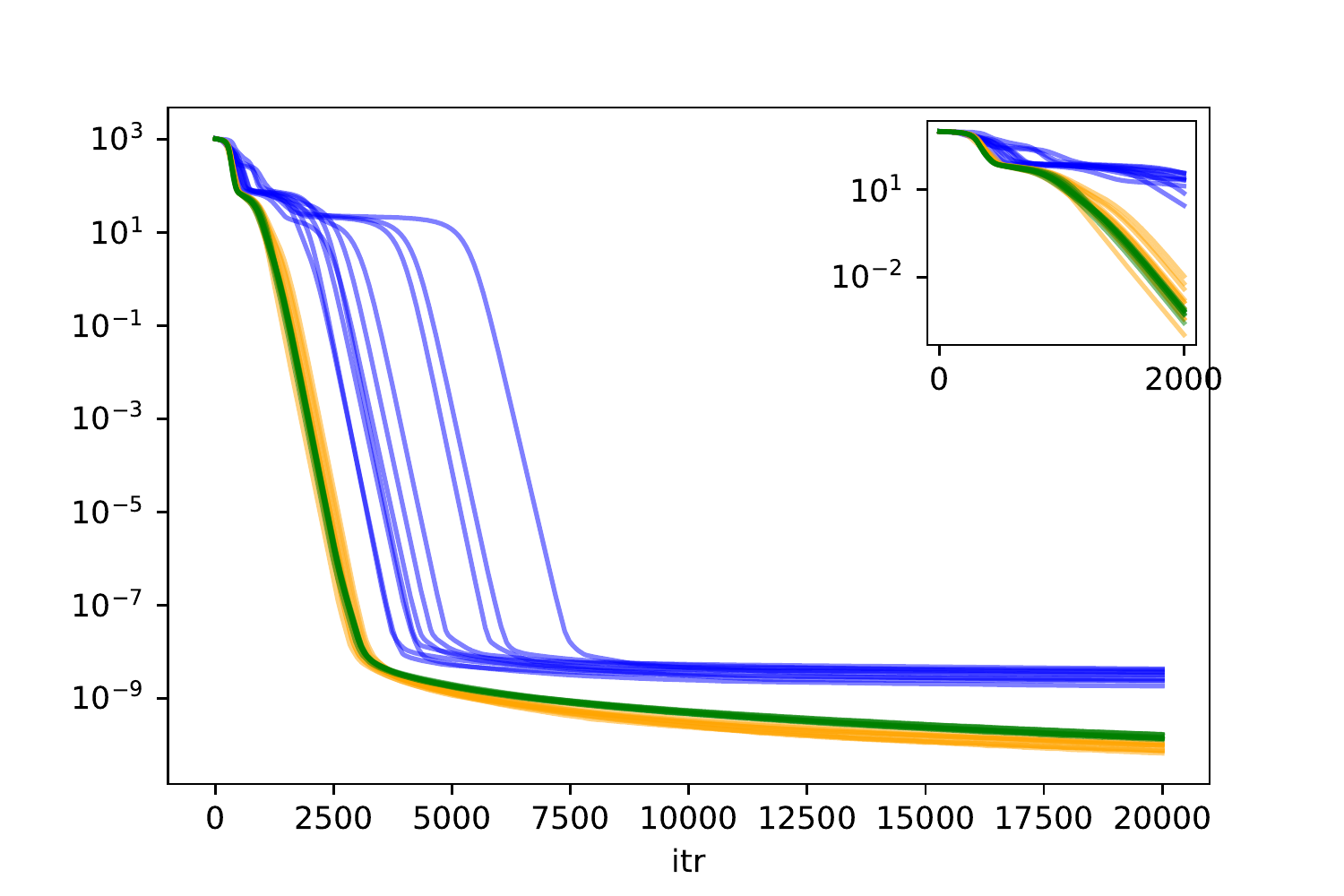}
  	\label{fig:regimes-of-training-MF}
      } \hspace{-0.6cm}
  \subfloat{
      \includegraphics[width=0.33\textwidth]{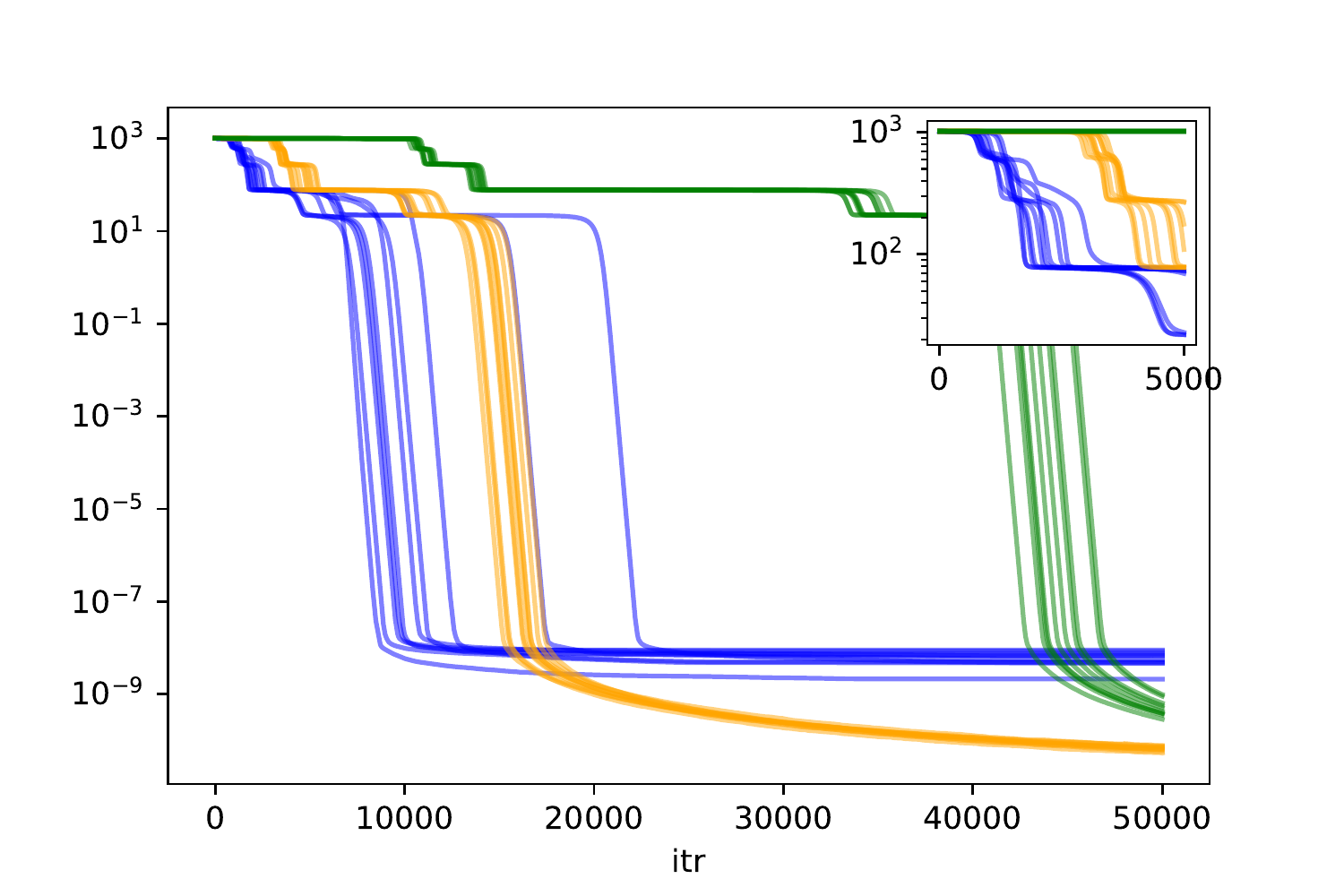}
  	\label{fig:regimes-of-training-StoS}
      } \\ 
  {\scriptsize \textbf{(a)} $\gamma =0.75$ (NTK) \hspace{3.5cm} \textbf{(b)} $\gamma =1$ (MF) \hspace{3.5cm} \textbf{(c)} $\gamma =1.5$ (S-S)}
  \caption{\textit{Training in \textbf{(a)} the NTK regime, \textbf{(b)} mean-field, \textbf{(c)} saddle-to-saddle regimes in deep linear networks for three widths $w=10,100,1000$, $L=4$, and $10$ seeds.}  Parameters are initialized with variance $\sigma^2 = w^{-\gamma}$.
  We observe that \textbf{(a)} in the NTK regime, the training loss shows typical linear convergence behavior for $w=1000$ and $w=100$; \textbf{(b)} in the mean-field regime, we observe that even the large width networks approach to a saddle at the beginning of the training and that the length of the plateaus remains constant between widths $w=1000$ and $w=100$; \textbf{(c)} in the saddle-to-saddle regime, the plateaus become longer as the width grows. In all cases, we see a reduction in the variation between the different seeds as $w \to \infty$.
  }
  \label{fig:regimes-of-training}
\end{figure*}

\section{Saddle-to-Saddle Dynamics: $\gamma \gg 1$}\label{sec:saddle}

We now study the dynamics of DLN during training as the variance at initialization goes to zero. Specifically, we sample some random parameters $\theta_{0}$
with i.i.d. $\mathcal{N}(0,1)$ entries, consider the gradient flow
$\theta_{\alpha}(t) = \Gamma (t,\alpha \theta_0) $, and let $\alpha\searrow0$.
Since the origin is a saddle, for all fixed times $t$, $\lim_{\alpha\searrow0}\theta_{\alpha}(t)=0$. We will show however that there is an escape
time $t_{\alpha}$, which grows to infinity as $\alpha\searrow0$, such
that the limit $\lim_{\alpha\searrow0}\theta_{\alpha}(t_{\alpha}+t)$
is non-trivial for all $t\in\mathbb{R}$.

The study of shallow networks ($L=2$) is facilitated by the fact that the saddle at the origin is strict: its Hessian has negative
eigenvalues. For deeper networks ($L>2$), the saddle is highly degenerate: the $L-1$ first order derivatives vanish. In Section \ref{sec:escapepath}, we develop new theoretical tools to analyze the two types of saddles and their escape paths.

\subsection{First Path}
%

It turns out that gradient flow paths naturally escape the saddle
at the origin along so-called \emph{optimal escape paths}. We
say that a gradient flow path $\theta(\cdot):\mathbb{R}\to\mathbb{R}^P$ is an \emph{escape path}
of a critical point $\theta^{*}$ if $\lim_{t\to-\infty}\theta(t)=\theta^{*}$.
Informally, the optimal escape paths, whose precise definition is given in Section \ref{sec:escapepath}, are the escape paths that allow the fastest exit from a saddle. In DLNs, these optimal escape paths are of the form $RI^{(1\to w)}\underline{\theta}^{1}(t)$
where $\underline{\theta}^{1}(t)$ is a path of a width 1 DLN which escapes from the origin:
\begin{thm}
\label{thm:first-path}Assume that the largest singular value $s_{1}$
of the gradient of $C$ at the origin $\nabla C(0)\in\mathbb{R}^{n_{L}\times n_{0}}$
has multiplicity 1. There is a deterministic gradient flow path $\underline{\theta}^{1}$ in the space
of width-$1$ DLNs such that, with probability $1$ if $L\leq 3$, and probability at least $\nicefrac{1}{2}$ if $L>3$, there exists an escape time $t_{\alpha}^{1}$ and a rotation $R$ such that
\[
\lim_{\alpha\to0}\theta_{\alpha}(t_{\alpha}^{1}+t)=RI^{(1\to w)}\underline{\theta}^{1}(t).
\]
\end{thm}
The unicity of the largest singular value of the gradient at the origin guarantees the unicity (up to rotation) of the optimal escape paths. For example, with the MSE loss, the gradient at the origin is $2YX^T$: for generic $Y$ and $X$, the largest singular value of the gradient has a multiplicity of $1$.

The reason why, for DLN with $L>3$, we can only guarantee a probability of $\frac 1 2$ in the previous theorem, is that we need to ensure that gradient descent does not get stuck at the saddle at the origin or at other saddles connected to it. For $L=2$, this follows from the fact that the saddle is strict. When $L>2$, the saddle is not strict and we were only able to prove it in the case where $L=3$. We conjecture that the behavior described in Theorem \ref{thm:first-path} happens with probability 1 for all $L\geq 2$.

As shown in the Appendix \ref{subsec:proof-first-path}, the escape time $t_{\alpha}$ is of order $-\log\alpha$
for shallow networks and of order $\alpha^{-(L-2)}$ for networks of depth $L>2$. Hence, the deeper the network, the
slower the gradient flow escapes the saddle.

Besides, as also discussed in the Appendix \ref{subsec:proof-first-path}, the norm $\left\Vert \theta^{1}(t)\right\Vert $ of the limiting escape
path $\theta^{1}(t)=RI^{(1\to w)}\underline{\theta}^{1}(t)$ grows
at an optimal speed: as $e^{s^{*}(t+T)}$ for some $T$ when $L=2$
and as $(s^{*}(L-2)(T-t))^{-\frac{1}{L-2}}$ for some $T$ when $L>2$,
where $s^{*}$ is the optimal escape speed $s^{*}=L^{-\frac{L-2}{2}}s_{1}$.
These are optimal in the sense that given an other gradient flow path $\theta(t)$ which exits from the origin, there exists a ball $B$ centered at the origin such that, for any small $\epsilon$, if $t_1$ and $t_2$ are the times such that $\left\Vert \theta^{1}(t_1)\right\Vert = \epsilon = \left\Vert \theta (t_2)\right\Vert$, then $\left\Vert \theta^{1}(t+t_1)\right\Vert  \geq \left\Vert \theta(t+t_2)\right\Vert$ for any positive $t$, until one of the paths exits the ball $B$.

\subsection{Subsequent Paths}
What happens after this first path? The width-$1$ gradient flow path $\underline{\theta}^1(t)$ converges to a width-$1$ critical point $\underline{\vartheta}^1$ as $t\to\infty$. While $\underline{\vartheta}^1$ may be a local minimum amongst width-1 DLNs, its inclusion $\vartheta^1 = RI^{(1\to w)}(\underline{\vartheta}^1)$ will be a saddle assuming it is not a global minimum already and that the network is wide enough, since if $w\geq \min \{ n_0, n_L\} $ all critical points are either global minima or saddles \cite{nouiehed2021local_openness}.

Theorem \ref{thm:first-path} guarantees that, as $\alpha \searrow 0$, the gradient flow path $\theta_\alpha(t)$ will approach the saddle $\vartheta^1$. It is then natural to assume that $\theta_\alpha(t)$ will escape this saddle along an optimal escape path (which is the inclusion of a width-2 path). Repeating this process, we expect gradient flow to converge as $\alpha \searrow 0$ to the concatenation of paths going from saddle to saddle of increasing width:

\begin{conjecture}
\label{conj:Conj-Saddle-Saddle} With probability $1$, there exist $K+1$
critical points $\vartheta^{0},\dots,\vartheta^{K}\in\mathbb{R}^{P_{L,w}}$ (with
$\vartheta^{0}=0$) and $K$ gradient flow paths $\theta^{1},\ldots,\theta^{K}:\mathbb{R}\to\mathbb{R}^{P_{L,w}}$
connecting the critical points (i.e. $\lim_{t\to-\infty}\theta^{k}(t)=\vartheta^{k-1}$
and $\lim_{t\to+\infty}\theta^{k}(t)=\vartheta^{k}$) such that the
path $\theta_{\alpha}(t)$ converges
as $\alpha\to0$ to the concatenation of $\theta^{1}(t),\dots,\theta^{K}(t)$
in the following sense: for all $k<K$, there exist times $t_{\alpha}^{k}$ (which depend on $\theta_0$)
such that
\[
\lim_{\alpha\to0}\theta_{\alpha}(t_{\alpha}^{k}+t)=\theta^{k}(t).
\]

Furthermore, for all $k < K$, there is a deterministic path $\underline{\theta}^{k}(t)$ and a local minimum $\underline{\vartheta}^{k}$ of a width-$k$ network such that for some rotation $R$ (which depends on $\theta_0$), $\theta^k(t)=RI^{(k\to w)}(\underline{\theta}^k(t))$ and $\vartheta^k=RI^{(k\to w)}(\underline{\vartheta}^k)$ for all $k$ and $t$.
\end{conjecture}

This Saddle-to-Saddle behavior explains why for small initialization scale, the train error gets stuck at plateaus during training (Figures \ref{fig:linear_saddle_to_saddle_MC} and \ref{fig:regimes-of-training}). Conjecture \ref{conj:Conj-Saddle-Saddle} suggests that these plateaus correspond to the saddle visited.

Note that for losses such as the cross-entropy, the gradient descent may diverge towards infinity, as studied in \cite{soudry2018implicit,gunasekar_2018_implicit_bias}.
From now on, we focus on the case where $\vartheta^K$ is a finite global minimum. By the invariance under gradient flow of $\mathrm{Im}[I^{(k\to w)}]$ (the image of the inclusion map), the inclusion of a width-$k$ local minimum $\underline{\vartheta}^k$ into a larger network is a saddle $\vartheta^k$ (if $A_{\vartheta^k}$ is not a global minimum of $C$). These types of saddles are closely related to the symmetry-induced saddles studied in \cite{simsek-2021-symmetries-loss} in non-linear networks.

\begin{rem}
Note that each of the limiting paths $\theta^k$ and critical points $\vartheta^k$ will be \emph{balanced} (i.e. their weight matrices satisfy $W_\ell W_\ell^T = W_{\ell+1}^T W_{\ell+1}$ for all $\ell=1,\dots,L-1$). The origin is obviously balanced and since balancedness is an invariant of gradient flow and all other paths and saddles are connected to the origin by a sequence of gradient flow paths, they must be balanced too. Note however that for all $\alpha>0$, the path $\theta_\alpha(t)$ is almost surely not balanced.
\end{rem}

\subsection{Greedy Low-Rank Algorithm}\label{sec:greedyalgo}

%

Conjecture \ref{conj:Conj-Saddle-Saddle} suggests that the gradient flow with vanishing initialization implements a greedy low-rank algorithm which performs a greedy search for a lowest-rank solution: it first tries to fit a width $1$ network, then a width $2$ network and so on until reaching a solution. Thus, we expect that as $\alpha \searrow 0$, the dynamics of gradient flow corresponds, up
to inclusion and rotation, to the limit of the algorithm $\mathcal{A}_{\epsilon,T,\eta}$
as sequentially $T\to\infty$, $\eta\to0$ and $\epsilon\to0$. In particular, we used the Algorithm $\mathcal{A}_{\epsilon,T,\eta}$, with large $T$ and small $\eta$ and $\epsilon$ to approximate the paths $\underline{\theta}^{k}$ and points $\underline{\vartheta}^k$ in Figure \ref{fig:linear_saddle_to_saddle_MC}.
Note how this limiting algorithm is deterministic. This implies that even for finite widths the dynamics of gradient flow converge to a deterministic limit (up to random rotations $R$) as the variance at initialization goes to zero.

A similar algorithm has already been described in \cite{li2020towards}, however thanks to our different proof techniques, we are able to give a more precise description of the evolution of the parameters.

\renewcommand{\thealgorithm}{}

\begin{algorithm}[tb]
   \caption{$\mathcal{A}_{\epsilon,T,\eta}$}
   \label{alg:example}
\begin{algorithmic}
   \STATE \# Compute the first singular vectors of $\nabla C(0)$:
   \STATE $u,s,v \leftarrow \mathrm{SVD}_1(\nabla C(0))$
   \STATE $\theta \leftarrow (-\epsilon v^T,\epsilon,\dots,\epsilon u)$
   \STATE $w\leftarrow 1$
   \WHILE{$C(A_\theta) < C_{min} + \epsilon$}
   \STATE \# $T$ steps of GD on the loss of width-$w$ DLN with lr $\eta$
   \STATE $\theta \leftarrow \mathrm{SGD}_{w,T,\eta}(\theta)$
   \STATE $u,s,v \leftarrow \mathrm{SVD}_1(\nabla C(A_\theta))$
   \STATE $\theta \leftarrow \left( \left(\begin{array}{c}
        W_{1}\\
        -\epsilon v^{T}
        \end{array}\right)
        ,
	\left(\begin{array}{cc}
       	W_{2} & 0\\
        0 & \epsilon
        \end{array}\right)
        ,\dots,
        \left(\begin{array}{cc}
        W_{L} & \epsilon u \end{array}\right)
        \right)$
   \STATE $w\leftarrow w+1$
   \ENDWHILE
\end{algorithmic}
\end{algorithm}

\subsection{Description of the paths that escape a saddle}
\label{sec:escapepath}
%
%

Our proof relies on a theorem which relates the escape paths of the saddle at the origin of the cost $\mathcal{L}$ and the escape paths of the $L$-th order Taylor approximation $H$ of $\mathcal L$. This correspondence only applies to paths which escape the saddle sufficiently fast.

We define the set of fast escaping paths $\mathcal{F}_{\mathcal{L}}(s)$
of the cost $\mathcal{L}$ with speed at least $s$ as follows:
\begin{itemize}
\item for shallow networks ($L=2$), it is the set of gradient
flow paths that satisfy $\left\Vert \theta(t)\right\Vert =O(e^{st})$
as $t\to-\infty$,
\item for deep networks ($L>2$), it is the set of gradient
flow paths that satisfy $\left\Vert \theta(t)\right\Vert \leq\left(s(L-2)(T-t)\right)^{-\frac{1}{L-2}}$
for some $T$ and any small enough $t$.
\end{itemize}

The optimal escape speed is $s^{*}=L^{-\frac{L-2}{2}}s_{1}$ where $s_1$ is the largest singular value of $\nabla C(0)$. It is the optimal escape speed in the sense that there are no faster escape paths:  $\mathcal{F}_{\mathcal{L}}(s)=\emptyset$ if $s>s^*$. Escape paths which exit the saddle at the optimal escape speed are called optimal escape paths.

There is a bijection between fast escaping paths of the loss $\mathcal{L}$
and those of its $L$th order Taylor approximation $H$:
\begin{thm}\label{thm:bijection-fast-escape-paths}
\textbf{Shallow networks:
}for all $s$ s.t. $s > \frac{1}{3}s^{*}$ there is a unique bijection $\Psi:\mathcal{F}_{\mathcal{L}}(s)\to\mathcal{F}_{H}(s)$
such that for all paths $\theta\in\mathcal{F}_{\mathcal{L}}(s)$,
$\left\Vert \theta(t)-\Psi(\theta)(t)\right\Vert =O(e^{3st})$ as
$t\to-\infty$.

\textbf{Deep networks:} for all $s>\frac{L-1}{L+1}s^{*}$, there is
a unique bijection $\Psi:\mathcal{F}_{\mathcal{L}}(s)\to\mathcal{F}_{H}(s)$
such that for all paths $\theta\in\mathcal{F}_{\mathcal{L}}(s)$,
$\left\Vert \theta(t)-\Psi(\theta)(t)\right\Vert =O((-t)^{-\frac{L+1}{L-2}})$
as $t\to-\infty$.
\end{thm}

We believe that this theorem is of independent interest, and it is stated in a more general setting in the Appendix. Theorem \ref{thm:bijection-fast-escape-paths} is similar to the Hartman-Grobman Theorem, which shows a bijection, in the vicinity of a critical point, between the gradient flow paths of $\mathcal{F}$ and of its linearization. The bijection in Theorem \ref{thm:bijection-fast-escape-paths} holds only between fast escaping paths, but it gives
stronger guarantees regarding how close the paths $\theta(\cdot)$ and $\Psi(\theta)(\cdot)$ are. In particular, Theorem \ref{thm:bijection-fast-escape-paths} guarantees that a fast escaping path $\theta(\cdot)$ and its image $\Psi(\theta)(\cdot)$
have the same `escape speed', whereas the correspondence between paths of in the Hartman-Grobman theorem does not in general conserve speed. This is due to the fact that the homeomorphism which allows to construct the bijection in the Hartman-Grobman theorem
is only H\"older continuous. This suggests that fast escaping
paths can be guaranteed to conserve their speed after the Taylor approximation
while slower paths can change speed. Finally, our result has the significant
advantage that it may be applied to higher order Taylor approximations, whereas the Hartman-Grobman Theorem only applies to the linearization of the flow (i.e. it could only be useful in the shallow case $L=2$).

\subsection{Sketch of Proof} \label{sec:sketch-of-proof}

In this section, we provide a sketch of proof for Theorem \ref{thm:first-path}.

We fix some small $r>0$ independent of $\alpha$. The \emph{escape time} $t_{\alpha}$ is the earliest
time such that $\left\Vert \theta_{\alpha}(t_{\alpha})\right\Vert =r$. We show that the
limiting escape path $(\theta^{1}(t))_{t\in\mathbb{R}}$ as $\theta^{1}(t)=\lim_{\alpha\searrow0}\theta_{\alpha}(t_{\alpha}+t)$
is well defined and non-trivial since $\theta^{1}(0)\neq0$. The next step of the proof is to show that $\theta^{1}$ escapes the
saddle at an almost optimal speed: for any $\epsilon>0$, for some $T$ and any small enough $t$, for shallow network
$\left\Vert \theta^{1}(t)\right\Vert =O(e^{(s^{*}-\epsilon)t})$, and for deeper networks $\left\Vert \theta^{1}(t)\right\Vert \leq\left[(L-2)(s^{*}-\epsilon)(T-t)\right]^{-\frac{1}{L-2}}$. We may therefore apply Theorem
\ref{thm:bijection-fast-escape-paths}: there exists a unique optimal escape path for the $L$-th order Taylor approximation $H$ around the origin which is `close', in the sense given in Theorem
\ref{thm:bijection-fast-escape-paths}, to $\theta^{1}$.

For the Taylor approximation $H$, we have a precise description of
the optimal escaping paths for the saddle at the origin. Assuming that the largest singular value $s_{1}$ of the gradient matrix $\nabla C(0)$ has multiplicity one, all
optimal escape paths of $H$ (i.e. the set of paths that escape with
the largest speed) are of the form $\theta_H(t)=d(t)RI^{(1\to w)}\left(\rho\right)$ where $R$ is some
rotation, the scalar function $d(t)$ is equal to $e^{s^{*}(t+T)}$ for shallow
networks and $(s^{*}(L-2)(T-t))^{-\frac{1}{L-2}}$ for deep networks, and the vector of parameters $\rho$ is given by:
\[
\rho=\left(v_{1}^{T},1,\dots,1,u_{1}\right)
\]
with $u_{1},v_{1}$ the left and right singular vectors of the
largest singular value $s_{1}$ of the gradient matrix $\nabla C(0)$.

Let us consider the unique optimal escape path $\theta_H(t)=d(t)RI^{(1\to w)}\left(\rho\right)$ for $H$ which is `close' to $\theta^1$. The path $\underline{\theta}_H(t)= d(t)\rho$ is also an optimal escape path for $H$: from Theorem \ref{thm:bijection-fast-escape-paths},  there exists a unique optimal escape path $\underline{\theta}(t)$ which is `close' to $\underline{\theta}_H$. The former escape path corresponds to a $1$-width DLN and it is easy to show that $RI^{(1\to w)}(\underline{\theta})$ is an optimal escape path for $\mathcal{L}$ which is  `close' to $RI^{(1\to w)}(\underline{\theta}_H)=\theta_H$.

In particular, we obtain that both $\theta^1$ and $RI^{(1\to w)}(\underline{\theta})$ are optimal escape path for $\mathcal{L}$ which are  `close' to $\theta_H$. By the unicity property in Theorem \ref{thm:bijection-fast-escape-paths}, we obtain that $\theta^1=RI^{(1\to w)}(\underline{\theta})$ which allows us to conclude.


\begin{rem}
To prove  Conjecture \ref{conj:Conj-Saddle-Saddle}, one needs to apply a similar argument to understand how gradient flow escapes the subsequent saddles $\vartheta^1,\dots,\vartheta^K$. There are two issues:

First, even though Theorem \ref{thm:first-path} guarantees that gradient
descent will come arbitrarily close to the next saddle $\vartheta^{1}$,
it may not approach it along a generic direction: it could approach along
a ``bad'' direction. For the first path, we relied on the fact that $\theta_0$ is Gaussian to  guarantee that these bad directions are avoided with probability $1$ (or $\nicefrac 1 2$). Note that this problem c

ould be addressed using the so-called perturbed stochastic gradient descent described in \cite{Chi2017, Du2017} since, in this learning algorithm, once in the vicinity of the saddle, a small Gaussian noise is added to the parameters: as a consequence, they end up being in a generic position in the neighborhood of the saddle.

Second, for deep networks ($L>2$), the saddle $\vartheta^{1}$
has a different local structure to $\vartheta^{0}$. Indeed, at the origin,
the $L-1$ first derivatives vanish, leading to an (approximately) $L$-homogeneous saddle at the origin. On the contrary, at the rank $1$ saddle $\vartheta^{1}=RI^{(1\to w)}(\underline{\vartheta}^{1})$, if $\underline{\vartheta}^{1}$ is a local minimum of the
width $1$ network, the
Hessian is positive along the inclusion $\mathrm{Im}\left[RI^{(1\to w)}\right]$. This implies that the dynamics can only escape the saddle through the Hessian null-space, along which the first $L-1$ derivatives vanish. Although
the loss restricted to this null-space around $\vartheta_1$ has a similar structure to the loss around the origin,
the fact that the Hessian at $\vartheta_1$ is not null complexifies the analysis.
\end{rem}

\section{Characterization of the Regimes of Training}

\begin{figure}[!t]
  \centering
  \subfloat{
      \includegraphics[width=0.4\textwidth]{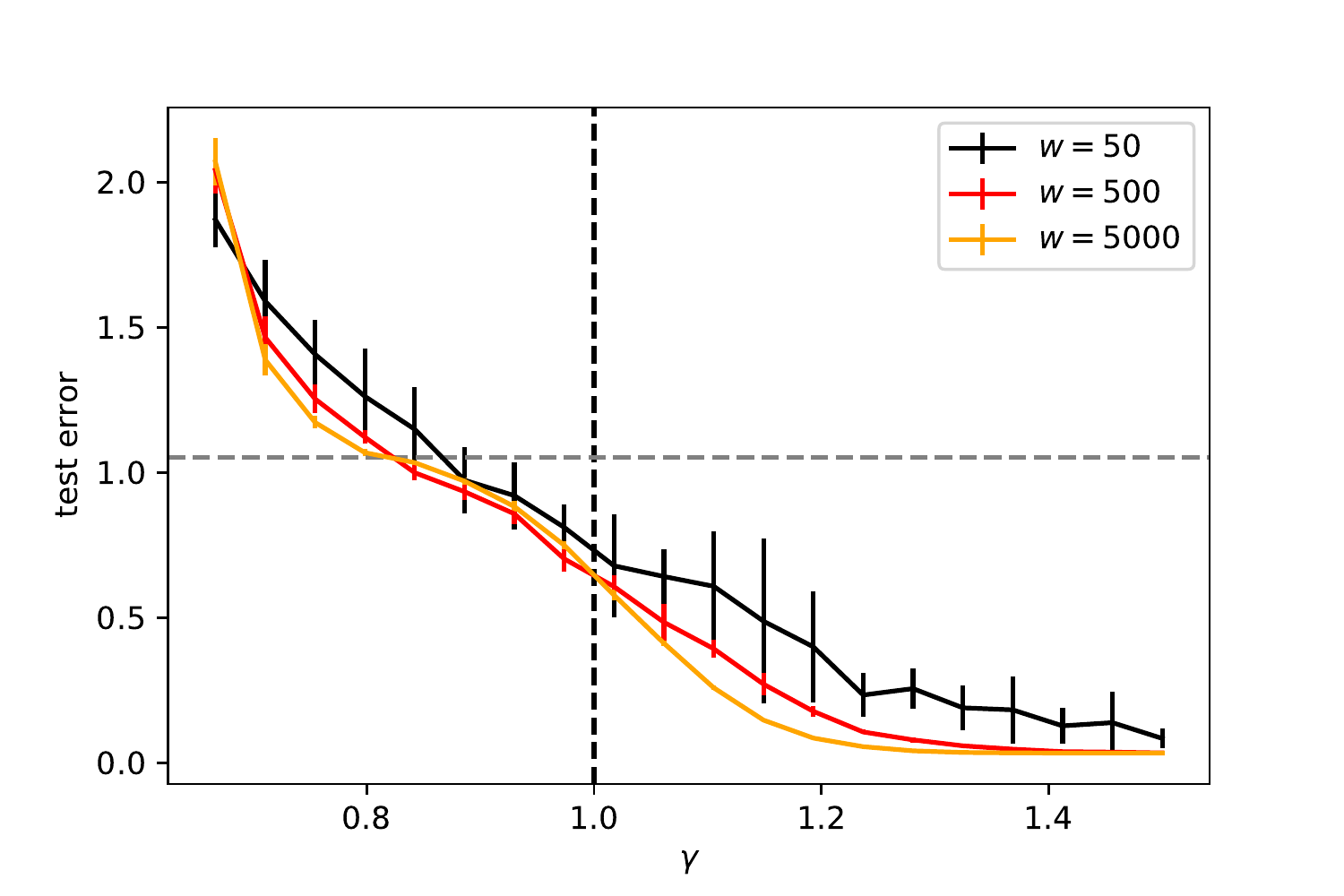}
      } \vspace{-0.8cm} \\
  \subfloat{
      \includegraphics[width=0.4\textwidth]{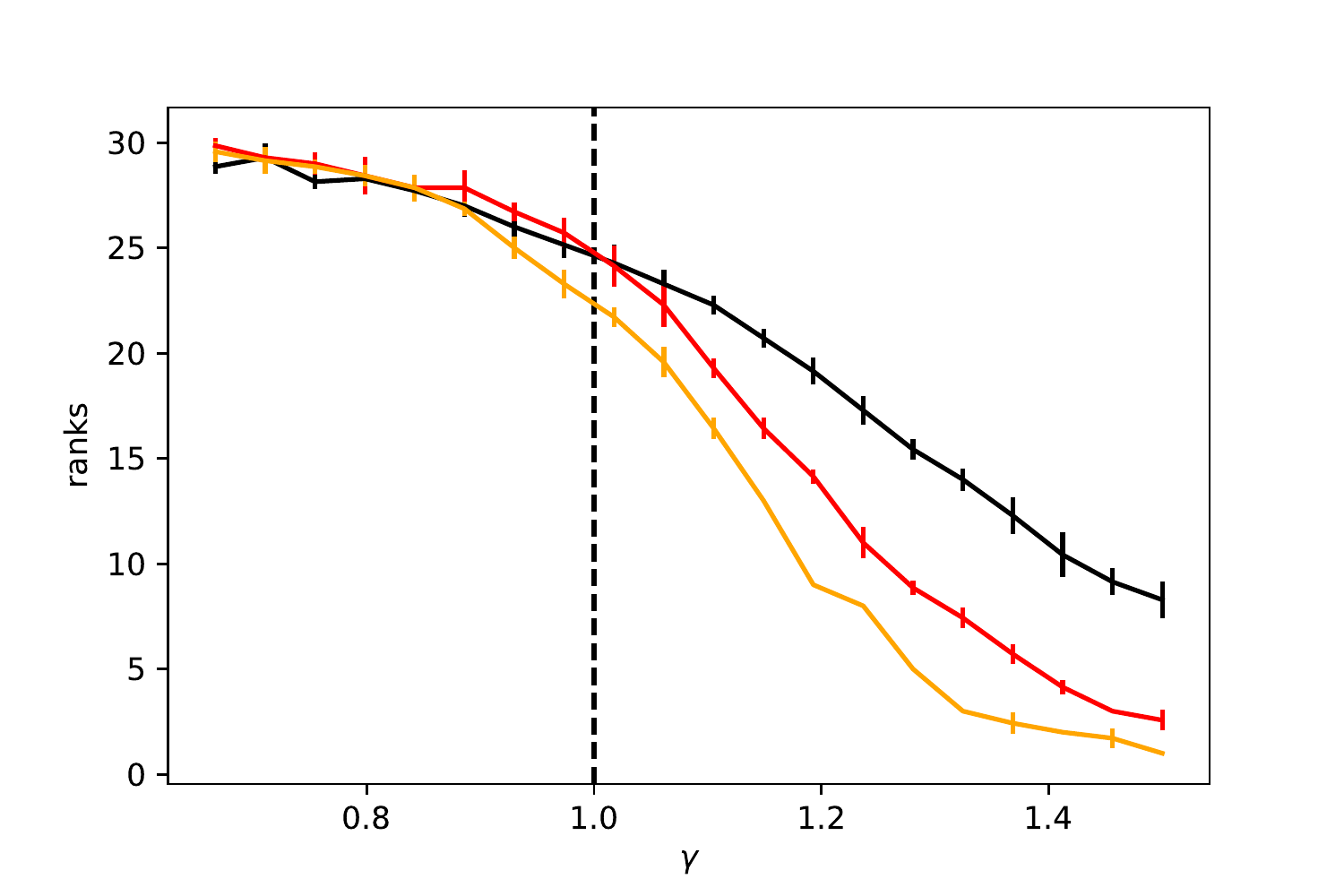}
      } \vspace{-0.2cm}
  \caption{\textit{Test errors and ranks at convergence as a function of initialization scale $\gamma$, matrix completion task.} The task is finding a matrix of size $30 \times 30$ and rank $1$ from $20\%$ of its entries. The test error and ranks are averaged over $7$ seeds ($\pm 1$ standard deviations are reported in the error bar). In the NTK regime, the solutions at convergence are almost full-rank and the test error is roughly the same or worse than that of the zero predictor. On the other hand in the Saddle-to-Saddle regime the test error approaches zero. As the width grows the transition between regimes becomes sharper and the test error becomes more consistent within each regimes.
  }
  \label{fig:matrix-completion}
\end{figure}

In light of the results presented in this paper, we discuss the three regimes that can be obtained by varying the initialization scale $\gamma$: the kernel regime ($\gamma<1$), the Mean-Field regime ($\gamma=1$) and the Saddle-to-Saddle regime ($\gamma>1$).

The NTK limit ($\gamma=1-\frac{1}{L}$) \cite{jacot2018neural,lee2019wide} is representative of the other
scalings $1-\frac{1}{L}\leq\gamma<1$ \cite{yang2020feature}. The
critical regime $\gamma=1$ corresponds to the Mean-Field limit for
shallow networks \cite{Chizat2018,Rotskoff2018} or the Maximal Update
parametrization for deep networks \cite{yang2020feature}. Finally, we conjecture that the last regime where $\gamma > 1$, displays features very akin to the $\gamma=+\infty$ case studied in this article. Under this assumption, we obtain the following list of properties that characterize each of these regimes:

In the \textbf{NTK regime} ($1-\frac{1}{L}\leq\gamma<1$):
\begin{enumerate}
\item During training, the parameters converge to a nearby global minimum,
and do not approach any saddle (Figure \ref{fig:regimes-of-training-NTK} shows how the plateaus disappear as $w$ grows).
\item If the cost on matrices $C$ is strictly convex, one can guarantee
exponential decrease of the loss (i.e. linear convergence).
\item The NTK is asymptotically fixed during training.
\item No low-rank bias in the learned matrix - as a result the test error for matrix completion is the same (or even larger) than the zero predictor in the NTK regime, as shown in Figure \ref{fig:matrix-completion}. 
\end{enumerate}
The \textbf{Saddle-to-Saddle regime} ($\gamma>1$):
\begin{enumerate}
\item The parameters start in the vicinity of a saddle and visit a sequence
of saddles during training. They come closer to each of these saddles
as the width grows.
\item As the width grows, it takes longer to escape each saddle, leading
to long plateaus for the training error. The training time is therefore asymptotically
infinite (see Figure \ref{fig:regimes-of-training-StoS}).
\item The rate of change $\left\Vert \Theta (\theta_{T}) - \Theta (\theta_{0}) \right\Vert$ (where $T\in\mathbb{R}$ is the stopping time) of the NTK is infinitely larger than the NTK at initialization $\left\Vert \Theta (\theta_{0}) \right\Vert$. This follows from the fact that the NTK at initialization goes to zero, while it has finite size at the end of training.
\item The learned matrix is the result of a greedy algorithm that finds the lowest rank solution.
\end{enumerate}
The \textbf{Mean-Field regime} $\gamma=1$ lies at the transition
between the two previous regimes and is more difficult to characterize:
\begin{enumerate}
\item In this critical regime, the constant factor $c$ in the variance at initialization
$\sigma^{2}=cw^{-\gamma}$ can have a strong effect on the dynamics.
\item Plateaus can still be observed (see Figure \ref{fig:regimes-of-training-MF}), however in contrast to the Saddle-to-Saddle regime, the length of the plateaus does not increase as the width grows, but remains roughly constant.
\item The NTK and its rate of change are of same order.
\end{enumerate}
In general, we observe some tradeoff: the NTK regime
leads to fast convergence without low-rank bias, while the Saddle-to-Saddle
regime leads to some low-rank bias, but at the cost of an asymptotically
infinite training time.

\section{Conclusion}
We propose a simple criterion to identify three regimes in the training of large DLNs: the distances from the initialization to the nearest global minimum and to the nearest saddle. The NTK regime ($1-\frac 1 L \leq \gamma <1$) is characterized by an initialization which is close to a global minimum and far from any saddle, the Saddle-to-Saddle regime ($\gamma>1$) is characterized by an initialization which is close to a saddle and (comparatively) far from any global minimum and, finally, in the critical Mean-Field regime ($\gamma=1$), these two distances are of the same order as the width grows.

While the NTK and Mean-Field limits are well-studied, the Saddle-to-Saddle regime is less understood. We therefore investigate the case $\gamma=+\infty$ (i.e. we fix the width and let the variance at initialization go to zero). In this limit, the initialization converges towards the saddle at the origin $\vartheta^0=0$. We show that gradient flow naturally escapes this saddle along an `optimal escape path' along which the network behaves as a width-1 network. This leads the gradient flow to subsequently visit a second saddle $\vartheta^1$ which has the property that the matrix $A_{\vartheta^1}$ has rank $1$. We conjecture that the gradient flow next visits a sequence of critical points $\vartheta^2,\dots,\vartheta^K$ of increasing rank, implementing some form of greedy low-rank algorithm. These saddles explain the plateaus in the loss curve which are characteristic of the Saddle-to-Saddle regime.

Similar plateaus can be observed in non-linear networks: this suggests that the regimes and dynamics described in this paper could be generalized to non-linear networks.


\bibliographystyle{authordate1}
\bibliography{../../../main}

\appendix

\onecolumn


We organize the Appendix as follows:
\begin{itemize}
\item In Section \ref{sec:exps}, we present the details for the numerical
results presented in the main text together with some discussions.
\item In Section \ref{sec:Regimes-of-Training}, we present the proofs for
the result on the proximity of critical points, i.e. Theorem \ref{th:distances}.
\item In Section \ref{sec:proofs-saddle-to-saddle}, we present the proofs
for the Saddle-to-Saddle regime, in particular Theorems \ref{thm:first-path} and  \ref{thm:bijection-fast-escape-paths}.
\item In Section \ref{sec:technical-results}, we state and prove a few technical results.
\end{itemize}

\section{Further Experimental Details}

\label{sec:exps}

\begin{figure*}[!t]
  \centering
  \subfloat{
      \includegraphics[height=0.185\textwidth]{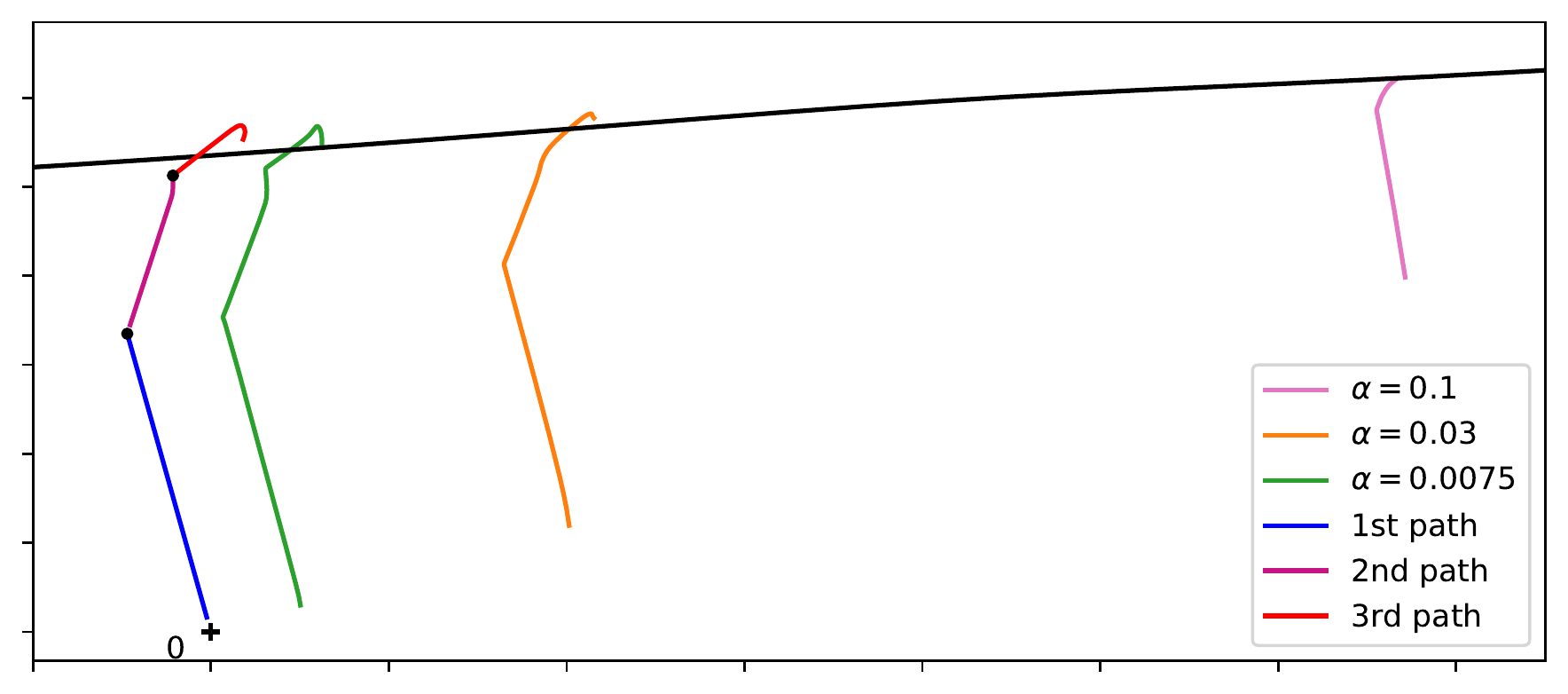}
      }
  \subfloat{
      \includegraphics[height=0.185\textwidth]{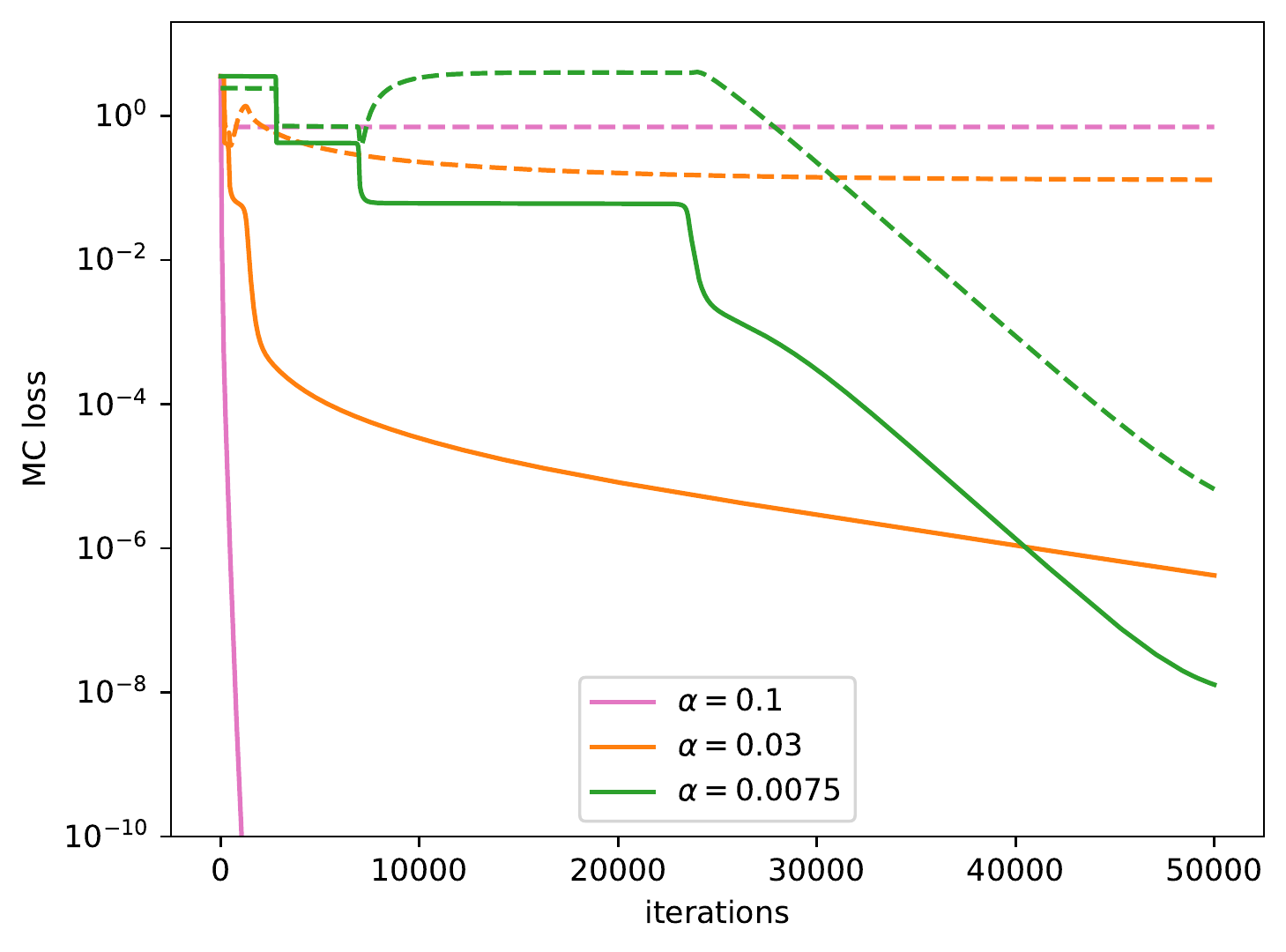}
      }
  \subfloat{
      \includegraphics[height=0.185\textwidth]{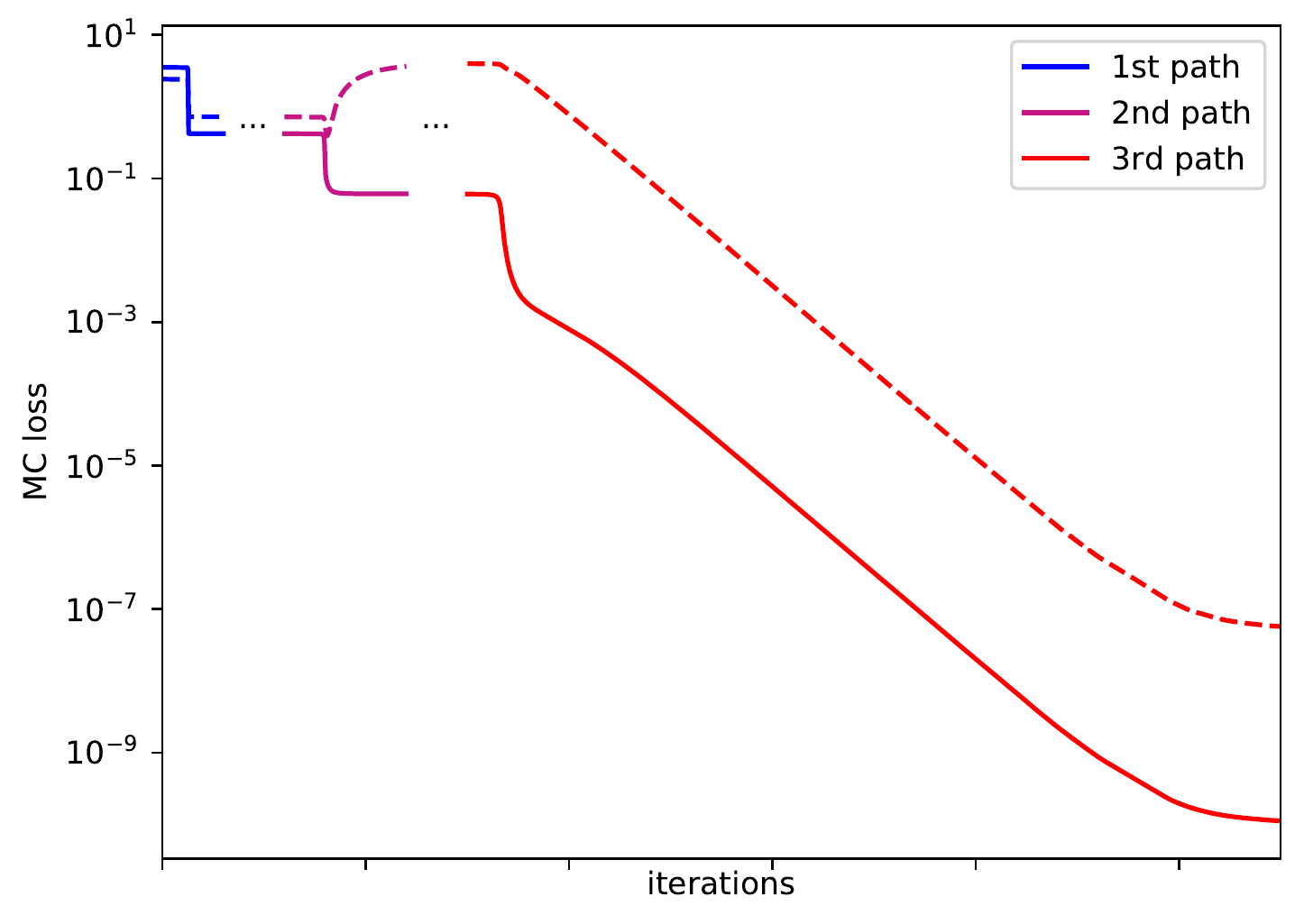}
      }
       \caption{\textit{Matrix Completion in linear/lazy vs. saddle-to-saddle regimes.}
  3 DLNs ($L=4,w=100$) trained on a MC loss fitting a $10\times 10$ matrix of rank $3$ with initialization $\alpha \theta_0$ for a fixed random $\theta_0$ and three values of $\alpha$. \textbf{Left:} Train (solid) and test (dashed) MC cost for the three networks, for large $\alpha$ the network is in the linear/lazy regime and does not learn the low-rank structure. For smaller $\alpha$ plateaus appear and the network generalizes. \textbf{Middle:} Visualization of the gradient paths in parameter space. The black line represents the manifold of solutions to which all example paths converge. As $\alpha \to 0$ the training trajectory converges to a sequence of 3 paths (in blue, purple and red) starting from the origin (+) and passing through 2 saddles ($\cdot$) before converging. \textbf{Right:} The train (solid) and test (dashed) loss of the three paths plotted sequentially, in the saddle-to-saddle limit; $\cdots$ represent an infinite amount of steps separating these paths.}

\end{figure*}

\begin{figure*}[!t]
  \centering
  \subfloat{
      \includegraphics[width=0.33\textwidth]{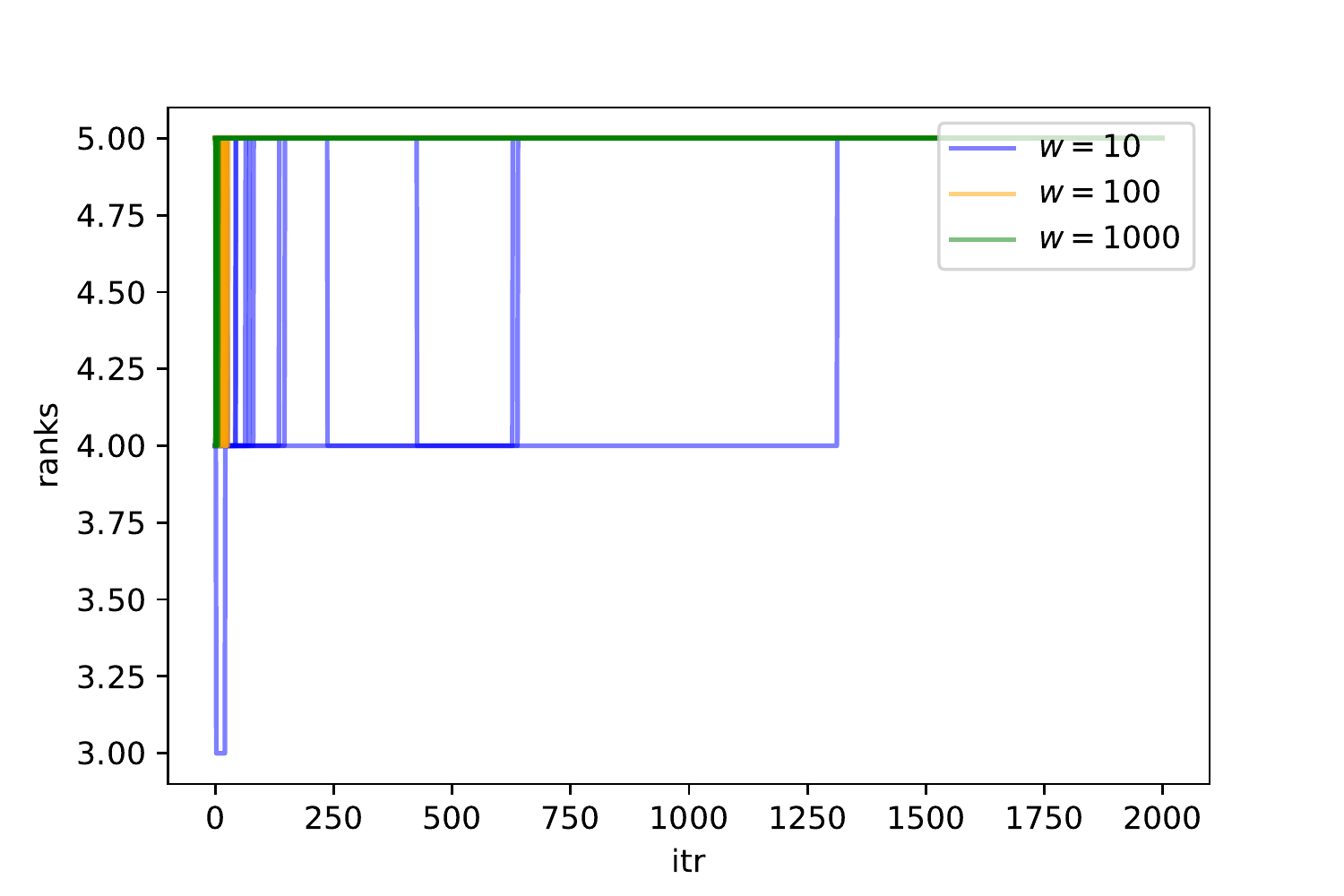}
      } \hspace{-0.6cm}
  \subfloat{
      \includegraphics[width=0.33\textwidth]{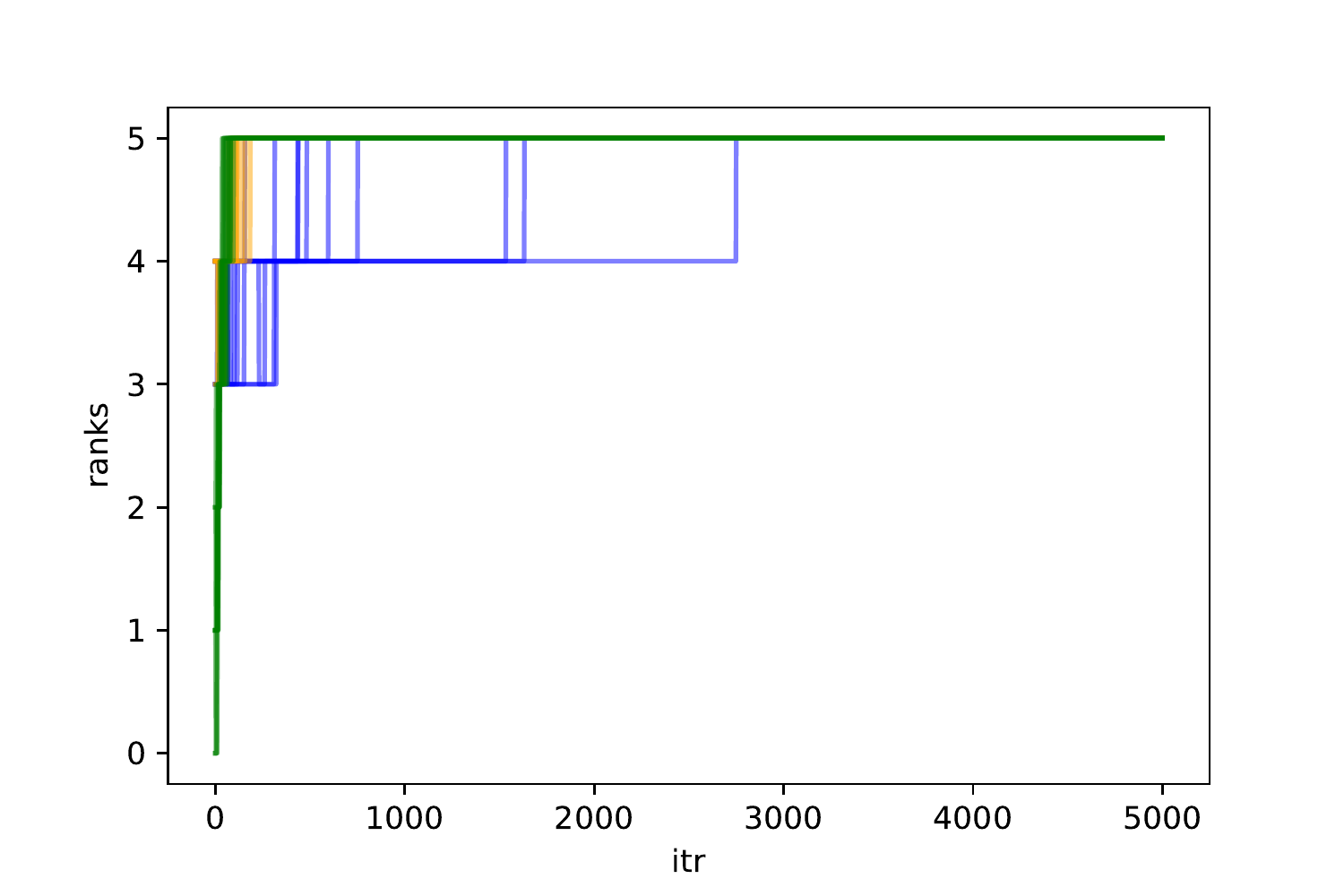}
      } \hspace{-0.6cm}
  \subfloat{
      \includegraphics[width=0.33\textwidth]{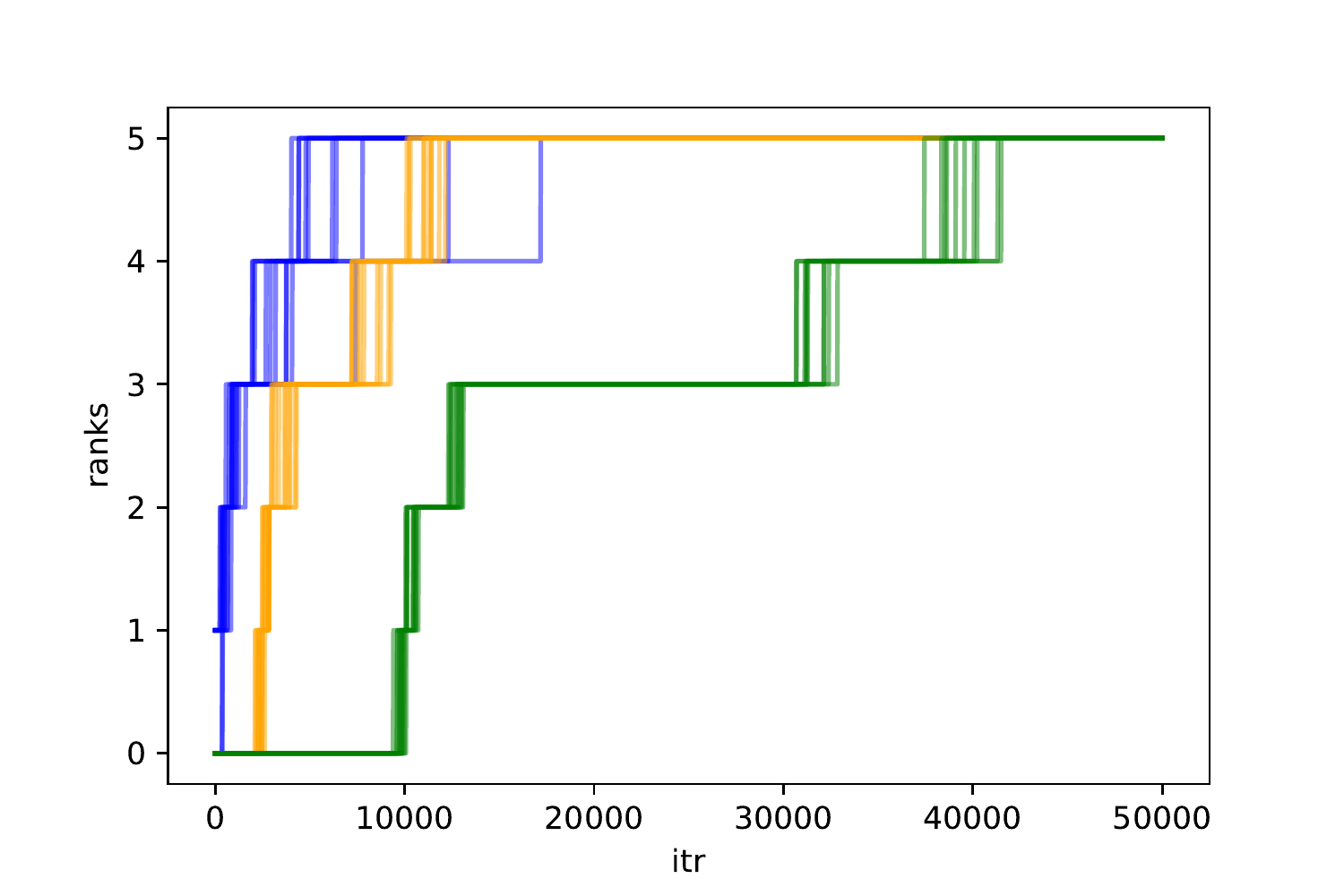}
      } \\ \vspace{-0.5cm}
  \subfloat{
      \includegraphics[width=0.33\textwidth]{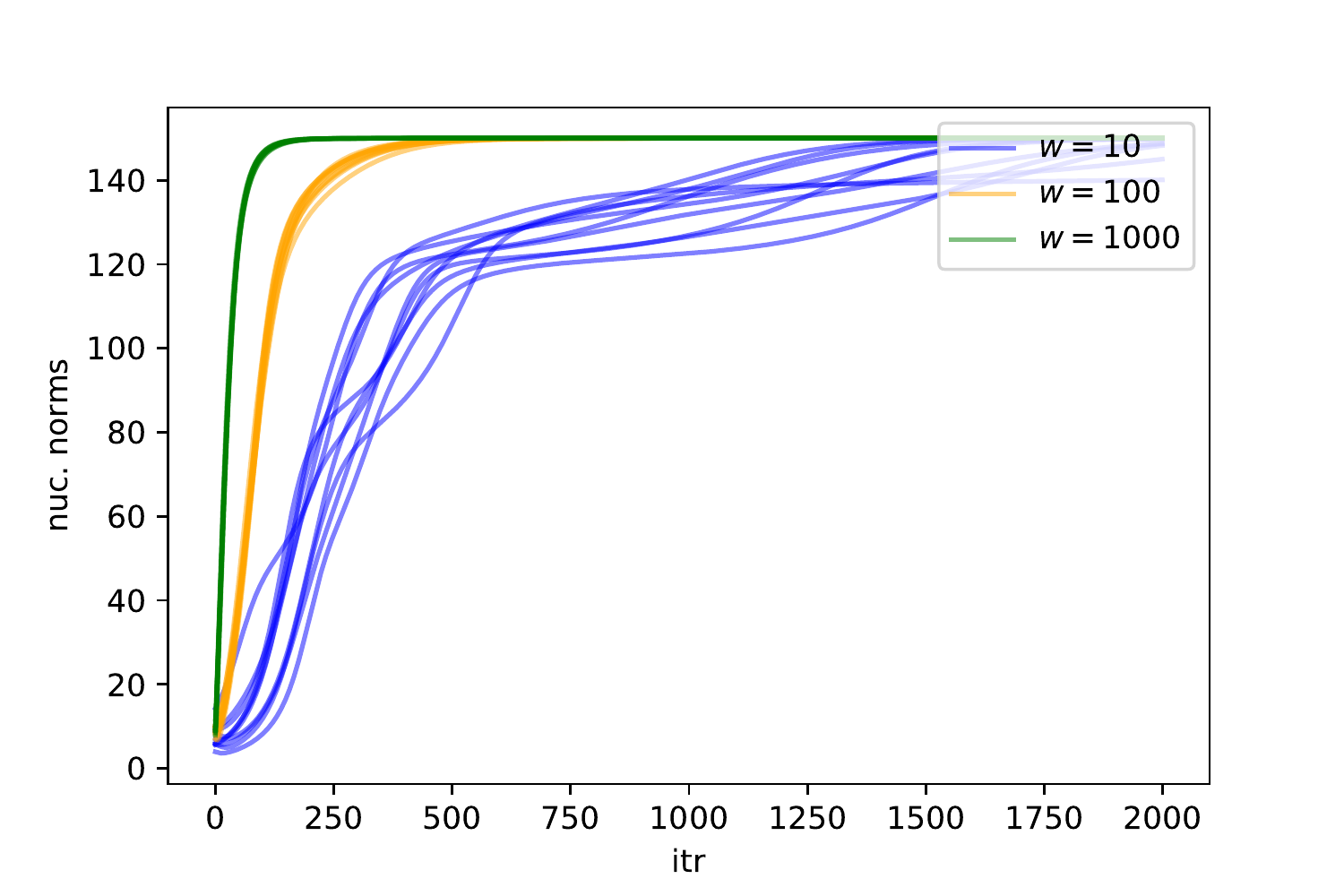}
      } \hspace{-0.6cm}
  \subfloat{
      \includegraphics[width=0.33\textwidth]{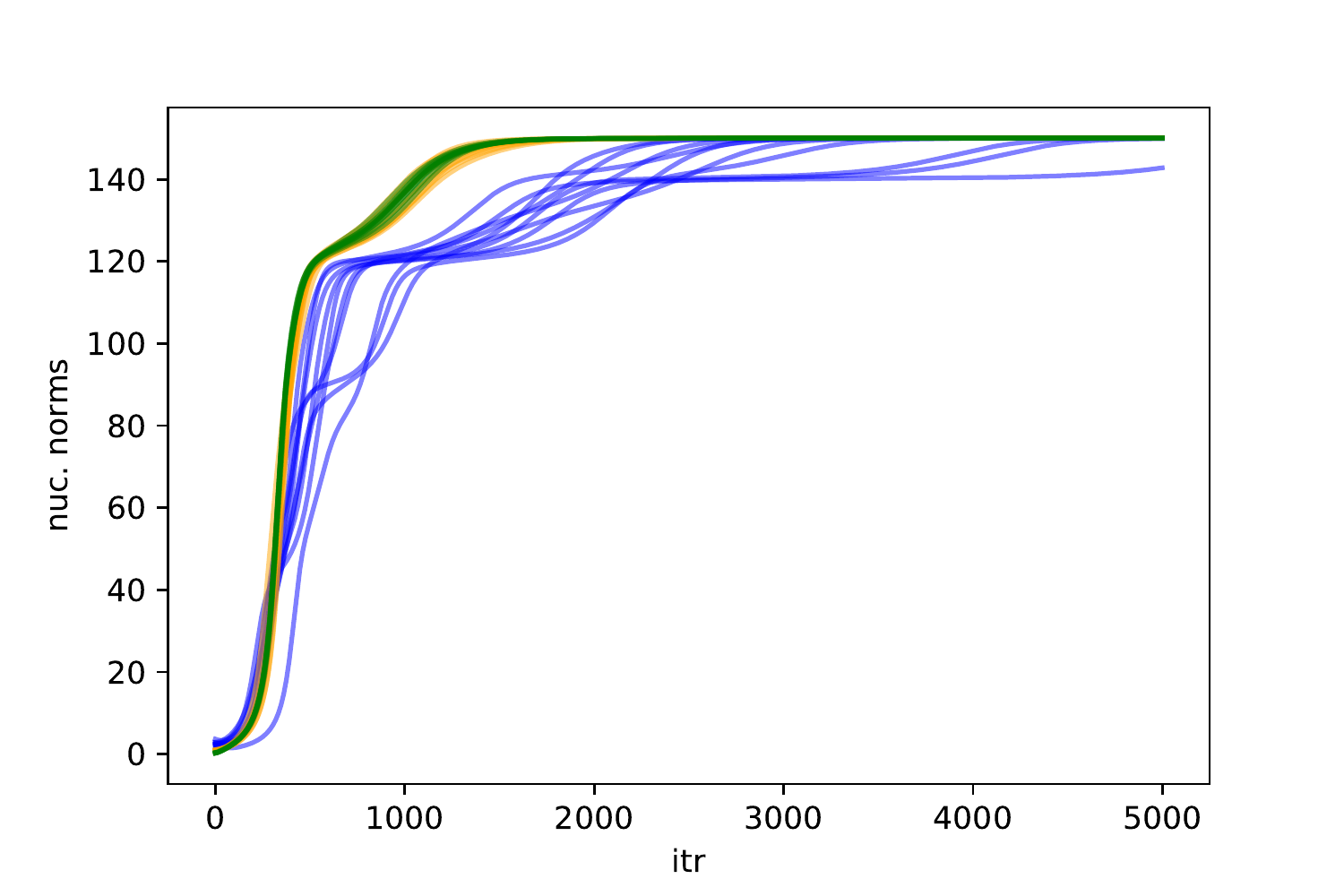}
      } \hspace{-0.6cm}
  \subfloat{
      \includegraphics[width=0.33\textwidth]{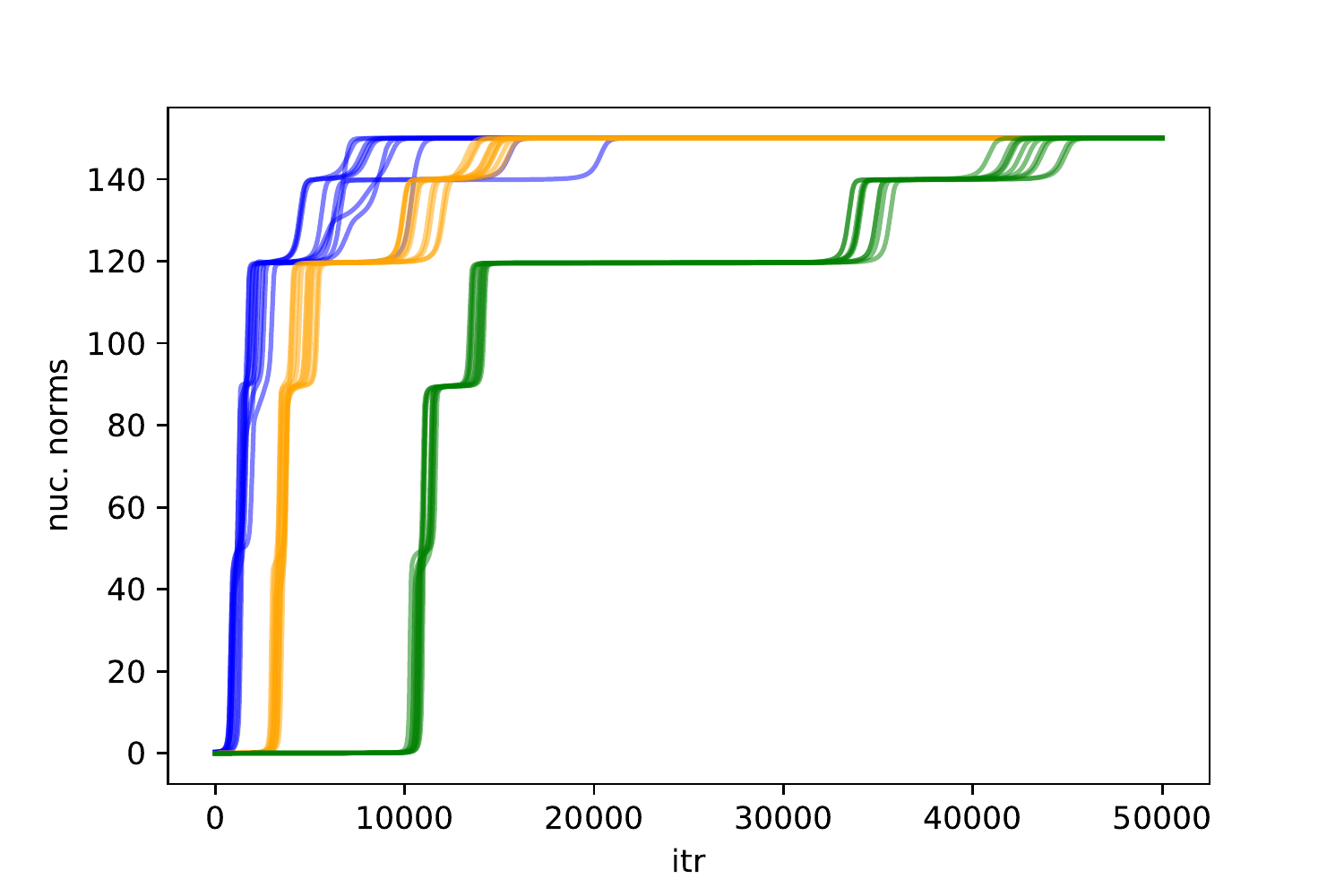}
      } \vspace{-0.5cm} \\
  \subfloat{
      \includegraphics[width=0.33\textwidth]{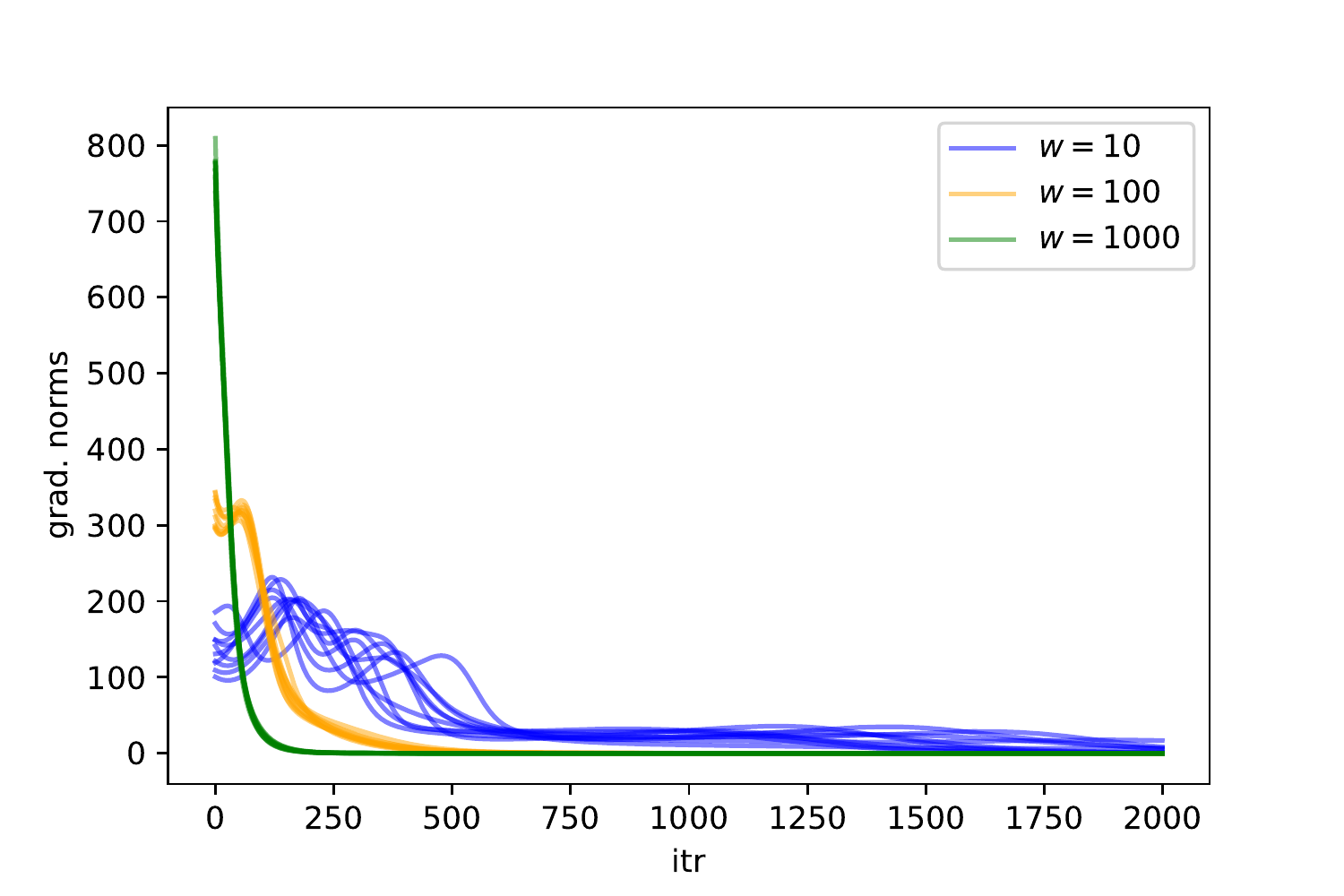}
      } \hspace{-0.6cm}
  \subfloat{
      \includegraphics[width=0.33\textwidth]{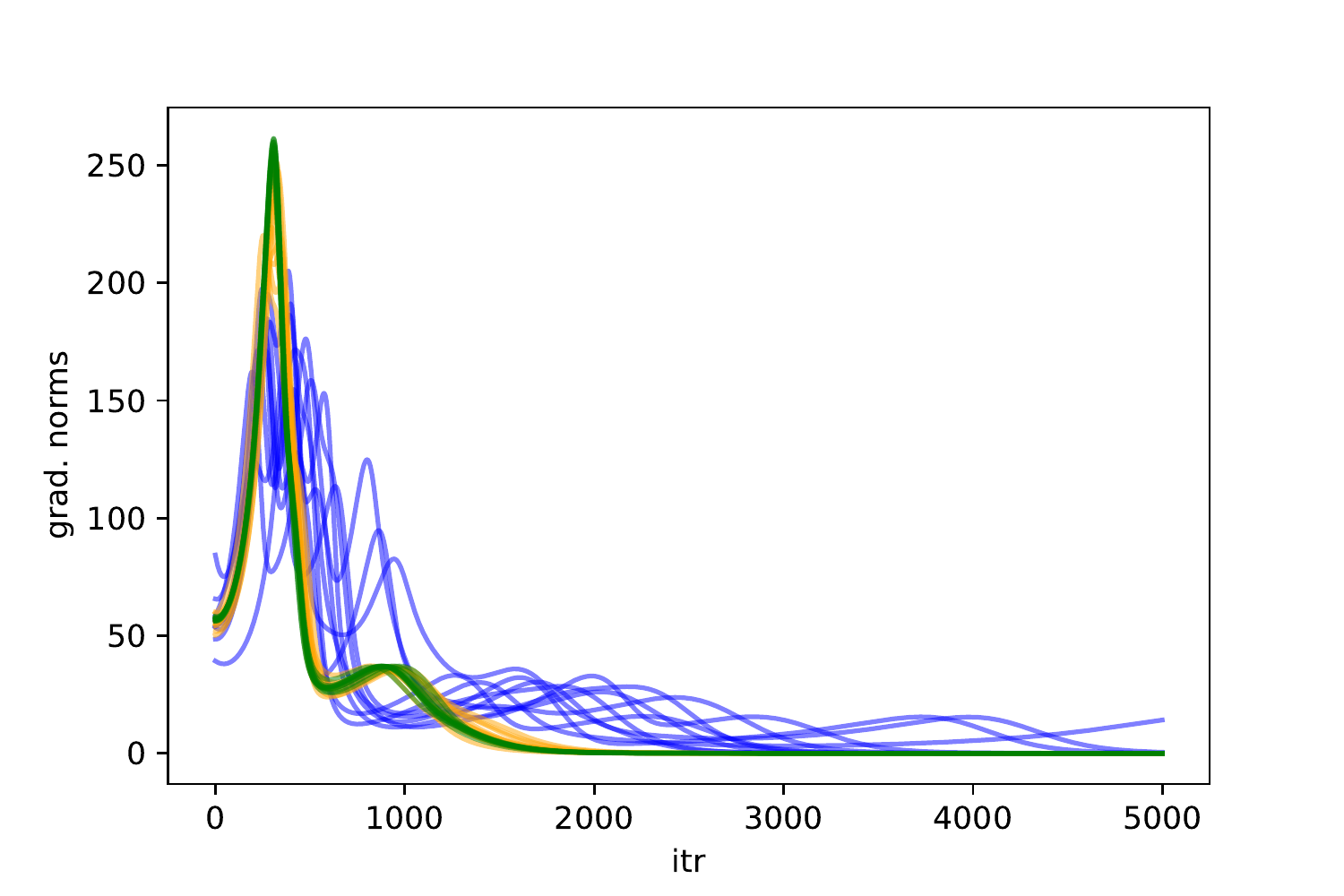}
      } \hspace{-0.6cm}
  \subfloat{
      \includegraphics[width=0.33\textwidth]{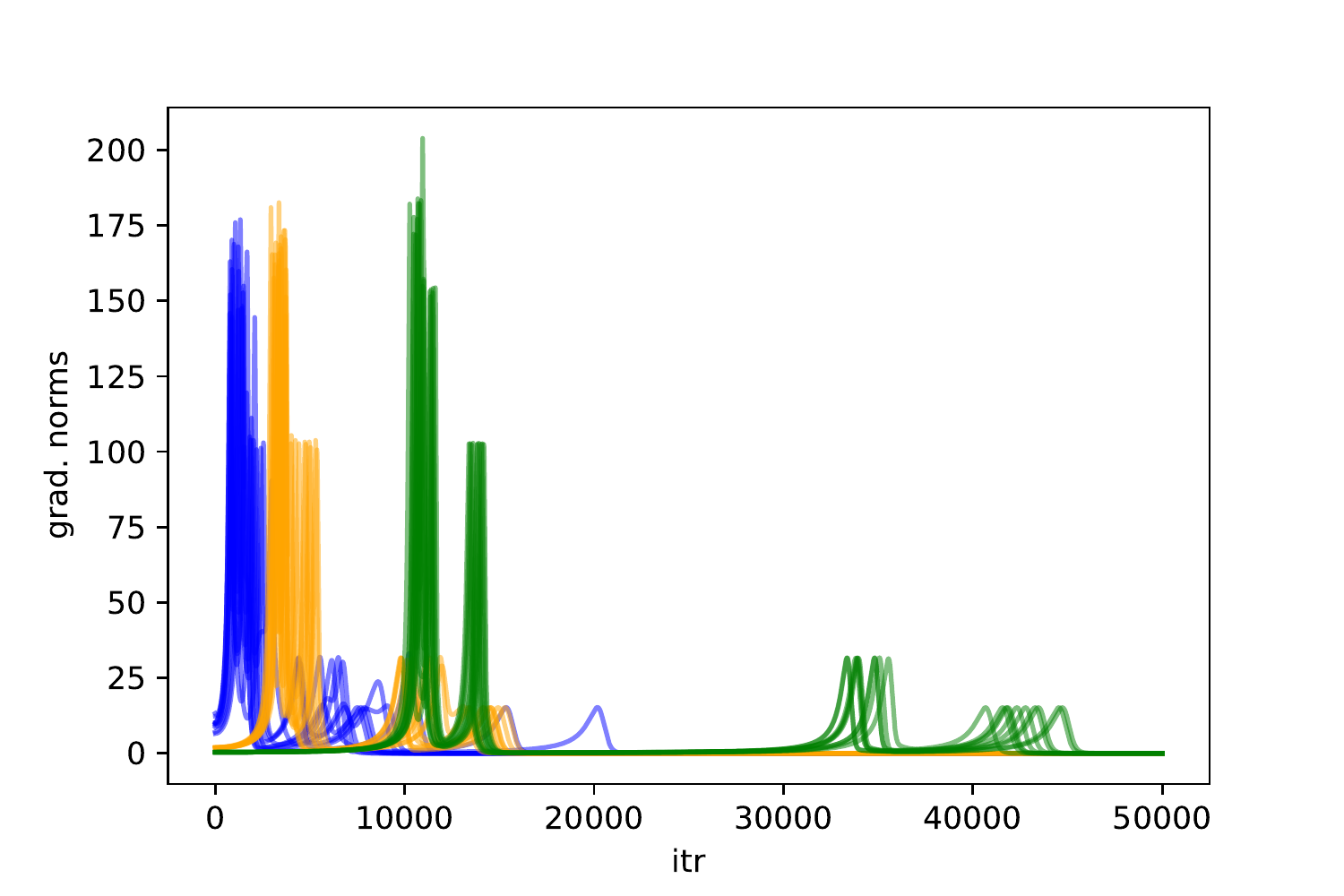}
      } \\
  {\scriptsize \textbf{(a)} $\gamma =0.75$ (NTK) \hspace{3.5cm} \textbf{(b)} $\gamma =1$ (MF) \hspace{3.5cm} \textbf{(c)} $\gamma =1.5$ (S-S)}
  \caption{\textit{Training in \textbf{(a)} the NTK regime, \textbf{(b)} mean-field, \textbf{(c)} saddle-to-saddle regimes in deep linear networks for three widths $w=10,100,1000$, $L=4$, and $10$ seeds; extension of Fig.~\ref{fig:regimes-of-training} in the main.} \textbf{Top:} The evolution of the rank of the network matrices during training. Tolerance of the matrix is set at $1e-1$.
  \textbf{Middle:} The evolution of the nuclear norm during training, we can see that the smooth jumps are aligned with the rank transitions.
  \textbf{Bottom:} The evolution of the gradient norm of the parameters. Decrease of the gradient norm down to zero indicates approaching to a saddle, and the following increase indicates escaping it.
  }
  \label{fig:regimes-of-training-app}
\end{figure*}

\begin{figure}[!h]
\centering
\subfloat{\includegraphics[width=0.245\textwidth]{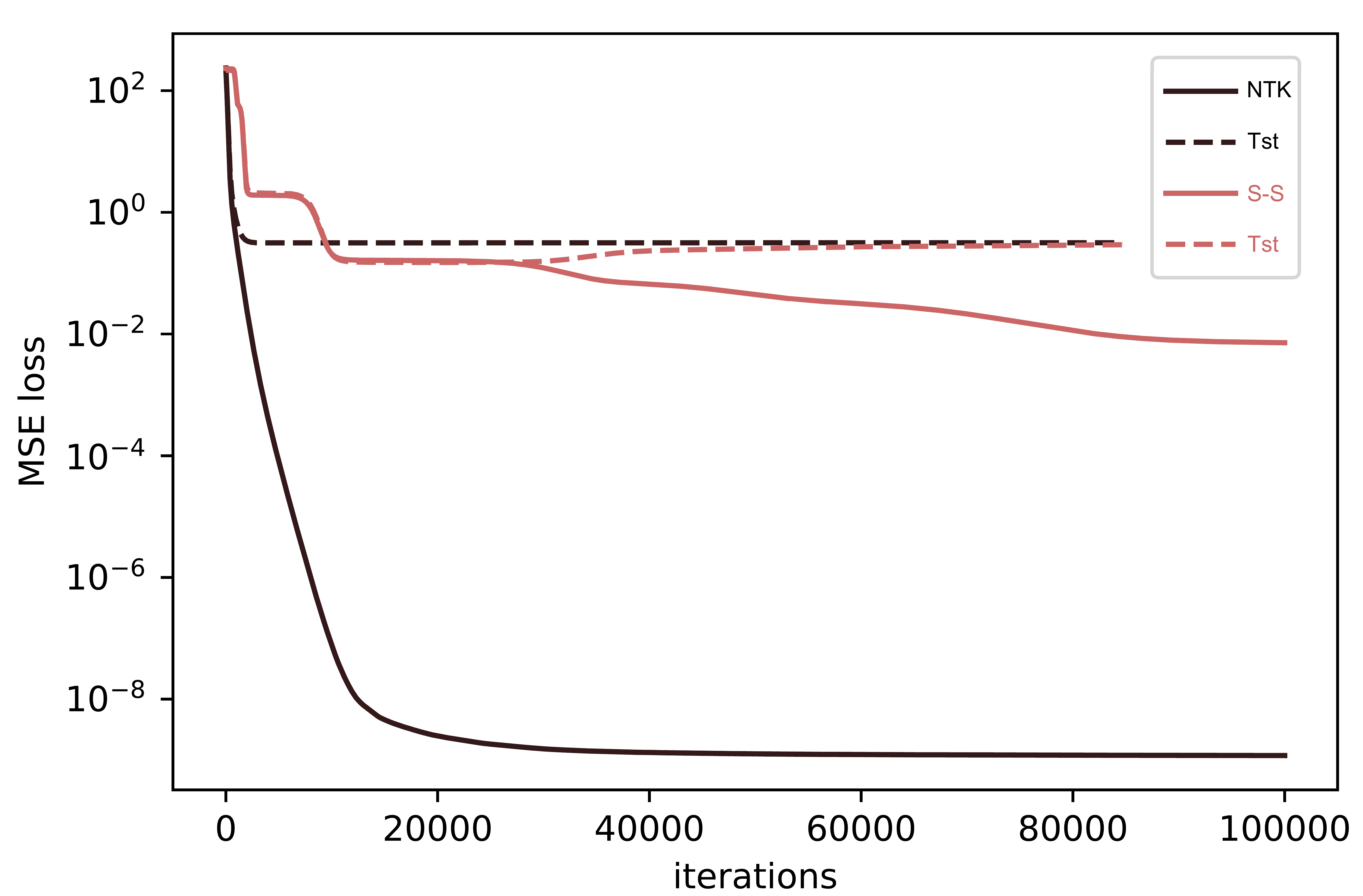}
}\subfloat{\includegraphics[width=0.245\textwidth]{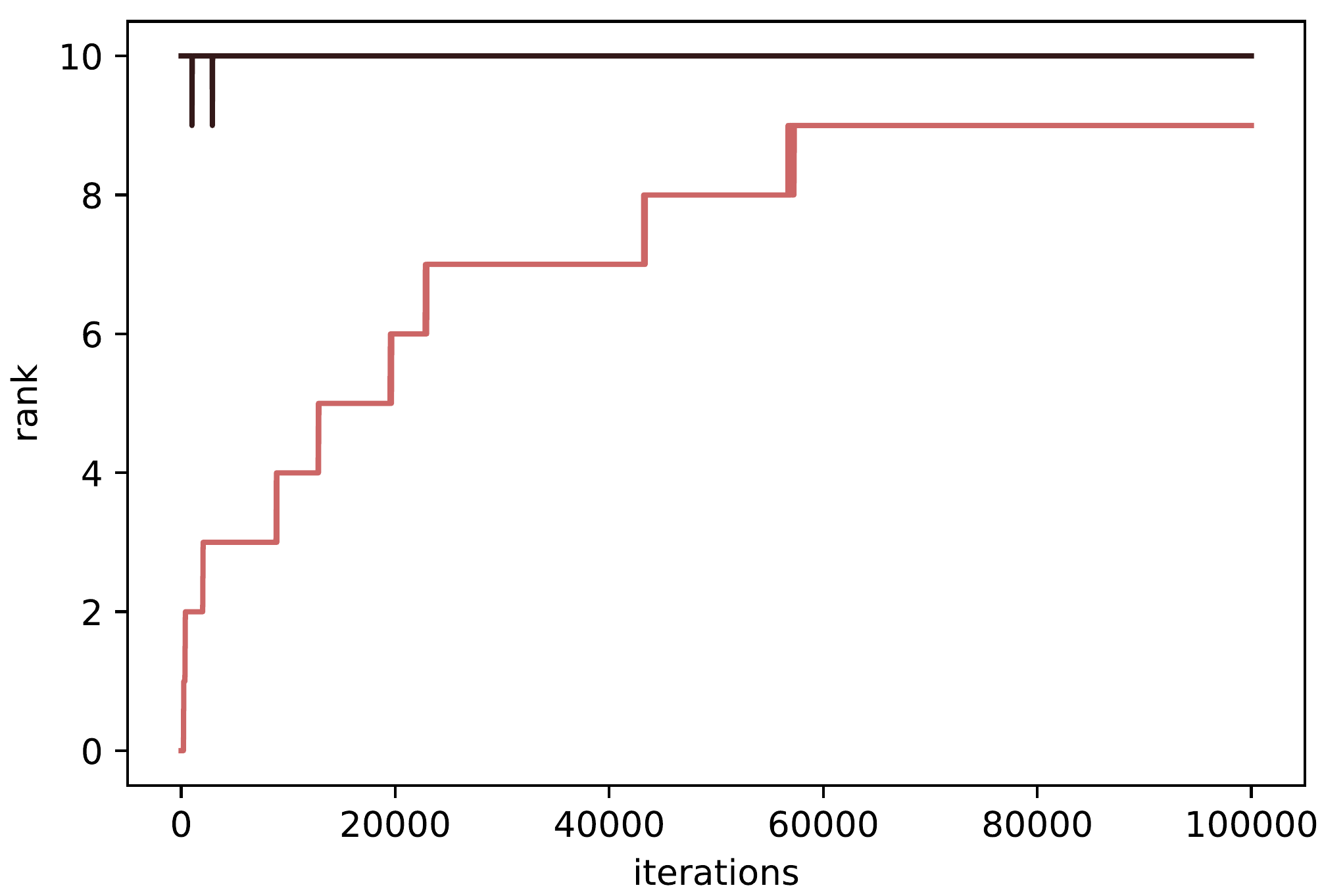} }\subfloat{\includegraphics[width=0.245\textwidth]{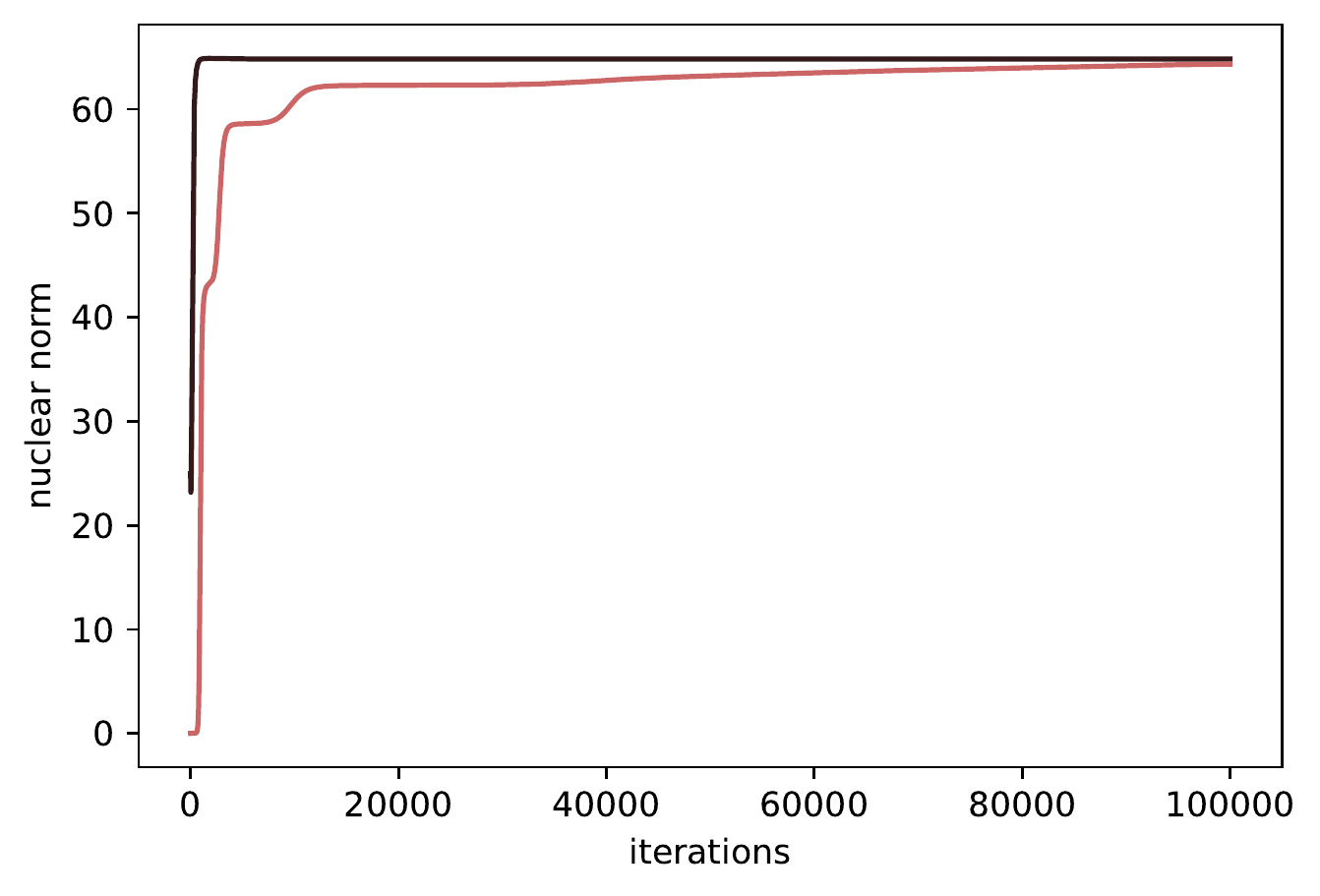} }\subfloat{\includegraphics[width=0.245\textwidth]{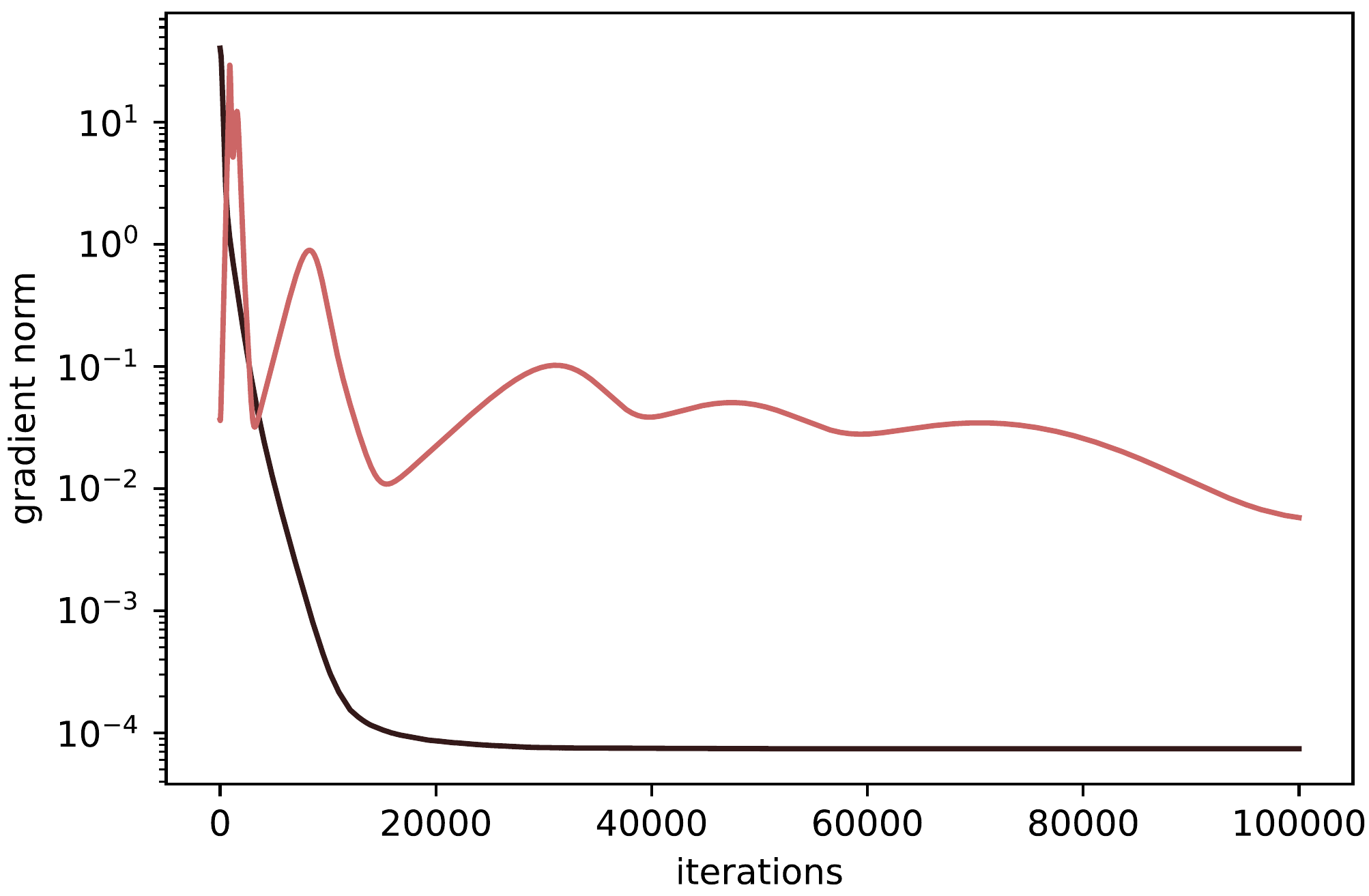} }\vfill{}
\subfloat{\includegraphics[width=0.245\textwidth]{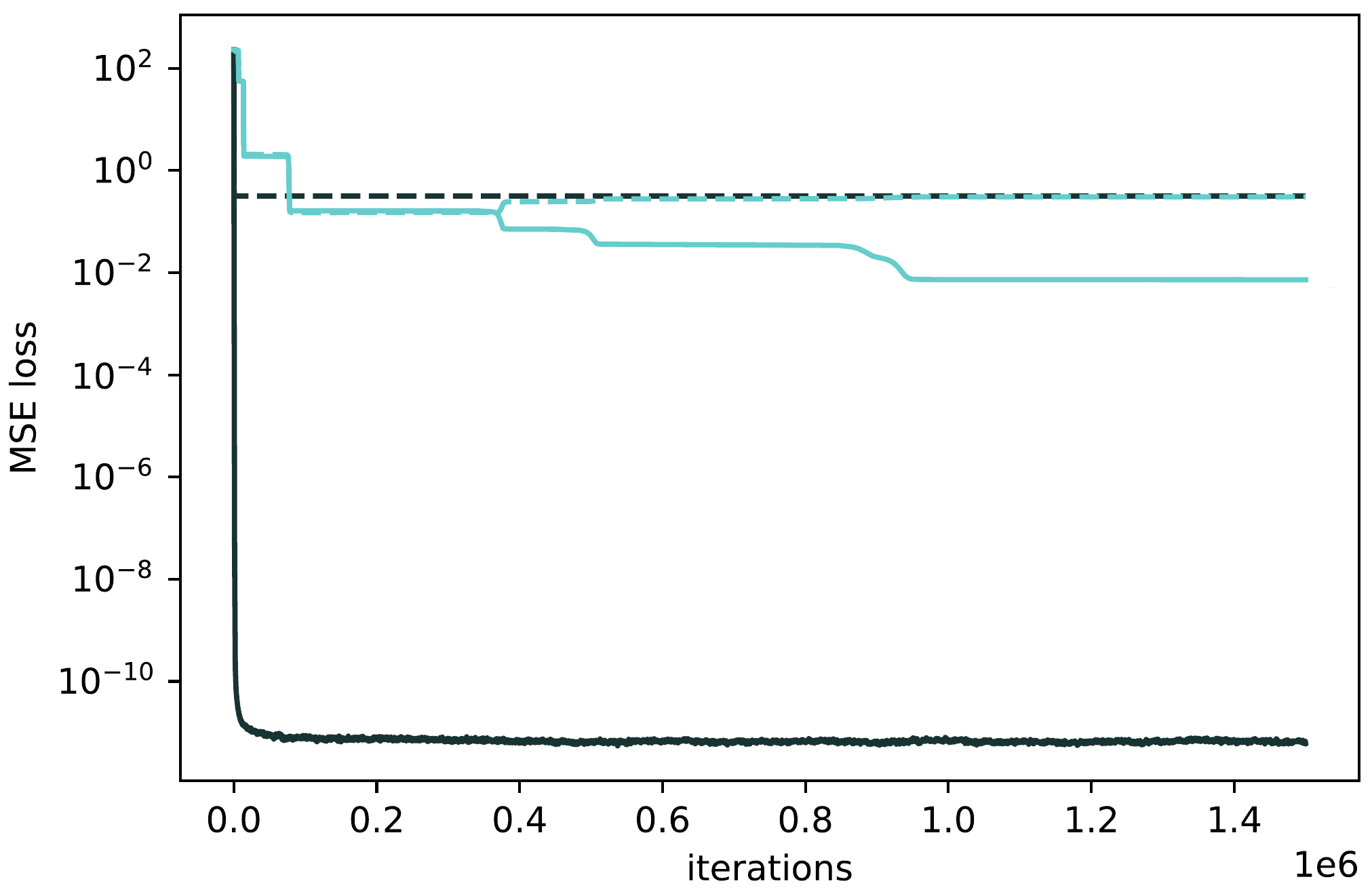}}
\subfloat{ \includegraphics[width=0.245\textwidth]{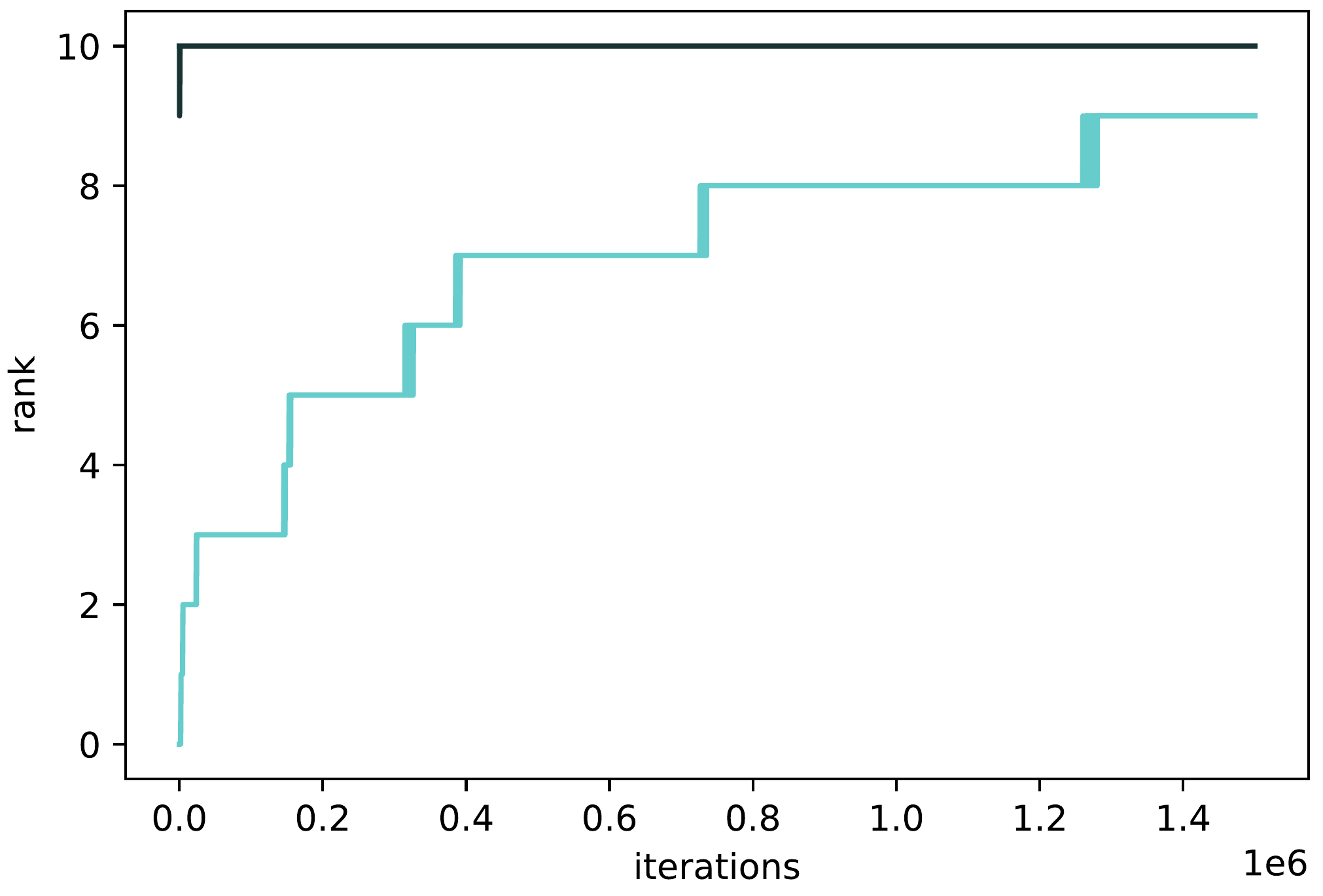} }\subfloat{ \includegraphics[width=0.245\textwidth]{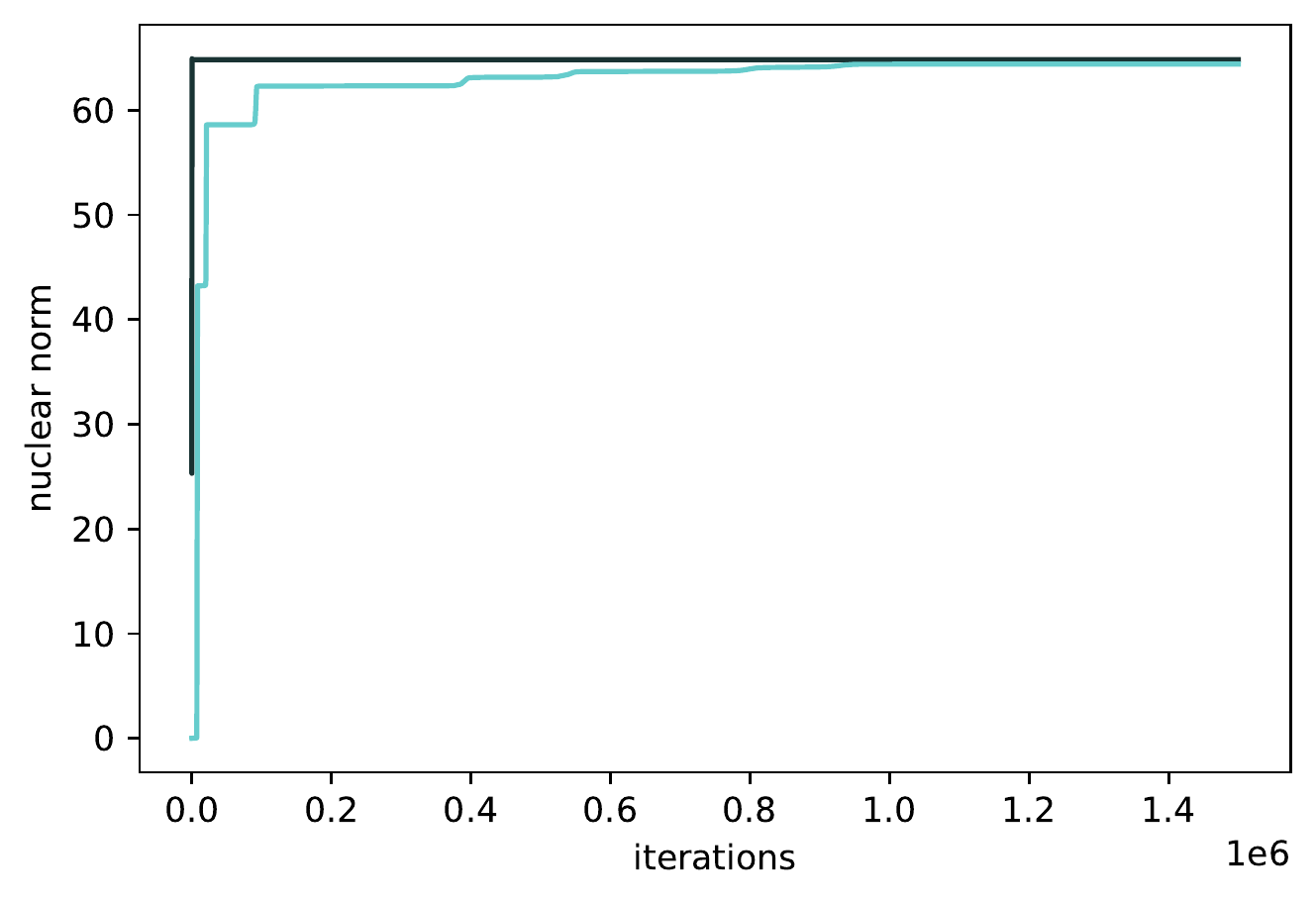} }\subfloat{ \includegraphics[width=0.245\textwidth]{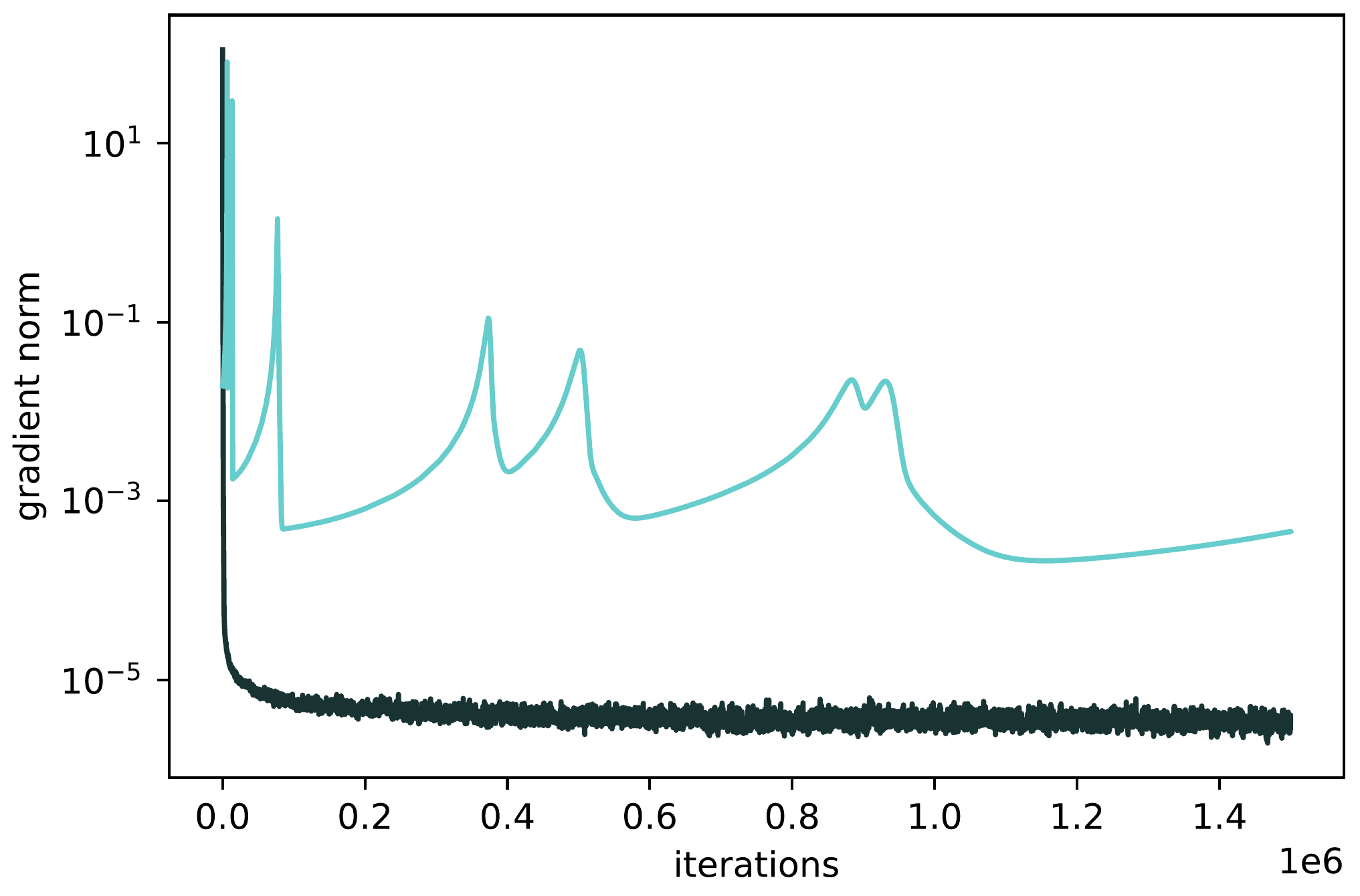} }\caption{\textit{Training in the NTK vs. saddle-to-saddle regimes in shallow (top) and deep (bottom) networks when learning a low rank matrix corrupted with noise.}
Black lines (the NTK regime): the parameters are initialized with the standard deviation $ \tilde{\sigma}=w^{-\nicefrac{L-1}{2L}} $.
The rank of the network matrix increases incrementally as the gradient trajectory follows the paths between the saddles.
\textbf{Top/Shallow case:} $ L = 2 $ and $ w = 50 $; in the saddle-to-saddle regime (shown in red), the initialization scale is $ \tilde{\sigma} = w^{-2} $.
Bigger initialization scales result in shorter plateaus in the loss curve if the same learning rate is used.
\textbf{Bottom/Deep case:} $ L = 4 $ and $ w = 100 $; in the saddle-to-saddle regime (shown in blue), the initialization scale is $ \tilde{\sigma} = w^{-1} $.
We observe that the transitions from saddles to saddles are sharper.
We observe that the
gradient norm of the parameters is highly non-monotonic; a decrease
down to $0$ indicates approaching to a saddle, and a following increase
indicates escaping it. We note that the peaks of the gradient norm
are sharper in the deep case, suggesting a different rate of escape.
In the NTK regime, the gradient norm decreases down to
$0$ monotonically. In the deep case the GD training is implemented
for $1500000$ iterations whereas in the shallow case it is only $100000$
iterations.The input data is standard Gaussian, the outputs are generated by a rank $ 3 $ teacher of size $ 10 \times 10 $ corrupted with noise, and the loss is MSE. \label{fig:shallow-deep}}
\end{figure}

\textbf{Experimental details of Fig.~\ref{fig:regimes-of-training}:} A teacher network matrix of size $5 \times 5$ is generated as $10 \mathrm{diag}([1, 2, 3, 4, 5])$. The input data is i.i.d. $5$-dim. standard Gaussian samples, and the number of training samples is $100$. The labels are generated by the teacher, no noise is added. Training is performed with gradient descent for $50000$ epochs and a learning rate of $1e-4$ is used.

\textbf{Experimental details of Fig.~\ref{fig:matrix-completion}:} A random matrix $A^*$ of size $30 \times 30$ is generated by multiplying two i.i.d. matrices of size $30\times1$ with i.i.d. standard Gaussian entries. $0.2$ of the entries of this matrix was accecible in training, and the training objecive is the squared difference between the (observed) entries of the linear network matrix and those of the matrix $A^*$.
The training is performed for $20000$ gradient descent iterations with a learning rate of $\eta_0 = 0.05$ if $\gamma>1$, and $\eta = \eta_0 w^{(L-1)(\gamma - 1)}$ for $\gamma \leq 1$. The tolerance for computing the rank is set to $0.1$.

\textbf{Experimental details of Fig.~\ref{fig:shallow-deep} in the Appendix:} We created
a rank $3$ teacher weight matrix $W_{T}=W_{0}W_{0}^{T}$ of size
$10\times10$ where $W_{0}$ is a $10\times3$ matrix with all entries
independent Gaussian with zero mean and where all entries in $i$-th
column has variance $i$ for all $i\in\{1,2,3\}$. We corrupted the
teacher weight matrix by an addition of a $10\times10$ matrix where
each entry is i.i.d. centered Gaussian with standard deviation $0.2$.
Input points are isotropic Gaussians. The training outputs are generated
by the noisy teacher, and the test outputs are generated by the noiseless
teacher. We generated $100$ training and $1000$ test data points.
Different runs of the same experiment yielded effectively the same
figure. The learning rate is $0.001$ both for the shallow and the
deep case. Tolerance for the rank is set to $10^{-4}$ (i.e. eigenvalues
smaller than $10^{-4}$ are set to $0$ for the rank calculation).


\section{Regimes of Training\label{sec:Regimes-of-Training}}

In this section we describe the regimes of training depending on the
scaling $\gamma$ of the variance at initialization $\sigma^{2}=w^{-\gamma}$.

\subsection{Equivalence of Parametrization/Initializations}\label{subsec:equivalence-parametrization}

\subsubsection{NTK Parametrization}\label{subsec:equivalence-parametrization-NTK}

Let us show that the NTK parametrization corresponds to a scaling
of $\gamma=1-\frac{1}{L}$.

The NTK parametrization \cite{jacot2018neural} for linear networks
is
\[
A_{\theta}^{NTK}=\frac{W_{L}}{\sqrt{n_{L-1}}}\cdots\frac{W_{1}}{\sqrt{n_{0}}}=\frac{1}{\sqrt{n_{0}\cdots n_{L-1}}}W_{L}\cdots W_{1}
\]
with all parameters initialized with a variance of $1$. One can show
that gradient flow $\theta^{NTK}(t)$ with the NTK parametrization,
initialized at some parameters $\theta_{0}^{NTK}$ is equivalent (up
to a rescaling of the learning rate) to gradient flow $\theta(t)$
with the classical parametrization with an initialization of $\theta_{0}=\left(n_{0}\cdots n_{L-1}\right)^{-\frac{1}{2L}}\theta_{0}^{NTK}$:
\begin{prop}
Let $\theta^{NTK}(t)$ be gradient flow on the loss $\mathcal{L}^{NTK}(\theta)=C(A_{\theta}^{NTK})$
initialized at some parameters $\theta_{0}^{NTK}$ and $\theta(t)$
be gradient flow on the cost $\mathcal{L}(\theta)=C(A_{\theta})$
initialized at $\theta_{0}=\left(n_{0}\cdots n_{L-1}\right)^{-\frac{1}{2L}}\theta_{0}^{NTK}$.
We have
\[
A_{\theta(t)}=A_{\theta^{NTK}(\sqrt{n_{0}\cdots n_{L-1}}t)}^{NTK}.
\]
\end{prop}

\begin{proof}
We will show that $\theta(t)=\left(n_{0}\cdots n_{L-1}\right)^{-\frac{1}{2L}}\theta^{NTK}(\sqrt{n_{0}\cdots n_{L-1}}t)$
which implies that $A_{\theta(t)}=A_{\theta^{NTK}(t)}^{NTK}$. This
is obviously true at $t=0$. Now assuming it is true at a time $t$,
we show that the time derivatives of $\theta(t)$ and $\left(n_{0}\cdots n_{L-1}\right)^{-\frac{1}{2L}}\theta^{NTK}(\sqrt{n_{0}\cdots n_{L-1}}t)$
match:
\[
\partial_{t}\theta^{NTK}(\sqrt{n_{0}\cdots n_{L-1}}t)=\frac{\sqrt{n_{0}\cdots n_{L-1}}}{\sqrt{n_{0}\cdots n_{L-1}}}\partial_{t}\theta(t)=\partial_{t}\theta(t).
\]
\end{proof}
This implies that the NTK parametrization with $\mathcal{N}(0,1)$
initialization is equivalent to the classical parametrization with
$\mathcal{N}(0,\left(n_{0}\cdots n_{L-1}\right)^{-\frac{1}{L}})$
initialization, which for rectangular networks corresponds to a $\mathcal{N}(0,n_{0}^{-\frac{1}{L}}w^{-\frac{L-1}{L}})$
initialization with scaling $\gamma=\frac{L-1}{L}=1-\frac{1}{L}$.

\subsubsection{Maximal Update Parametrization}\label{subsec:equivalence-parametrization-MUP}

The Maximal Update parametrization (or $\mu$-parametrization) \cite{yang2020feature}
is equivalent to $\gamma=1$. The $\mu$-parametrization for linear
rectangular networks is the same the classical one, since
\[
A_{\theta}^{\mu}=\frac{W_{L}}{\sqrt{w}}W_{L-1}\cdots W_{2}\left(\sqrt{w}W_{1}\right)=W_{L}\cdots W_{1}
\]
 and the parameters are initialized with variance $w^{-1}$, i.e.
$\gamma=1$.

\subsection{Distance to Different Critical Points}

Let $d_{\mathrm{m}}$ and $d_{\mathrm{s}}$ be the Euclidean distances between the initialization
$\theta$ and, respectively, the set of global minima and  the set of all saddles. For random variables $f(w),g(w)$ which depend on $w$, we write $f\asymp g$ if both $\nicefrac{f(w)}{g(w)}$ and $\nicefrac{g(w)}{f(w)}$ are stochastically bounded as $w\to\infty$.
The following theorem studies how $d_{\mathrm{m}}$ and $d_{\mathrm{s}}$  scale as $w\to\infty$:
\begin{thm}[Theorem \ref{th:distances} in the main]\label{prop:proximity-critical-points}
Suppose that the set of matrices that minimize $C$ is non-empty, has Lebesgue measure zero, and does not contain the zero matrix. Let $\theta$ be i.i.d. centered Gaussian r.v. of variance
$\sigma^{2}=w^{-\gamma}$ where $1-\frac{1}{L}\leq\gamma<\infty$. Then:
\begin{enumerate}
\item if $1-\frac{1}{L}\leq\gamma<1$, we have $d_{\mathrm{m}}\asymp w^{-\frac{(1-\gamma)(L-1)}{2}}$
and $d_{\mathrm{s}}\asymp w^{\frac{1-\gamma}{2}}$,
\item if $\gamma=1$, we have $d_{\mathrm{m}},d_{\mathrm{s}} \asymp 1$,
\item if $\gamma>1$ we have $d_{\mathrm{m}}\asymp 1$ and $d_{\mathrm{s}}\asymp w^{-\frac{\gamma-1}{2}}$.
\end{enumerate}
\end{thm}

To prove this result, we require a few Lemmas:
\begin{lem}
\label{lem:distance-matrix-init-min}Let $\theta$ be the vector of
parameters of a DLN with i.i.d. $\mathcal{N}(0,w^{-\gamma})$ Gaussian
entries, and let $\mathcal{A}_{\mathrm{min}}=\{A\in\mathbb{R}^{n_{L}\times n_{0}}:C(A)=0\}$
be the set of global minimizers of $C$. Under the same assumptions
on the cost $C$ as Proposition \ref{prop:proximity-critical-points},
we have $d(A_{\theta},\mathcal{A}_{\mathrm{min}})\asymp1$ as $w\to\infty$.
\end{lem}

\begin{proof}
If $\gamma>1-\frac{1}{L}$ then $A_{\theta}$ converges in distribution
to the zero matrix as $w\to\infty$, the distance $d(A_{\theta},\mathcal{A}_{\mathrm{min}})$
therefore converges to the distance finite value $d(0,\mathcal{A}_{\mathrm{min}})\neq0$.

If $\gamma=1-\frac{1}{L}$, then $A_{\theta}$ converges in distribution
to random Gaussian matrix with iid $\mathcal{N}(0,1)$ entries (this
can seen as a consequence of the more general results for non-linear
networks \cite{gaussianconv_lee2017,Matthews2018GaussianProcess}).
As a result the distribution of $d(A_{\theta},\mathcal{A}_{\mathrm{min}})$
converges to the distribution of $d(B,\mathcal{A}_{\mathrm{min}})$
for a matrix $B$ with iid Gaussian $\mathcal{N}(0,1)$ entries. Since
$\mathbb{P}\left[d(B,\mathcal{A}_{\mathrm{min}})=0\right]=0$ and
$\mathbb{P}\left[d(B,\mathcal{A}_{\mathrm{min}})>b\right]\to0$ as
$b\to\infty$ we have that $d(A_{\theta},\mathcal{A}_{\mathrm{min}})\asymp1$
as needed.
\end{proof}
\begin{lem}
\label{lem:bound-distance-output-from-distance-parameters}Let $\theta$
be the vector of parameters of a DLN with iid $\mathcal{N}(0,w^{-\gamma})$
Gaussian entries. For all $\epsilon$, there is a constant $C_{\epsilon,L}$
that does not depend on $w$ s.t. with prob. $1-\epsilon$, we have
for all $\theta'\in\mathbb{R}^{P}$ that

\[
\left\Vert A_{\theta'}-A_{\theta}\right\Vert _{F}^{2}\leq C_{\epsilon,L}\sum_{k=1}^{L}\left\Vert \theta-\theta'\right\Vert ^{2k}w^{(1-\gamma)(L-k)}.
\]
\end{lem}

\begin{proof}
By Corollary 5.35 in \cite{vershynin_2010_random_matrix_spectrum_bound},
reformulated as Theorem \ref{thm:random_matrix_operator_norm_bound}
below, we know that for all $\epsilon$, there is a constant $c_{\epsilon}$that
does not depend on $w$ s.t. with prob. $1-\epsilon$, we have for
all $\ell$
\[
\left\Vert W_{\ell}\right\Vert _{op}^{2}\leq c_{\epsilon}w^{1-\gamma}.
\]
We now write $d\theta=\theta'-\theta$ (and the corresponding matrices
$dW_{\ell}=W_{\ell}'-W_{\ell}$) so that we may write the difference
$A_{\theta+d\theta}-A_{\theta}$ as the following sum
\[
\sum_{\begin{array}{c}
a_{1},\dots,a_{L}\in\{0,1\}\\
\exists\ell,a_{\ell}\neq0
\end{array}}\left(\begin{cases}
W_{L} & \text{if \ensuremath{a_{L}=0}}\\
dW_{L} & \text{if \ensuremath{a_{L}=1}}
\end{cases}\right)\cdots\left(\begin{cases}
W_{1} & \text{if \ensuremath{a_{1}=0}}\\
dW_{1} & \text{if \ensuremath{a_{1}=1}}
\end{cases}\right)
\]
where the indicator $a_{\ell}$ determines whether we take $W_{\ell}$
or $dW_{\ell}$ in the product. We can therefore bound
\begin{align*}
\left\Vert A_{\theta+d\theta}-A_{\theta}\right\Vert _{F}^{2} & \leq\left(\sum_{\begin{array}{c}
a_{1},\dots,a_{L}\in\{0,1\}\\
\exists\ell,a_{\ell}\neq0
\end{array}}\left\Vert \left(\begin{cases}
W_{L} & \text{if \ensuremath{a_{L}=0}}\\
dW_{L} & \text{if \ensuremath{a_{L}=1}}
\end{cases}\right)\cdots\left(\begin{cases}
W_{1} & \text{if \ensuremath{a_{1}=0}}\\
dW_{1} & \text{if \ensuremath{a_{1}=1}}
\end{cases}\right)\right\Vert _{F}\right)^{2}\\
 & \leq(2^{L}-1)\sum_{\begin{array}{c}
a_{1},\dots,a_{L}\in\{0,1\}\\
\exists\ell,a_{\ell}\neq0
\end{array}}\left\Vert \left(\begin{cases}
W_{L} & \text{if \ensuremath{a_{L}=0}}\\
dW_{L} & \text{if \ensuremath{a_{L}=1}}
\end{cases}\right)\cdots\left(\begin{cases}
W_{1} & \text{if \ensuremath{a_{1}=0}}\\
dW_{1} & \text{if \ensuremath{a_{1}=1}}
\end{cases}\right)\right\Vert _{F}^{2}\\
 & \leq(2^{L}-1)\sum_{\begin{array}{c}
a_{1},\dots,a_{L}\in\{0,1\}\\
\exists\ell,a_{\ell}\neq0
\end{array}}\left(\begin{cases}
\left\Vert W_{L}\right\Vert _{op}^{2} & \text{if \ensuremath{a_{L}=0}}\\
\left\Vert dW_{L}\right\Vert _{F}^{2} & \text{if \ensuremath{a_{L}=1}}
\end{cases}\right)\cdots\left(\begin{cases}
\left\Vert W_{1}\right\Vert _{op}^{2} & \text{if \ensuremath{a_{1}=0}}\\
\left\Vert dW_{1}\right\Vert _{F}^{2} & \text{if \ensuremath{a_{1}=1}}
\end{cases}\right)
\end{align*}
We now bound $\left\Vert W_{L}\right\Vert _{op}^{2}$ by $c_{\epsilon}w^{1-\gamma}$
and $\left\Vert dW_{L}\right\Vert _{F}^{2}$ by $\left\Vert d\theta\right\Vert ^{2}$
so that we obtain the bound
\[
\left\Vert A_{\theta+d\theta}-A_{\theta}\right\Vert _{F}^{2}\leq(2^{L}-1)\sum_{k=1}^{L}\left(\begin{array}{c}
L\\
k
\end{array}\right)\left\Vert d\theta\right\Vert ^{2k}c_{\epsilon}^{L-k}w^{(1-\gamma)(L-k)}\leq C_{\epsilon,L}\sum_{k=1}^{L}\left\Vert d\theta\right\Vert ^{2k}w^{(1-\gamma)(L-k)}
\]
for $C_{\epsilon,L}=(2^{L}-1)\max_{k=1,\dots,L}\left(\begin{array}{c}
L\\
k
\end{array}\right)c_{\epsilon}^{L-k}$.
\end{proof}
Let us now prove Theorem \ref{prop:proximity-critical-points}:
\begin{proof}
\textbf{(1) Distance to minimum: }Let us first give an lower bound
on the distance from initialization to a global minimum. Let $\theta$
be the intialization and $\theta+d\theta$ be the closest minimum.
By Lemma \ref{lem:bound-distance-output-from-distance-parameters},
we obtain
\[
\left\Vert A_{\theta+d\theta}-A_{\theta}\right\Vert _{F}^{2}\leq C_{L}'\sum_{k=1}^{L}\left\Vert d\theta\right\Vert ^{2k}w^{(1-\gamma)(L-k)}.
\]
If $\gamma>1$, the term with $k=L$ dominates, in which case $\left\Vert A_{\theta+d\theta}-A_{\theta}\right\Vert _{F}^{2}\leq\left\Vert d\theta\right\Vert ^{2L}$
which implies that $\left\Vert d\theta\right\Vert \geq\left\Vert A_{\theta+d\theta}-A_{\theta}\right\Vert _{F}^{\frac{1}{L}}\geq d(A_{\theta},\mathcal{A}_{\mathrm{min}})^{\frac{1}{L}}\asymp1$
by Lemma \ref{lem:distance-matrix-init-min}.

If $\gamma<1$, the term $k=1$ dominates, which implies $\left\Vert A_{\theta+d\theta}-A_{\theta}\right\Vert _{F}^{2}\leq\left\Vert d\theta\right\Vert ^{2}w^{(1-\gamma)(L-1)}$
which implies that $\left\Vert d\theta\right\Vert \geq\left\Vert A_{\theta+d\theta}-A_{\theta}\right\Vert _{F}w^{-\frac{(1-\gamma)(L-1)}{2}}=O(w^{-\frac{(1-\gamma)(L-1)}{2}})$,
which decreases with width.

Let us now show upper bounds on $\left\Vert d\theta\right\Vert $.
When $\gamma>1$, we will construct a closeby minimum. Let us first
define the parameters $\bar{\theta}=(\bar{W}_{1},\dots,\bar{W}_{L})$
where $\bar{W}_{1}=0$ and $\bar{W}_{L}=0$ and $\bar{W}_{\ell,ij}=\begin{cases}
W_{\ell,ij} & \text{if \ensuremath{i,j>\min\{n_{0},n_{L}\}}}\\
0 & \text{otherwise}
\end{cases}$. Since we have set only $O(w)$ parameters to zero, we have $\left\Vert \theta-\bar{\theta}\right\Vert ^{2}=O(\sigma^{2}w)=O(w^{1-\gamma})$.
Now let the matrix $A$ be a global minimum of the cost $C$ with
SVD $A=USV^{T}$ (with inner dimension equal to the rank $k$ of $A$),
we then set $\theta^{*}=\bar{\theta}+I^{(k\to w)}(S^{\frac{1}{L}}V^{T},S^{\frac{1}{L}},\dots,S^{\frac{1}{L}},US^{\frac{1}{L}})$.
The parameters $\theta^{*}$ are a global minimum since $A_{\theta^{*}}=A$
and $\left\Vert \theta^{*}-\theta\right\Vert \leq\left\Vert \theta^{*}-\bar{\theta}\right\Vert +\left\Vert \bar{\theta}-\theta\right\Vert =O(1)+O(w^{\frac{1-\gamma}{2}})=O(1)$.

When $\gamma<1$, with prob. $1-\epsilon$, we have $s_{min}\left(W_{L-1}\cdots W_{1}\right)>\frac{1}{2}\sigma^{(L-1)}w^{\frac{L-1}{2}}=w^{\frac{(1-\gamma)(L-1)}{2}}$,
we can reach a global minimum by only changing $W_{L}$, we need $dW_{L}W_{L-1}\cdots W_{1}=A^{*}-A_{\theta}$
hence we take $dW_{L}=\left(A^{*}-A_{\theta}\right)\left(W_{L}\cdots W_{1}\right)^{+}$
with norm $\left\Vert d\theta\right\Vert =\left\Vert dW_{L}\right\Vert _{F}\leq\frac{\left\Vert A^{*}-A_{\theta}\right\Vert }{s_{min}\left(W_{L-1}\cdots W_{1}\right)}=O(w^{-\frac{(1-\gamma)(L-1)}{2}}).$

\textbf{(2) Distance to saddles: }Given parameters $\theta=(W_{1},\dots,W_{L})$,
we can obtain a saddle $\theta^{*}$ by setting all entries of $W_{1}$
and $W_{L}$ to zero. We have
\[
\mathbb{E}\left[\left\Vert \theta-\theta^{*}\right\Vert ^{2}\right]=\mathbb{E}\left[\left\Vert W_{1}\right\Vert _{F}^{2}\right]+\mathbb{E}\left[\left\Vert W_{L}\right\Vert _{F}^{2}\right]=\sigma^{2}(n_{0}+n_{L})w=O(w^{1-\gamma}).
\]
This gives an upper bound of order $w^{1-\gamma}$ on the distance
between $\theta$ and the set of saddles $\theta^{*}$ with $A_{\theta^{*}}=0$.

Now let $\theta^{*}=\theta+d\theta$ be the saddle closest to $\theta$,
we know that
\begin{align*}
0 & =\partial_{W_{L}}\mathcal{L}(\theta^{*})=\nabla C(A_{\theta^{*}})\left(W_{1}^{*}\right)^{T}\cdots\left(W_{L-1}^{*}\right)^{T}.
\end{align*}
Since $A_{\theta^{*}}$ is not a global minimum, $\nabla C(A_{\theta^{*}})\neq0$,
for the above to be zero, we therefore need $\left(W_{1}^{*}\right)^{T}\cdots\left(W_{L-1}^{*}\right)^{T}$
to not have full column rank, i.e. $\mathrm{Rank}\left(W_{1}^{*}\right)^{T}\cdots\left(W_{L-1}^{*}\right)^{T}=n_{0}$.

We will show that at initialization $\left(W_{1}\right)^{T}\cdots\left(W_{L-1}\right)^{T}$
has rank $n_{0}$ and its smallest non-zero singular value $s_{min}$
is of order $w^{\frac{(1-\gamma)(L-1)}{2}}$. We will use the fact
that $\left\Vert \left(W_{1}\right)^{T}\cdots\left(W_{L-1}\right)^{T}-\left(W_{1}^{*}\right)^{T}\cdots\left(W_{L-1}^{*}\right)^{T}\right\Vert _{F}\geq s_{min}$
to lower bound the distance $\left\Vert \theta-\theta^{*}\right\Vert $
using Lemma \ref{lem:bound-distance-output-from-distance-parameters}.

The singular values of $W_{1}^{T}\cdots W_{L-1}^{T}$ are the squared
root of the eigenvalues of the $n_{0}\times n_{0}$ matrix $W_{1}^{T}\cdots W_{L-1}^{T}W_{L-1}\cdots W_{1}$.
One can show that as $w\to\infty$ this matrix concentrates in its
expectation
\[
\mathbb{E}\left[W_{1}^{T}\cdots W_{L-1}^{T}W_{L-1}\cdots W_{1}\right]=\sigma^{2(L-1)}w^{L-1}=w^{(1-\gamma)(L-1)}.
\]
which implies that $s_{\min}$ concentrates in $w^{\frac{(1-\gamma)(L-1)}{2}}$
and therefore $s_{\mathrm{min}}\asymp w^{\frac{(1-\gamma)(L-1)}{2}}$.

Now by Lemma \ref{lem:bound-distance-output-from-distance-parameters}
(applied to the depth $L-1$ this time), we have with prob. $1-\epsilon$
\begin{align*}
s_{\mathrm{min}}^{2} & \leq\left\Vert \left(W_{1}\right)^{T}\cdots\left(W_{L-1}\right)^{T}-\left(W_{1}^{*}\right)^{T}\cdots\left(W_{L-1}^{*}\right)^{T}\right\Vert _{F}^{2}\\
 & \leq C_{\epsilon,L-1}\sum_{k=1}^{L-1}\left\Vert \theta-\theta'\right\Vert ^{2k}w^{(1-\gamma)(L-1-k)}
\end{align*}
and $\left\Vert \theta-\theta'\right\Vert $ needs to be at least
of order $w^{\frac{(1-\gamma)}{2}}$ for any of the terms in the sum
to be at least of order $w^{(1-\gamma)(L-1)}$ (actually all these
become of the right order at the same time).
\end{proof}

\subsection{Spectrum bounds}

An important tool in our analysis is the following Theorem (which
is a reformulation of Corollary 5.35 in \cite{vershynin_2010_random_matrix_spectrum_bound})
\begin{thm}
\label{thm:random_matrix_operator_norm_bound} Let $A$ be a $m\times n$
matrix with i.i.d. $\mathcal{N}(0,\sigma^{2})$ entries. For all $t\geq0$,
with probability at least $1-2e^{-\frac{t^{2}}{2}}$, it holds that
\[
\sigma(-\sqrt{m}-\sqrt{n}-t)\leq s_{min}\left(A\right)\leq s_{max}\left(A\right)\leq\sigma\left(\sqrt{m}+\sqrt{n}+t\right).
\]
\end{thm}

\begin{cor}
If the parameters $\theta$ are independent centered Gaussian with
variance $\sigma^{2}$, for all $t\geq0$, with probability at least
$1-2Le^{-\frac{t^{2}}{2}}$, it holds that
\[
\left\Vert A_{\theta}\right\Vert _{op}\leq(1+t)^{L}\sigma^{L}\left(\sqrt{n_{0}}+\sqrt{w}\right)(4w)^{\frac{L-2}{2}}\left(\sqrt{w}+\sqrt{n_{L}}\right).
\]
\end{cor}

\begin{proof}
By Theorem \ref{thm:random_matrix_operator_norm_bound}, with probability
greater than $1-2Le^{-\frac{t^{2}}{2}}$, for all $\ell=1,\ldots,L$,
$\left\Vert W_{\ell}\right\Vert _{op}\leq\sigma\left(\sqrt{n_{\ell-1}}+\sqrt{n_{\ell}}+t\right),$
where $n_{\ell}=w$ for $\ell\in\{1,\cdots,L-1\}$. Hence
\[
\left\Vert A_{\theta}\right\Vert _{op}\leq\left\Vert W_{L}\right\Vert _{op}\cdots\left\Vert W_{1}\right\Vert _{op}\leq\sigma^{L}\prod_{\ell=1}^{L}\left(\sqrt{n_{\ell-1}}+\sqrt{n_{\ell}}+t\right)\leq(1+t)^{L}\sigma^{L}\prod_{\ell=1}^{L}\left(\sqrt{n_{\ell-1}}+\sqrt{n_{\ell}}\right).
\]
\end{proof}

\section{Proofs for the Saddle-to-Saddle regime}

\label{sec:proofs-saddle-to-saddle}

In this section, we prove Theorem 4 of the main. Given a saddle $\vartheta^{*}=RI^{(k\to w)}(\vartheta)$
where $\vartheta$ is a local minimum in a width $k$ network, we
want to describe the dynamics of gradient descent $\theta_{\alpha}(t)=\gamma(t,\vartheta^{*}+\alpha\theta_{0})$,
initialized close to $\vartheta^{*}$. We shall consider $\vartheta^{*}=0$
for convenience, though the same arguments could be applied for $\vartheta^{*}\neq0$.
We will start by studying the case of homogeneous costs, which will
allow us to describe costs that locally look homogeneous around $0$.
Later on, after having defined the notion of \emph{escape paths,}
we will show that as $\alpha\to0$, the path $\left(\theta_{\alpha}(t)\right)_{t\in\mathbb{R}_{+}}$
converges to an escape path with specific direction and speed. We
will then show that the escape paths which escape at this speed are
unique in some aspects.

\subsection{Homogeneous Costs}

As in the main text, we use $\theta$ to denote an element in the
parameter space $\mathbb{R}^{P}$. Let $k\geq2$ be an integer. We
say that a cost $H$ is $k$-homogeneous if $H(\alpha\theta)=\alpha^{k}H(\theta)$
for all $\theta\in\mathbb{R}^{P}$ and all scalar $\alpha>0$. Later
in this paper, we will be particularly interested in the case where
$H(\theta)=\mathrm{Tr}\left[GA_{\theta}\right]$ for a linear network
$A_{\theta}$ of depth $L$ and some $n_{L}\times n_{0}$ matrix $G$.
Thus defined, $H$ is a $L$-homogeneous polynomial.

Throughout, when studying a $k$-homogeneous cost $H$, we will always
assume that it is twice differentiable.

A useful property of gradient descent on a homogeneous cost is that:
\begin{lem}
\label{lem:invariance-GD-homogeneous-cost} Gradient flow on a twice-differentiable
$k$-homogeneous cost $H$ satisfies
\[
\gamma_{H}(t,\lambda\theta_{0})=\lambda\gamma_{H}(\lambda^{k-2}t,\theta_{0})
\]
for all $\theta_{0}\in\mathbb{R}^{P},$all $\lambda>0$ and all $t\geq0$.
\end{lem}

\begin{proof}
We simply need show that for all $\theta_{0}\in\mathbb{R}^{P},\lambda>0,t\geq0$,
we have $\frac{1}{\lambda}\gamma_{H}(\lambda^{2-k}t,\lambda\theta_{0})=\gamma_{H}(t,\theta_{0})$,
i.e. that the path $t\mapsto\frac{1}{\lambda}\gamma_{H}(\lambda^{2-k}t,\lambda\theta_{0})$
is the solution of gradient descent starting at $\theta_{0}$. Clearly,
the path starts at $\theta_{0}$ and satisfies
\[
\partial_{t}\frac{1}{\lambda}\gamma_{H}(\lambda^{2-k}t,\lambda\theta_{0})=-\lambda^{1-k}\nabla H\left(\gamma_{H}(\lambda^{2-k}t,\lambda\theta_{0})\right)=-\nabla H\left(\frac{1}{\lambda}\gamma_{H}(\lambda^{2-k}t,\lambda\theta_{0})\right)
\]
since, using the fact that $H$ is $k$-homogeneous, for all scalar
$\alpha>0$, and any $\theta\in\mathbb{R}^{P}$, $\alpha\nabla H(\alpha\theta)=\alpha^{k}\nabla H(\theta)$.
One concludes using Picard-Lindelöf Theorem, using that $\nabla H$
is locally Lipschitz around $0$ since $H$ is twice differentiable.
\end{proof}
An \textbf{Escape Direction} at $0$ of $H$ is a vector on the sphere
$\rho\in\mathbb{S}^{P-1}$ such that $\nabla H(\rho)=-s\rho$ for
some $s\in\mathbb{R}_{+}$ which we call the \emph{escape speed} associated
with $\rho$. A path $(\theta(t))_{t<0}$ indexed by negative times
and following gradient flow on $H$ such that $\theta(t)$ is on one
escape direction for some $t<0$ will remain along this direction
(these paths are equal for $t<0$ to $\theta(t)=\rho e^{st}$ when
$k=2$ and $\theta(t)=\rho\left(-(k-2)st\right)^{-\frac{1}{k-2}}$
when $k>2$). Note that this entails that $\theta(t)\to0$ as $t\to-\infty$.
When $H(\theta)=\theta^{T}A\theta$ for some symmetric matrix $A$,
the escape directions are simply the eigenvectors of the Hessian $A$
and the escape speeds are twice the eigenvalues of $A$.

An \textbf{Optimal Escape Direction $\rho^{*}\in\mathbb{S}^{P-1}$
}is an escape direction with the largest speed $s^{*}>0$. It is a
minimizer of $H$ restricted to $\mathbb{S}^{P-1}$:
\begin{equation}
\rho^{*}\in\arg\min_{\rho\in\mathbb{S}^{P-1}}H(\rho).\label{eq:escape_min_min_H}
\end{equation}
Indeed, critical points of $H(\rho)$ restricted to the sphere are
the escape directions, and by Euler's condition (i.e. $\nabla H(\rho)^{T}\rho=kH(\rho)$
if $H$ is $k$-homogeneous), if $\rho$ is an escape direction with
speed $s$, then $H(\rho)=-\frac{s}{k}$: optimal escape directions
are thus global minimizers of $H$ restricted on the unit sphere.

Under some conditions on the Hessian along the escape directions,
one can guarantee that gradient descent will escape along an optimal
escape path:
\begin{prop}
\label{prop:convergence_in_direction_condition} Assume that the optimal
escape speed $s^{*}$ is positive and that for all escape directions
$\rho$ which are not optimal, there is a vector $v\perp\rho$ such
that $v^{T}\mathcal{H}H(\rho)v<-s^{*}v^{T}v$. Let $\Omega$ be the
set of $\theta_{0}$ such that the direction $\frac{\gamma(t,\theta_{0})}{\left\Vert \gamma(t,\theta_{0})\right\Vert }$
of the gradient descent flow converges towards an optimal escape direction
as $t\to T$, where $T$ is the explosion time of the path (which
can be infinite). The set $\mathbb{S}^{P-1}\setminus\Omega$ has spherical
measure zero.
\end{prop}

\begin{proof}
Let $\Omega'\subset\mathbb{R}^{P}$ be the set of points $\theta\neq0$
such that gradient flow on the 0-homogeneous cost $\bar{H}(\theta)=H\left(\frac{\theta}{\left\Vert \theta\right\Vert }\right)$
converges to a global minimum. Our proof is divided in two steps:
(1) we show that $\Omega'\subset\Omega$, (2) we show that $\Omega'\cap\mathbb{S}^{P-1}$
has spherical measure $1$.

(1) Note that both sets $\Omega$ and $\Omega'$ are cones: for any
$\alpha>0$, $\Omega=\alpha\Omega$ and $\Omega'=\alpha\Omega'$.
Therefore, we only need to show that $\Omega'\cap\mathbb{S}^{P-1}\subset\Omega\cap\mathbb{S}^{P-1}$.
Besides, note that since $\bar{H}$ is $0$-homogenous, $\nabla\bar{H}(\theta)^{T}\theta=0$
for all $\theta\in\mathbb{R}^{P}$ and thus, the norm is an invariant
of the descent gradient flow for $\bar{H}$: for any $\theta_{0}\in\mathbb{R}^{P}$
, $t\mapsto\left\Vert \gamma_{\bar{H}}(t,\theta_{0})\right\Vert $
is constant.

In particular, if $\theta_{0}\in\Omega'\cap\mathbb{S}^{P-1}$, then
$\bar{\theta}(t)=\gamma_{\bar{H}}(t,\theta_{0})$ converges to an
optimal escape direction $\rho^{*}$ as $t\to\infty$, by Equation
(\ref{eq:escape_min_min_H}). The gradient flow path $\theta(t)$
can be obtained from the gradient flow path $\bar{\theta}(s)$ directly.
First we define the function $\alpha(s)=e^{-k\int_{0}^{s}H(\bar{\theta}(r))dr}$
such that
\begin{align*}
\partial_{s}\left[\bar{\theta}(s)\alpha(s)\right] & =-(I-\bar{\theta}(s)\bar{\theta}(s)^{T})\nabla H(\bar{\theta}(s))\alpha(s)-k\bar{\theta}(s)H(\bar{\theta}(s))\alpha(s)\\
 & =-\nabla H(\bar{\theta}(s))\alpha(s)
\end{align*}
where we used the fact that $\theta^{T}\nabla H(\theta)=kH(\theta)$.
Let us now define $\tau(t)=\int_{0}^{t}\alpha(s)^{k-2}ds$, we have
\[
\bar{\theta}(\tau(t))\alpha(\tau(t))=-\nabla H(\bar{\theta}(\tau(t)))\alpha(\tau(t))^{k-1}=-\nabla H\left(\bar{\theta}(\tau(t))\alpha(\tau(t))\right)
\]
which implies that $\theta(t)=\bar{\theta}(\tau(t))\alpha(\tau(t))$.
As $r\to\infty$, we have $H(\bar{\theta}(r))\to-s^{*}<0$, which
implies that $\alpha(s)\to\infty$ as $s\to\infty$. This in turn
implies that $\tau(t)\to\infty$ as $t\to\infty$. As a result, we
obtain that
\[
\lim_{t\to\infty}\frac{\theta(t)}{\left\Vert \theta(t)\right\Vert }=\lim_{t\to\infty}\frac{\bar{\theta}(\tau(t))\alpha(\tau(t))}{\left\Vert \bar{\theta}(\tau(t))\alpha(\tau(t))\right\Vert }=\lim_{t\to\infty}\bar{\theta}(\tau(t))=\rho^{*}
\]
and hence $\theta_{0}\in\Omega$ as needed.

(2) We now show that $\Omega'\cap\mathbb{S}^{P-1}$ has spherical
measure $1$: this is a consequence of the fact that the critical
points of $\bar{H}$ are global minima or strict saddle points. By
taking the gradient of $\bar{H}$, one sees that the critical points
of $\bar{H}$ on the sphere $\mathbb{S}^{P-1}$ are the points $\theta\in\mathbb{S}^{P-1}$
such that
\[
\nabla H(\theta)=\theta\theta^{T}\nabla H(\theta).
\]
Since $\theta\theta^{T}$ is the orthogonal projection on the line
$\mathbb{R}.\theta$, the critical points of $\bar{H}$ on $\mathbb{S}^{P-1}$
are the escape directions. As explained before, global minima of $\overline{H}$
are optimal escape directions. The other escape directions are strict
saddle points: consider such $\rho$ and let $v\perp\rho$ be such
that $v^{T}\mathcal{H}H(\rho)v<-s^{*}v^{T}v$. Differentiating $\bar{H}$
twice and using that $v\perp\rho$, one can show that
\[
v^{T}\mathcal{H}\bar{H}(\rho)v=v^{T}\mathcal{H}H(\rho)v-\nabla H(\rho)^{T}\rho v^{T}v.
\]
Since $\rho$ is an escape direction, $\nabla H(\rho)=-s\rho$ with
$s<s^{*}$ and $\left\Vert \rho\right\Vert =1$: this implies $v^{T}\mathcal{H}\bar{H}(\rho)v<\left(s-s^{*}\right)v^{T}v<0$.
In particular, the points such that the gradient descent on $\bar{H}$
converge to a saddle have spherical measure $0$ on $\mathbb{S}^{P-1}$.
This shows that $\Omega'\cap\mathbb{S}^{P-1}$ has spherical measure
$1$, and allows us to conclude.
\end{proof}

\subsubsection{Deep Linear Networks}

For a depth $L$ DLN and the homogeneous cost $H(\theta)=\mathrm{Tr}\left[G^{T}A_{\theta}\right]$
with SVD decomposition $G=USV^{T}$, the escape directions $\rho$
are of the form
\[
\frac{1}{\sqrt{L}}(\pm u_{i}w_{L-1}^{T},w_{L-1}w_{L-2}^{T},\dots,w_{1}v_{i}^{T})
\]
with speed $s=\mp\frac{s_{i}}{L^{\frac{L-2}{2}}}$, where $u_{i},v_{i}$
are the $i$-th columns of $U,V$ respectively. The optimal speed
is $\frac{s_{1}}{L^{\frac{L-2}{2}}}$, where $s_{1}$ is the largest
singular value of $G$.

Furthermore this loss satisfies the property required to ensure convergence
along the fastest escape path:
\begin{lem}
For a network of depth $L$ and width $w\geq1$, for any escape direction
of the form $\rho=\frac{1}{\sqrt{L}}(\pm u_{i}w_{L-1}^{T},w_{L-1}w_{L-2}^{T},\dots,w_{1}v_{i}^{T})$
with speed $\mp\frac{s_{i}}{L^{\frac{L-2}{2}}}<\frac{s_{1}}{L^{\frac{L-2}{2}}}$
the vector $v=(-u_{1}w_{L-1}^{T},0,\dots,0,w_{1}v_{1}^{T})$ satisfies
\[
v^{T}\mathcal{H}H(\rho)v<\mp\frac{s_{i}}{L^{\frac{L-2}{2}}}v^{T}v.
\]
\end{lem}

\begin{proof}
We have $v^{T}\mathcal{H}H(\rho)v=-\frac{2s_{1}}{L^{\frac{L-2}{2}}}$
and $\pm\frac{s_{i}}{L^{\frac{L-2}{2}}}v^{T}v=\pm\frac{2s_{i}}{L^{\frac{L-2}{2}}}$
as needed.
\end{proof}
This guarantees that gradient flow will not escape along a non-optimal
direction, but it does not rule out the possibility that it converges
to a saddle of the loss $H(\theta)=\mathrm{Tr}\left[G^{T}A_{\theta}\right]$.
Each non-zero saddle $\theta^{*}=(W_{1},\dots,W_{L})$ is technically
proportional to an escape direction $\rho$ with escape speed $0$,
since $\nabla H(\theta^{*})=0$. For shallow networks these saddles
are strict \cite{Kawaguchi_2016_no_local_min} and so they are almost
surely avoided, guaranteeing convergence in direction. For depth $L=3$
we can apply Proposition \ref{prop:convergence_in_direction_condition}
since we have:
\begin{lem}
Consider the cost $H(\theta)=\mathrm{Tr}\left[GA_{\theta}\right]$
for a rank $\min\{n_{0},n_{L}\}$ matrix $G$ and a network of depth
$L=3$ and width $w\geq1$. For any escape direction $\rho$ with
speed $0$ there is a vector $v$ such that $v^{T}\mathcal{H}H(\rho)v<0$.
\end{lem}

\begin{proof}
Since $\rho\neq0$ there must a non-zero $W_{1}$,$W_{2}$ or $W_{3}$.
We separate the case $W_{2}\neq0$ from $W_{1}$ or $W_{3}$ is non-zero.

Case $W_{2}\neq0$: let $u_{1},v_{1}$ be the largest singular vectors
of $G$ and $\tilde{u}_{1},\tilde{v}_{1}$ the largest singular vectors
of $W_{2}$, then $v=(-\tilde{u}_{1}v_{1}^{T},0,u_{1}\tilde{v}_{1}^{T})$
satisfies
\[
v^{T}\mathcal{H}H(\rho)v=-\mathrm{Tr}\left[Gu_{1}\tilde{v}_{1}^{T}W_{2}\tilde{u}_{1}v_{1}^{T}\right]=-s_{1}\tilde{s}_{1}<0.
\]

Case $W_{1}\neq0$ (the case $W_{3}\neq0$ is similar): Let $u_{1},v{}_{1}$
be the largest singular vectors of $W_{1}G$ and $b$ be any unitary
$w$-dim vector, then the parameters $v=(0,bv_{1}^{T},u_{1}b^{T})$
satisfy
\[
v^{T}\mathcal{H}H(\rho)v=-\mathrm{Tr}\left[Gu_{1}b^{T}bv_{1}^{T}W_{1}\right]=-s_{1}<0.
\]
\end{proof}
For $L>3$ we were not able to prove that the saddles can be avoided
with prob. 1, we therefore introduce the assumption:
\begin{assumption}
\label{ass:deep_avoid_saddle} Let $I$ be the set of initializations
which converge to a saddle of the cost $H(\theta)=\mathrm{Tr}\left[GA_{\theta}\right]$.
We shall work on the event $E={\theta_{0}\notin I}$.
\end{assumption}

It can easily be proven for a Gaussian initialization that $\mathbb{P}(E)\geq1/2$,
i.e. that saddles can be avoided with probability at least $1/2$,
since $P\left(H(\theta_{0})<0\right)=\frac{1}{2}$ at initialization
(this follows from the fact that $H((W_{1},\dots,W_{L}))=-H((-W_{1},\dots,W_{L}))$).

Another motivation for this assumption is the fact that if the network
is initialized with balanced weights \cite{arora_2018_DLN_convergence,arora_2019_matrix_factorization},
i.e. if $W_{\ell}^{T}W_{\ell}=W_{\ell-1}W_{\ell-1}^{T}$ for $1<\ell<L$,
then necessarily $\theta_{0}\notin I$. This is because the balancedness
is conserved during training: if gradient flow converges to a saddle,
this saddle must be balanced. However the only balanced saddle of
$H$ is the origin, which can only be approached along an escape direction
$\rho$ with positive speed $s>0$, which are avoided with prob. 1
by Proposition \ref{prop:convergence_in_direction_condition}.

\subsection{Approximately Homogeneous Costs\label{subsec:Approximately-Homogeneous-Costs}}

In the previous section, we studied the escape paths for homogeneous
costs $H$. We extend these results to more general cost functions,
which are only locally homogeneous around a saddle $\vartheta^{*}$,
i.e. we consider costs of the form
\begin{equation}
C(\theta)=H(\theta-\vartheta^{*})+e(\theta-\vartheta^{*}),\label{eq: approximately homogeneous cost}
\end{equation}
where $H$ is a $k$-homogeneous cost $H$, and where $e$ is infinitely
differentiable such that its $m-1$ first derivatives vanish at $0$
for a given $m>k$. We call such costs \emph{$(k,m)$-approximately
homogeneous}. In the setting of a cost $C(A_{\theta})$
for a neural network of depth $L$, the saddle at the origin $\theta=0$
is $(L,2L)$-approximately homogeneous, since the only non-vanishing
derivatives are the $kL$-th derivatives for $k=0,1,2,\dots$.

Since we are only interested in the local behaviour around the saddle
$\vartheta^{*}$, we localize the cost: let $h:\mathbb{R_{+}}\to\mathbb{R_{+}}$
be a smooth cut-off function such that $h(x)=1$ if $0\leq x\leq1$,
$0\leq h(x)\leq1$ if $1<x<2$ and $h(x)=0$ when $x\geq2$. For $r>0$,
we define the localization $C_{r}$ of the cost $C$ as
\begin{equation}
C_{r}(\theta)=H(\theta-\vartheta^{*})+e(\theta-\vartheta^{*})h\left(\frac{\left\Vert \theta-\vartheta^{*}\right\Vert }{r}\right).\label{eq: localized cost}
\end{equation}
As usual, we assume for simplicity that $\vartheta^{*}=0$. We note
for later use that by assumption on $e$, for all compact set $K$
containing $0$, there exists a finite constant $c>0$ such that

\begin{equation}
||\nabla e(\theta)||\leq c||\theta||^{k},\forall\theta\in K.\label{eq: bound gradient e}
\end{equation}

\begin{lem}
\label{lem:bound_cut_off}Let $h:\mathbb{R}_{+}\to\mathbb{R}$ and
$e:\mathbb{R}^{P}\to\mathbb{R}$ be as above. The correction $e_{r}(\theta)=e(\theta)h\left(\frac{1}{r}\left\Vert \theta\right\Vert \right)$
satisfies

\[
\left\Vert \partial_{\theta}^{k}\left[e_{r}\right]\right\Vert _{\infty}=O\left(r^{m-k}\right),\quad\text{as }r\to0.
\]
\end{lem}

\begin{proof}
We have
\[
\partial_{\theta}^{k}\left[e(\theta)h\left(\frac{1}{r}\left\Vert \theta\right\Vert \right)\right]=\sum_{k_{1}+k_{2}=k}\partial_{\theta}^{k_{1}}e(\theta)\partial_{\theta}^{k_{2}}h\left(\frac{1}{r}\left\Vert \theta\right\Vert \right)
\]
Since $\partial_{\theta}^{k_{2}}h\left(\frac{1}{r}\left\Vert \theta\right\Vert \right)=0$
whenever $\left\Vert \theta\right\Vert >2r$ and $\left\Vert \partial_{\theta}^{k_{2}}h\left(\frac{1}{r}\left\Vert \theta\right\Vert \right)\right\Vert _{\infty}=O\left(r^{-k_{2}}\right)$,
while $\left\Vert \partial_{\theta}^{k_{1}}e(\theta)\right\Vert =O(\left\Vert \theta\right\Vert ^{m-k_{1}})$
we see from the above equation that
\begin{align*}
\left\Vert \partial_{\theta}^{k}\left[e(\theta)h\left(\frac{1}{r}\left\Vert \theta\right\Vert \right)\right]\right\Vert _{\infty} & =O(r^{m-k_{1}}r^{-k_{2}})=O(r^{m-k}),
\end{align*}
as claimed.
\end{proof}

\subsection{Escape Cones}

We will approximate approximately homogeneous costs by homogeneous
ones using the following approximation:
\begin{lem}
\label{lem:approx-GD-homogeneous-approx.-homogeneous}Suppose that
$C(\theta)=H(\theta)+e(\theta)$ is \textup{$(k,m)$}-approximately
homogeneous around $0$ as defined in \ref{eq: approximately homogeneous cost}.
Let $\theta_{0}\in\mathbb{R}^{P}$. It holds that for all $t\geq0$
and all $\alpha>0$, there is a finite constant $c_{1}(t)$ that does
not depend on $\alpha$ such that
\[
\left\Vert \gamma_{C}(\alpha^{2-k}t,\alpha\theta_{0})-\gamma_{H}(\alpha^{2-k}t,\alpha\theta_{0})\right\Vert \leq c_{1}(t)\alpha^{2}.
\]
\end{lem}

\begin{proof}
Fix $\theta_{0}\in\mathbb{R}^{P}$ and $t\geq0$ and let $d_{t}=d(t,\theta_{0}):=\sup_{s\leq t}{\gamma_{C}(s,\theta_{0}),\gamma_{H}(s,\theta_{0})}$.
We can bound how fast the distance between the two paths $\gamma_{C}$
and $\gamma_{H}$ increases as follows:
\begin{align*}
\partial_{t}\left\Vert \gamma_{C}(t,\theta_{0})-\gamma_{H}(t,\theta_{0})\right\Vert  & =-\frac{\left(\gamma_{C}(t,\theta_{0})-\gamma_{H}(t,\theta_{0})\right)^{T}}{\left\Vert \gamma_{C}(t,\theta_{0})-\gamma_{H}(t,\theta_{0})\right\Vert }\left(\nabla H\left(\gamma_{C}(t,\theta_{0})\right)+\nabla e\left(\gamma_{C}(t,\theta_{0})\right)-\nabla H\left(\gamma_{H}(t,\theta_{0})\right)\right)\\
 & \leq\left(\sup_{\left\Vert \theta\right\Vert \leq d_{t}}\left\Vert \mathcal{H}H(\theta)\right\Vert _{op}\right)\left\Vert \gamma_{C}(t,\theta_{0})-\gamma_{H}(t,\theta_{0})\right\Vert +\left\Vert \nabla e\left(\gamma_{C}(t,\theta_{0})\right)\right\Vert \\
 & \leq c'd_{t}^{k-2}\left\Vert \gamma_{C}(t,\theta_{0})-\gamma_{H}(t,\theta_{0})\right\Vert +cd_{t}^{k}
\end{align*}
where $c$ comes from \ref{eq: bound gradient e} and $c'=\sup_{\left\Vert x\right\Vert \leq1}\left\Vert \mathcal{H}H(x)\right\Vert _{op}$.
Applying Grönwall's inequality on $[0,t]$ to $A(s)=\left\Vert \gamma_{C}(s,\theta_{0})-\gamma_{H}(s,\theta_{0})\right\Vert +\frac{c}{c'}d_{t}^{2}$
(such that $\partial_{s}A(s)\leq c'd_{t}^{k-2}A(s)$), we obtain
\[
\left\Vert \gamma_{C}(s,\theta_{0})-\gamma_{H}(s,\theta_{0})\right\Vert +\frac{c}{c'}d_{t}^{2}=A(s)\leq A(0)e^{c'd_{t}^{k-2}s}=\frac{c}{c'}d_{t}^{2}e^{c'd_{t}^{k-2}s}.
\]
Hence $\left\Vert \gamma_{C}(t,\theta_{0})-\gamma_{H}(t,\theta_{0})\right\Vert \leq\frac{c}{c'}d_{t}^{2}e^{c'd_{t}^{k-2}t}$
for all times $t\geq0$. To finish the proof, one uses that for a
fixed $t\geq0$, $d(t,\alpha\theta_{0})=O(\alpha)$ as $\alpha\to0$,
which is true because $0$ is a saddle of $C$ and $H$ so their gradient
tends to $0$ with $\alpha$.
\end{proof}
We define the $\epsilon$-\textbf{Escape Cone }as the set $\mathcal{C}_{\epsilon}=\left\{ \theta\in\mathbb{R}^{P}:\frac{H(\theta)}{\left\Vert \theta\right\Vert ^{k}}<\frac{-s^{*}+\epsilon}{k}\right\} $
where we recall that $s^{*}$ denotes the optimal escape speed.
\begin{prop}
\label{prop:escape-cone}For all $\epsilon>0$ small enough there
is a $r>0$ such that
\begin{enumerate}
\item for any $\theta\in\partial\mathcal{C}_{\epsilon}$ with $\left\Vert \theta\right\Vert <r$,
the negative of the gradient of $C$ at $\theta$ points inside the
cone, i.e. denoting by $n$ the normal of $\mathcal{C}_{\epsilon}$
at $\theta$ pointing inside of the cone, we have $-\nabla C(\theta)^{T}n\geq0$.
\item for any point $\theta$ inside the cone with $\left\Vert \theta\right\Vert <r$,
we have $\left\Vert \theta\right\Vert ^{k-1}(-s^{*}-\epsilon)\leq\nabla C(\theta)^{T}\frac{\theta}{\left\Vert \theta\right\Vert }\leq\left\Vert \theta\right\Vert ^{k-1}(-s^{*}+2\epsilon)$.
\item Let $\theta_{0}\in\mathcal{C}_{\epsilon}$ and $\left\Vert \theta_{0}\right\Vert <r$,
let $T$ be the time when $\left\Vert \gamma_{C}(t,\theta_{0})\right\Vert =r$.
When $k=2$ we have for all time $0\leq t<T$
\[
\left\Vert \theta_{0}\right\Vert e^{-(-s^{*}+2\epsilon)t}\leq\left\Vert \gamma(t,\theta_{0})\right\Vert \leq\left\Vert \theta_{0}\right\Vert e^{-(-s^{*}-\epsilon)t}
\]
and when $k\neq2$ we have for all time $0\leq t<T$
\[
\left[\left\Vert \theta_{0}\right\Vert ^{-(k-2)}+(k-2)(-s^{*}+2\epsilon)t\right]^{-\frac{1}{k-2}}\leq\left\Vert \gamma(t,\theta_{0})\right\Vert \leq\left[\left\Vert \theta_{0}\right\Vert ^{-(k-2)}+(k-2)(-s^{*}-\epsilon)t\right]^{-\frac{1}{k-2}}.
\]
\end{enumerate}
\end{prop}

\begin{proof}
For all non-zero $\theta\in\mathbb{R}^{P},$ define $P_{\theta}=\left[I_{d}-\frac{\theta\theta^{T}}{\left\Vert \theta\right\Vert ^{2}}\right]$,
which is the orthogonal projection to the tangent space of $\mathbb{S}^{P-1}$
at $\frac{\theta}{\left\Vert \theta\right\Vert }$. Denote by $\partial\mathcal{C}_{\epsilon}$
the boundary of the cone and note that for any $\theta\in\partial\mathcal{C}_{\epsilon}$,
it holds that $H(\theta)=(-s^{*}+\epsilon)/k$. Choose
\begin{align*}
r=r(\epsilon)=\min\left\{ \frac{\inf_{\rho\in\mathbb{S}^{P-1}\cap\partial\mathcal{C}_{\epsilon}}\left\{ \nabla H\left(\rho\right)^{T}P_{\rho}\nabla H\left(\rho\right)\right\} }{c\sup_{\rho\in\mathbb{S}^{P-1}\cap\partial\mathcal{C}_{\epsilon}}\left\{ \sqrt{\nabla H\left(\rho\right)^{T}P_{\rho}\nabla H\left(\rho\right)}\right\} },\sqrt{\frac{\epsilon}{c}}\right\} ,
\end{align*}

where the constant $c$ comes from \ref{eq: bound gradient e}.

\textbf{(1)} Let $\theta\in\partial\mathcal{C}_{\epsilon}$, so that
$\frac{H(\theta)}{\left\Vert \theta\right\Vert ^{k}}=H\left(\frac{\theta}{\left\Vert \theta\right\Vert }\right)=(-s^{*}+\epsilon)/k$.
The normal pointing inside the cone is equal (up to a positive scaling)
to
\[
-\nabla_{\theta}\Bigg(H\bigg(\frac{\theta}{\left\Vert \theta\right\Vert }\bigg)\Bigg)=-\nabla H\left(\frac{\theta}{\left\Vert \theta\right\Vert }\right)\frac{P_{\theta}}{\left\Vert \theta\right\Vert }.
\]
We then have that
\begin{align*}
\left(-\nabla_{\theta}\bigg(\frac{H(\theta)}{\left\Vert \theta\right\Vert ^{k}}\bigg)\right)^{T}\left(-\nabla C(\theta)\right) & =-\nabla H\left(\frac{\theta}{\left\Vert \theta\right\Vert }\right)^{T}\frac{P_{\theta}}{\left\Vert \theta\right\Vert }\left(-\nabla H(\theta)-\nabla e(\theta)\right)\\
 & =\left\Vert \theta\right\Vert ^{k-1}\nabla H\left(\frac{\theta}{\left\Vert \theta\right\Vert }\right)^{T}P_{\theta}\nabla H\left(\frac{\theta}{\left\Vert \theta\right\Vert }\right)+\nabla H\left(\frac{\theta}{\left\Vert \theta\right\Vert }\right)^{T}\frac{P_{\theta}}{\left\Vert \theta\right\Vert }\nabla e(\theta)\\
 & \geq\left\Vert \theta\right\Vert ^{k-1}\inf_{\rho\in\mathbb{S}^{P-1}\cap\partial\mathcal{C}_{\epsilon}}\left\{ \nabla H\left(\rho\right)^{T}P_{\rho}\nabla H\left(\rho\right)\right\} \\
 & -\frac{1}{\left\Vert \theta\right\Vert }\sup_{\rho\in\mathbb{S}^{P-1}\cap\partial\mathcal{C}_{\epsilon}}\sqrt{\nabla H\left(\rho\right)^{T}P_{\rho}\nabla H\left(\rho\right)}\left\Vert \nabla e(\theta)\right\Vert \\
 & \geq\left\Vert \theta\right\Vert ^{k-1}\inf_{\rho\in\mathbb{S}^{P-1}\cap\partial\mathcal{C}_{\epsilon}}\left\{ \nabla H\left(\rho\right)^{T}P_{\rho}\nabla H\left(\rho\right)\right\} \\
 & -c\left\Vert \theta\right\Vert ^{k}\sup_{\rho\in\mathbb{S}^{P-1}\cap\partial\mathcal{C}_{\epsilon}}\left\{ \sqrt{\nabla H\left(\rho\right)^{T}P_{\rho}\nabla H\left(\rho\right)}\right\} ,
\end{align*}
where we used \ref{eq: bound gradient e} for the last inequality.
The right-hand side above is positive since $\left\Vert \theta\right\Vert <r(\epsilon)\leq\frac{\inf_{H(\rho)=s^{*}+\epsilon}\left\{ \nabla H\left(\rho\right)^{T}P_{\rho}\nabla H\left(\rho\right)\right\} }{c\sup_{H(\rho)=s^{*}+\epsilon}\left\{ \sqrt{\nabla H\left(\rho\right)^{T}P_{\rho}\nabla H\left(\rho\right)}\right\} }$.

\textbf{(2) }Let $\theta\in\mathcal{C}_{\epsilon}$. By \ref{eq: bound gradient e},
we have that
\begin{align*}
\nabla C(\theta)^{T}\theta & =\partial_{\lambda}H(\lambda\theta)_{\Big|\lambda=1}+\left(\nabla e(\theta)\right)^{T}\theta\\
 & \leq kH(\theta)+c\left\Vert \theta\right\Vert ^{k+2}\\
 & =k\left\Vert \theta\right\Vert ^{k}H\left(\frac{\theta}{\left\Vert \theta\right\Vert }\right)+c\left\Vert \theta\right\Vert ^{k+2}\\
 & \leq\left\Vert \theta\right\Vert ^{k}\left(-s^{*}+\epsilon+c\left\Vert \theta\right\Vert ^{2}\right){\color{magenta}}\\
 & \leq\left\Vert \theta\right\Vert ^{k}\left(-s^{*}+2\epsilon\right)
\end{align*}
where we used that $\left\Vert \theta\right\Vert ^{2}<r^{2}\leq\frac{\epsilon}{c}$.\textcolor{magenta}{{}
}In the other direction we obtain
\begin{align*}
\nabla C(\theta)^{T}\theta & =\partial_{\lambda}H(\lambda\theta)_{\Big|\lambda=1}+\left(\nabla e(\theta)\right)^{T}\theta\\
 & \geq kH(\theta)-c\left\Vert \theta\right\Vert ^{k+2}\\
 & =k\left\Vert \theta\right\Vert ^{k}H\left(\frac{\theta}{\left\Vert \theta\right\Vert }\right)-c\left\Vert \theta\right\Vert ^{k+2}\\
 & \geq\left\Vert \theta\right\Vert ^{k}\left(-s^{*}-c\left\Vert \theta\right\Vert ^{2}\right)\\
 & \geq\left\Vert \theta\right\Vert ^{k}\left(-s^{*}-\epsilon\right).
\end{align*}

\textbf{(3)} Applying Grönwall's inequality generalized to polynomial
bounds (Lemma \ref{lem:Gronwall_inequality_polynomial_bounds}), we
have
\[
\partial_{t}\left\Vert \theta\right\Vert ^{2}=-2\nabla C(\theta)^{T}\theta\leq c_{1}\left(\left\Vert \theta\right\Vert ^{2}\right)^{\frac{k}{2}}.
\]
\end{proof}
Putting it all together, this guarantees that with probability 1 over
the initialization, gradient flow escapes the saddle at a specific
speed along a path $\theta^{1}$:
\begin{prop}
\label{prop:lower-bound-rate-escape-path} Let $\theta_{\alpha}(t)=\gamma_{C}(t,\alpha\theta_{0})$
for all $t\geq0$. With prob. 1 over initialization (and under Assumption
\ref{ass:deep_avoid_saddle} when $L>3$) there is a time horizon
$t_{\alpha}^{1}$ that tends to $\infty$ as $\alpha\to0$ and a path
$(\theta^{1}(t))_{t\in\mathbb{R}}$ such that for all $t\in\mathbb{R}$,
$\lim_{\alpha\to0}\theta_{\alpha}(t_{\alpha}^{1}+t)=\theta^{1}(t)$.
Furthermore, for all $\epsilon>0$ s.t. $\epsilon<s^{*}/2$, there
exists $T\in\mathbb{R_{+}}$ such that:

(1) Shallow networks: $e^{(s^{*}-2\epsilon)(T+t)}\leq\left\Vert \theta^{1}(t)\right\Vert \leq e^{(s^{*}+\epsilon)(T+t)}$
for all $t\in\mathbb{R}$.

(2) Deep networks: $\left[(L-2)(s^{*}-2\epsilon)(T-t)\right]^{-\frac{1}{L-2}}\leq\left\Vert \theta^{1}(t)\right\Vert \leq\left[(L-2)(s^{*}+\epsilon)(T-t)\right]^{-\frac{1}{L-2}}$
for all $t<T$ (the path $\theta^{1}$ is defined up to time $T$
in this case).
\end{prop}

\begin{proof}
We consider the gradient flow path $\tilde{\theta}_{\alpha}(t)=\gamma_{H}(t,\alpha\theta_{0})$
on the $k$-homogeneous cost $H$. With prob. 1 (and under Assumption
\ref{ass:deep_avoid_saddle} when $L>3$), we have $\frac{H(\tilde{\theta}_{\alpha}(t))}{\left\Vert \tilde{\theta}_{\alpha}(t)\right\Vert ^{L}}\to-\frac{s^{*}}{k}$
as $t\to\infty$. In particular, for all $\epsilon>0$, there exists
a finite $t\in\mathbb{R}$ such that $\tilde{\theta}_{\alpha=1}\in\mathcal{C}_{\epsilon}$
and more generally, by Lemma \ref{lem:invariance-GD-homogeneous-cost},
we have $\tilde{\theta}_{\alpha}(\alpha^{-(L-2)}t)=\alpha\tilde{\theta}_{1}(t)\in\mathcal{C}_{\epsilon}$.
Lemma \ref{lem:approx-GD-homogeneous-approx.-homogeneous} then shows
that there exists $\alpha_{0}>0$ such that for all $\alpha<\alpha_{0}$,
it holds that $\theta_{\alpha}(\alpha^{-(L-2)}t)\in\mathcal{C}_{\epsilon}$.
Define $t_{0}:=\inf\left\{ t\in\mathbb{R}:\tilde{\theta}_{\alpha=1}(t_{0})\in\mathcal{C}_{\epsilon}\right\} <+\infty$.

By Proposition \ref{prop:escape-cone}, once the gradient flow path
is inside $\mathcal{C}_{\epsilon}$, it cannot leave the escape cone
until the norm $\left\Vert \theta_{\alpha}(t)\right\Vert $ is larger
than some radius $r$. We define the time horizon $t_{\alpha}^{1}=$$\inf\left\{ t\in\mathbb{R}:\left\Vert \theta_{\alpha}(t)\right\Vert =\frac{r}{2}\right\} $
and the escape path $\theta^{1}$ as the limit $\theta^{1}(t)=\lim_{\alpha\to0}\theta_{\alpha}(t_{\alpha}^{1}+t)$
for $t\in\mathbb{R}$ (the limit is well defined by continuity of
$\theta\mapsto\gamma_{C}(t,\theta)$ and is an escape path by continuity
of $\theta\mapsto\nabla\gamma_{C}(t,\theta)$). One can see that for
any $t<0$, there exists $\alpha>0$ small enough such that $t_{\alpha}+t>\alpha^{-(L-2)}t_{0}$,
thus it holds that $\theta^{1}(t)\in\mathcal{C_{\epsilon}}$ since
for a small enough $\alpha$, we have $\theta(t_{\alpha}^{1}+t)\in\mathcal{C_{\epsilon}}$.
Proposition \ref{prop:escape-cone} then implies the escape rates
for deep and shallow networks.
\end{proof}

\subsection{Optimal Escape Paths}

In this section, we define the notions of escape paths, optimal escape
paths and we give a description of the optimal escape paths at the
origin.

Proposition \ref{prop:lower-bound-rate-escape-path} shows that as
$\alpha\searrow0$ one has convergence to an escape path which escapes
with an almost optimal speed $s^{*}-2\epsilon$ for a small $\epsilon>0$.
We will show that the only such escape paths are the optimal escape
paths, i.e. those that escape exactly at a speed of $s^{*}$, furthermore
these escape paths are unique up to rotations of the network.

We understand well the escape paths of the homogeneous loss $H$,
and want to use this knowledge to describe the escape paths of the
locally homogogeneous loss $C$. We will show a bijection between
the escape paths of $H$ and those of $C$ such that their speed is
preserved, but only between the set of escape paths which escape faster
than a certain speed. It seems that in general there is no speed-preserving
bijection between escape paths, indeed while for shallow networks
(when the saddle is strict) one may apply the Hartman-Grobman Theorem
to obtain a bijection, it does not preserve speed (since the bijection
is in general not differentiable, only Hölder continuous).

This bijection is described by the following theorem (which is a more
general version of Theorem 5 {[}{]} from the main - one simply needs
to set $k=L$ and $m=2L$ and $s^{*}=L\left\Vert H\right\Vert _{\infty}$
to recover theorem 5, i.e. the DLN case):
\begin{thm}[Theorem \ref{thm:bijection-fast-escape-paths} of the main text]\label{thm:bijection-optimal-escape-paths-appendix}
Let $C=H+e$ be a $(k,m)$-approximately homogeneous loss, where $H$
is a polynomial.

\textbf{When $k=2$:} for all $s_{0}$ s.t. $s_{0}>\frac{k\left\Vert H\right\Vert _{\infty}}{m-1}$
there is a unique bijection $\Psi:\mathcal{F}_{H}(s_{0})\to\mathcal{F}_{C}(s_{0})$
such that for all paths $x\in\mathcal{F}_{C}(s_{0})$, we have $\left\Vert x(t)-\Psi(x)(t)\right\Vert =O(e^{(m-1)s_{0}t})$
as $t\to-\infty$.

\textbf{When $k>2$:} for all $s_{0}>\frac{k-1}{m-k+1}k\left\Vert H\right\Vert _{\infty}$
there is a unique bijection $\Psi:\mathcal{F}_{H}(s_{0})\to\mathcal{F}_{C}(s_{0})$
such that for all paths $x\in\mathcal{F}_{C}(s_{0})$, we have $\left\Vert x(t)-\Psi(x)(t)\right\Vert =O((-t)^{-\frac{m-k+1}{k-2}})$
as $t\to-\infty$.
\end{thm}

Note that in the case $s_{0}>k\left\Vert H\right\Vert _{\infty}$,
the set $\mathcal{F}_{H}(s_{0})$ is empty (and therefore so is $\mathcal{F}_{C}(s_{0})$).
\begin{proof}
For $r>0$, recall that $C_{r}$ denotes the localization of the cost
$C$ as introduced in Section \ref{subsec:Approximately-Homogeneous-Costs}.
It is readily seen that for all $r>0$, there is a bijection $\Psi_{r}$
between $\mathcal{F}_{C}\left[s_{0}\right]$ and $\mathcal{F}_{C_{r}}\left[s_{0}\right]$
such that for all $x_{0}\in\mathcal{F}_{C}\left[s_{0}\right]$, $\left\Vert x_{0}(t)-\Psi_{r}(x_{0})(t)\right\Vert =O(0)$
(i.e. the difference is zero for small enough $t<0$). We therefore
only need to show a bijection between $\mathcal{F}_{C_{r}}\left[s_{0}\right]$
and $\mathcal{F}_{H}\left[s_{0}\right]$.

Consider a fast escaping path $x_{0}\in\mathcal{F}_{H}(s_{0})$ of
the homogeneous approximation of the loss. The escape paths of the
origin w.r.t. to gradient flow on the cost $C_{r}$ are fixed points
of the following map:
\[
\Phi_{C_{r}}:x_{0}\mapsto\left(t\mapsto\int_{-\infty}^{t}-\nabla C_{r}(x_{0}(u))du\right).
\]
Our strategy is simply to iterate this map starting from the path
$x_{0}$ to find such a fixed point (note that any fixed point of
$\Phi_{C_{r}}$ is differentiable by the fundamental theorem of calculus).
We will show that this iteration converges to a gradient flow path
$x'_{0}$ of the cost $C_{r}$ which is, as $t\to-\infty$, $O(e^{(m-1)s_{0}t})$-close
to $x_{0}$ when $k=2$ and $O((-t)^{\frac{k-m-1}{k-2}})$-close to
$x_{0}$ when $k>2$.

For $c>0$, let $B_{c}$ be the \textit{set of corrections}, that
is the set of all paths $b:\mathbb{R}_{-}\to\mathbb{R}^{P}$ (which
are Lebesgue measurable functions) such that when $k=2$, $\left\Vert b(t)\right\Vert \leq ce^{(m-1)s_{0}t}$
for all $t\leq0$ and when $k>2$, $\left\Vert b(t)\right\Vert \leq c(-t)^{\frac{k-m-1}{k-2}}$
for all $t\leq0$.

The convergence of the iteration process follows from the fact that
$\Phi$ is a contraction w.r.t. to some norm on the set of paths $x_{0}+B_{c}$
(the set of possible corrections around $x_{0}$). Indeed, Lemma \ref{lem:shallow_contraction_phi}
(case $k=2$, stated and proven in Section \ref{subsec:contraction-Shallow-networks})
and Lemma \ref{lem:deep_contraction_phi} (case $k>2$, stated and
proven in Section \ref{subsec:contraction-deep-networks}) show that
for all $x_{0}\in\mathcal{F}_{H}[s_{0}]$, there exist $r>0$ small
enough and $c>0$ large enough, such that $\Phi_{C_{r}}$ is a contraction
on $x_{0}+B_{c}$ for some well-suited norm (defined below), hence
guaranteeing the existence and uniqueness of a fixpoint $x'_{0}$
of $\Phi_{C_{r}}$ (which is obtained by interating $\Phi_{C_{r}}$
infinitely many times). We thus define the map $\Psi:x_{0}\mapsto\Psi(x_{0})=x_{0}'$.

We need to show that $\Psi$ has an inverse that maps a fast escaping
path $y_{0}\in\mathcal{F}_{C}(s_{0})$ back to a path $\Psi^{-1}(y_{0})\in\mathcal{F}_{H}(s_{0})$.
We iterate the map
\[
\Phi_{H}:y_{0}\mapsto\left(t\mapsto\int_{-\infty}^{t}-\nabla H(y_{0}(u))du\right)
\]
whose fixed points are the escape paths w.r.t. to gradient flow on
the cost $H$. By a similar argument we can show that this map is
a contraction on $x_{0}+B_{c}$. Choosing $y_{0}=x'_{0}$, this again
implies the existence of a unique path $\Psi^{-1}(x_{0}')$ which
is $O(e^{(m-1)s_{0}t})$-close to $x_{0}'$ when $k=2$ and $O((-t)^{-\frac{m-k+1}{k-2}})$-close
to $x_{0}'$ when $k>2$. Because $x_{0}'\in x_{0}+B_{c}$ The uniqueness
implies that since $x_{0}'=\Psi(x_{0})$, the path $\Psi^{-1}(x_{0}')$
must be $x_{0}$. This shows that $\Psi$ is a bijection and it is
the only bijection between fast escaping paths with the property of
mapping a path to a closeby path as in the statement of the theorem.
\end{proof}

\subsubsection{Shallow networks \label{subsec:contraction-Shallow-networks}}

For the case $k=2$, we consider the following norm for $\beta<(m-1)s_{0}$
\[
\left\Vert b\right\Vert _{\beta}^{2}=\int_{-\infty}^{0}e^{-2\beta t}\left\Vert b(t)\right\Vert ^{2}dt
\]
defined on the corrections $b\in B_{c}$, where the set of corrections
$B_{c}$ is the set of all paths $b:\mathbb{R}_{-}\to\mathbb{R}^{P}$
such that $\left\Vert b(t)\right\Vert \leq ce^{(m-1)s_{0}t}$ for
all $t\leq0$. The condition $\beta<(m-1)s_{0}$ ensures that
\[
\left\Vert b\right\Vert _{\beta}^{2}=\int_{-\infty}^{0}e^{-2\beta t}\left\Vert b(t)\right\Vert ^{2}dt\leq c^{2}\int_{-\infty}^{0}e^{2((m-1)s_{0}-\beta)t}dt<\infty.
\]
 As a result the set $x_{0}+B_{c}$ equipped with the distance induced
by the norm $\left\Vert \cdot\right\Vert _{\beta}$ is a complete
metric space.

We define the scalar product for two corrections $x,y\in B$
\[
\left\langle x,y\right\rangle _{\beta}=\int_{-\infty}^{\infty}e^{-2\beta t}\left\langle x(t),y(t)\right\rangle dt.
\]
We first state a few useful properties of $\left\langle \cdot,\cdot\right\rangle _{\beta}$.
In the following, $\dot{x}$ is the path obtained by considering the
derivative of $x$.
\begin{lem}
\label{lem:properties_w_norm} For any two corrections $x,y\in B$,
we have
\begin{enumerate}
\item $\left\langle x,\dot{y}\right\rangle _{\beta}=2\beta\left\langle x,y\right\rangle _{\beta}-\left\langle \dot{x},y\right\rangle _{\beta}$.
\item $\left\langle x,\dot{x}\right\rangle _{\beta}=\beta\left\Vert x\right\Vert _{\beta}^{2}$.
\item $\left\Vert x\right\Vert _{\beta}\leq\frac{1}{\beta}\left\Vert \dot{x}\right\Vert _{\beta}$.
\end{enumerate}
\end{lem}

\begin{proof}
The first point is obtained by integration by part:
\begin{align*}
\left\langle x,\dot{y}\right\rangle _{\beta}=\int_{-\infty}^{\infty}e^{-2\beta t}\left(x(t)\right)^{T}\dot{y}(t)dt & =\int_{-\infty}^{\infty}2\beta e^{-2\beta t}\left(x(t)\right)^{T}y(t)dt-\int_{-\infty}^{\infty}e^{-2\beta t}\dot{x}(t)y(t)dt\\
 & =2\beta\left\langle x,y\right\rangle _{\beta}-\left\langle \dot{x},y\right\rangle _{\beta}.
\end{align*}
The second point is a consequence of the first one, by taking $x=y$.
Finally, the last point follows from the second one since $\left\Vert x\right\Vert _{\beta}^{2}=\frac{1}{\beta}\left\langle x,\dot{x}\right\rangle _{\beta}\leq\frac{1}{\beta}\left\Vert x\right\Vert _{\beta}\left\Vert \dot{x}\right\Vert _{\beta}$,
by Cauchy-Schwarz Inequality.
\end{proof}
We may now describe how for large enough $\beta$, one can guarantee
that the map $\Phi_{C_{r}}$ is a contraction on the set $x_{0}+B_{c}$:
\begin{lem}
\label{lem:shallow_contraction_phi} Let $C_{r}=H+e_{r}$ be a localized
$(2,m)$-approximately homogeneous loss as in\textcolor{magenta}{{}
}\textcolor{black}{\ref{eq: localized cost},} where $H$ is a polynomial.
Choose a $s_{0}>\frac{2\left\Vert H\right\Vert _{\infty}}{m-1}$.
There is a $r$ small enough such that for any $x_{0}\in\mathcal{F}_{H}[s_{0}]$
there is a constant $c$ such that the map $\Phi_{C_{r}}$ is contraction
on the set $x_{0}+B_{c}=\left\{ x_{0}+b:\left\Vert b(t)\right\Vert \leq ce^{s_{0}(m-1)t},\forall t<0\right\} $
w.r.t. the norm on paths $\left\Vert \cdot\right\Vert _{\beta}$ for
some $\beta$.
\end{lem}

\begin{proof}
We first show that for $r>0$ small enough and $c>0$ large enough,
the image of $x_{0}+B_{c}$ under $\Phi_{C_{r}}$ is contained in
itself and then show that $\Phi_{C_{r}}$is a contraction w.r.t. the
norm $\left\Vert \cdot\right\Vert _{\beta}$ for an adequate $\beta$.

(1) \textbf{Self-map: }Let $x\in x_{0}+B_{c}$, i.e. $x=x_{0}+b$
for some $b\in B_{c}$, then using the linearity of $\nabla H$ and
the fact $x_{0}$ is a gradient flow path of $H$, we obtain
\begin{align*}
\Phi_{C_{r}}(x)(t) & =-\int_{-\infty}^{t}\nabla H(x(u))+\nabla e_{r}(x(u))du\\
 & =-\int_{-\infty}^{t}\nabla H(x_{0}(u))+\nabla H(b(u))+\nabla e_{r}(x(u))du\\
 & =x_{0}(t)-\int_{-\infty}^{t}\nabla H(b(u))+\nabla e_{r}(x(u))du.
\end{align*}
Writing $b'(t)=\int_{-\infty}^{t}\left[\nabla H(b(u))+\nabla e_{r}(x(u))\right]du$
we need to show that $b'\in B_{c}$. We can bound $\left\Vert b'(t)\right\Vert $
by
\begin{align*}
\left\Vert b'(t)\right\Vert  & \leq\int_{-\infty}^{t}\left(\left\Vert \nabla H(b(u))\right\Vert +\left\Vert \nabla e_{r}(x_{0}(u))\right\Vert +\left\Vert \nabla e_{r}(x(u))-\nabla e_{r}(x_{0}(u))\right\Vert \right)du.
\end{align*}
Using the fact that for a map $g$ with uniformly bounded Hessian
$\left\Vert \mathcal{H}g\right\Vert _{\infty}<\infty$, we have $\left\Vert \nabla g(x)-\nabla g(y)\right\Vert \leq\left\Vert \mathcal{H}g\right\Vert _{\infty}\left\Vert x-y\right\Vert $,
it follows that $\left\Vert \nabla H(b(u))\right\Vert \leq\sup_{z\in\mathbb{R}^{P}}\left\Vert \mathcal{H}H(z)\right\Vert _{\mathrm{op}}\left\Vert b(u)\right\Vert $
and $\left\Vert \nabla e_{r}(x(u))-\nabla e_{r}(x_{0}(u))\right\Vert \leq\sup_{z\in\mathbb{R}^{P}}\left\Vert \mathcal{H}e_{r}(z)\right\Vert _{\mathrm{op}}\left\Vert b(u)\right\Vert $.
The last term $\left\Vert \nabla e_{r}(x_{0}(u))\right\Vert $ can
be bounded by $\frac{\sup_{z}\left\Vert \partial_{z}^{m}e_{r}(z)\right\Vert _{\mathrm{op}}\left\Vert x_{0}(u)\right\Vert ^{m-1}}{(m-1)!}$
since the first $m-1$ derivatives of $e_{r}$ vanish at $0$ (see
point 1 of Lemma \ref{lem:dist-gradients-from-kth-derivative}).

We therefore get
\begin{align*}
\left\Vert b'(t)\right\Vert  & \leq\int_{-\infty}^{t}\left(\left(\sup_{z\in\mathbb{R}^{P}}\left\Vert \mathcal{H}H(z)\right\Vert _{\mathrm{op}}+\sup_{z\in\mathbb{R}^{P}}\left\Vert \mathcal{H}e_{r}(z)\right\Vert _{\mathrm{op}}\right)\left\Vert b(u)\right\Vert +\sup_{z\in\mathbb{R}^{P}}\left\Vert \partial_{z}^{m}e_{r}(z)\right\Vert _{\mathrm{op}}\left\Vert x_{0}(u)\right\Vert ^{m-1}\right)du.
\end{align*}
Since $\sup_{z\in\mathbb{R}^{P}}\left\Vert \mathcal{H}H(z)\right\Vert _{\mathrm{op}}=2\left\Vert H\right\Vert _{\mathrm{op}}$
(where $\left\Vert \mathcal{H}H(x)\right\Vert _{\mathrm{op}}$ is
the operator norm of the Hessian $\mathcal{H}H$, while $\left\Vert H\right\Vert _{\mathrm{op}}=\max_{x\in\mathbb{S}^{P-1}}\left|H(x)\right|$)
and by Lemma \ref{lem:bound_cut_off} $\sup_{z\in\mathbb{R}^{P}}\left\Vert \mathcal{H}e_{r}(z)\right\Vert _{\infty}\leq\kappa_{0}r^{m-2}$
and $\sup_{z\in\mathbb{R}^{P}}\left\Vert \partial_{x}^{m}e_{r}(z)\right\Vert _{\infty}\leq\kappa_{1}$
\begin{align*}
 & \leq\frac{2\left\Vert H\right\Vert _{\mathrm{op}}+\kappa_{0}r^{m-2}}{s_{0}(m-1)}ce^{s_{0}(m-1)t}+\kappa_{1}c_{0}e^{s_{0}(m-1)t}\\
 & \leq\left(\frac{2\left\Vert H\right\Vert _{\mathrm{op}}+\kappa_{0}r^{m-2}}{s_{0}(m-1)}c+\kappa_{1}c_{0}\right)e^{s_{0}(m-1)t}
\end{align*}
Since by assumption $s_{0}>\frac{2\left\Vert H\right\Vert _{\mathrm{op}}}{m-1}$,
we can choose $r$ small enough such that $\frac{2\left\Vert H\right\Vert _{\mathrm{op}}+\kappa_{0}r^{m-2}}{s_{0}(m-1)}<1$.
We can then choose $c$ large enough so that $\frac{2\left\Vert H\right\Vert _{\mathrm{op}}+\kappa_{0}r^{m-2}}{s_{0}(m-1)}c+\kappa_{1}c_{0}\leq c$.
With these choices of $r$ and $c$, we obtain that $\left\Vert b'(t)\right\Vert \leq ce^{s_{0}(m-1)t}$
and therefore $b'\in B_{c}$ as needed.

(2) \textbf{Contraction: }We need to bound for any $x,y\in x_{0}+B_{c}$
\[
\left\Vert \Phi_{C_{r}}(x)-\Phi_{C_{r}}(y)\right\Vert _{\beta}^{2}=\left\Vert t\mapsto\int_{-\infty}^{t}\left[\nabla C_{r}(x(u))-\nabla C_{r}(y(u))\right]du\right\Vert _{\beta}^{2},
\]
for any $\beta<(m-1)s_{0}$.

From point (1), we know that $\Phi_{C_{r}}(x),\Phi_{C_{r}}(y)\in x_{0}+B_{c}$
and hence $\left\Vert \Phi_{C_{r}}(x)-\Phi_{C_{r}}(y)\right\Vert _{\beta}^{2}\leq\infty$
(since $\beta<(m-1)s_{0}$).

From point $(3)$ of Lemma \ref{lem:properties_w_norm} we have:
\begin{align*}
\left\Vert t\mapsto\int_{-\infty}^{t}\left[\nabla C_{r}(x(u))-\nabla C_{r}(y(u))\right]du\right\Vert _{\beta} & \leq\frac{1}{\beta}\left\Vert t\mapsto\nabla C_{r}(x(t))-\nabla C_{r}(y(t))\right\Vert _{\beta}\\
 & \leq\frac{\sup_{z}\left\Vert \mathcal{H}C_{r}(z)\right\Vert _{\mathrm{op}}}{\beta}\left\Vert x-y\right\Vert _{\beta}.
\end{align*}

By the localization, we have $\sup_{z}\left\Vert \mathcal{H}C_{r}(z)\right\Vert _{\mathrm{op}}\leq\sup_{z}\left\Vert \mathcal{H}H(z)\right\Vert _{\mathrm{op}}+\sup_{z}\left\Vert \mathcal{H}e_{r}(z)\right\Vert _{\mathrm{op}}\leq2\left\Vert H\right\Vert _{\mathrm{op}}+\kappa_{0}r^{m-2}$.

Therefore to guarantee a contraction, we choose $\beta>2\left\Vert H\right\Vert _{\mathrm{op}}+\kappa_{0}r^{m-2}$,
so that $\frac{\sup_{z}\left\Vert \mathcal{H}C_{r}(z)\right\Vert _{\mathrm{op}}}{\beta}<1$.
Therefore $\beta$ lies in an open interval
\[
\left(\frac{2\left\Vert H\right\Vert _{\mathrm{op}}+\kappa_{0}r^{m-2}}{(m-1)s_{0}}\right)(m-1)s_{0}<\beta<(m-1)s_{0}
\]
which is non-empty since we have chosen $r$ small enough in point
(1) such that $\frac{2\left\Vert H\right\Vert _{\mathrm{op}}+\kappa_{0}r^{m-2}}{s_{0}(m-1)}<1$.
\end{proof}

\subsubsection{Deep case \label{subsec:contraction-deep-networks}}

For the case $k>2$, we consider the following norm
\[
\left\Vert b\right\Vert _{\alpha}^{2}=\int_{-\infty}^{0}(-t)^{2\alpha-1}\left\Vert b(t)\right\Vert ^{2}dt.
\]
If $\alpha<\frac{m-k+1}{k-2}$ then this norm is finite on any corrections
$b\in B_{c}$ (i.e. if $\left\Vert b(t)\right\Vert \leq c(-t)^{-\frac{m-k+1}{k-2}}$),
since
\[
\left\Vert b\right\Vert _{\alpha}^{2}\leq c^{2}\int_{-\infty}^{0}(-t)^{2(\alpha-\frac{m-k+1}{k-2})-1}dt<\infty.
\]
The set $x_{0}+B_{c}$ equipped with the distance $\left\Vert \cdot\right\Vert _{\alpha}$
therefore defines a complete metric space.

Again, for paths $x,y$ such that $\left\Vert x\right\Vert _{w},\left\Vert y\right\Vert _{w}<\infty$,
we define the scalar product
\[
\left\langle x,y\right\rangle _{w}=\int_{-\infty}^{0}(-t)^{2\alpha-1}\left\langle x(t),y(t)\right\rangle dt.
\]
Lemma \ref{lem:properties_w_norm} is now replaced by the following:
\begin{lem}
\label{lem:properties_beta_norm}For any differentiable paths $x,y$
with $\left\Vert x\right\Vert _{w},\left\Vert y\right\Vert _{w}<\infty$,
we have
\begin{enumerate}
\item $\left\langle x,-t\dot{y}\right\rangle _{w}=2\alpha\left\langle x,y\right\rangle _{w}-\left\langle -t\dot{x},y\right\rangle _{w}$.
\item $\frac{1}{\alpha}\left\langle x,-t\dot{x}\right\rangle _{w}=\left\Vert x\right\Vert _{w}^{2}$.
\item $\left\Vert x\right\Vert _{w}\leq\frac{1}{\alpha}\left\Vert -t\dot{x}\right\Vert _{w}$.
\end{enumerate}
\end{lem}

\begin{proof}
The first point is obtained by integration by part:
\begin{align*}
\left\langle x,-t\dot{y}\right\rangle _{w} & =\int_{-\infty}^{0}(-t)^{2\alpha}x(t)\dot{y}(t)dt\\
 & =\int_{-\infty}^{0}2\alpha(-t)^{2\alpha-1}x(t)y(t)dt-\int_{-\infty}^{0}(-t)^{2\alpha}\dot{x}(t)y(t)dt\\
 & =2\alpha\left\langle x,y\right\rangle _{w}-\left\langle -t\dot{x},y\right\rangle _{w}.
\end{align*}
Taking $x=y$, we obtain the second point. Finally, the last point
follows from the second one since:
\[
\left\Vert x\right\Vert _{w}^{2}=\frac{1}{\alpha}\left\langle x,-t\dot{x}\right\rangle _{w}\leq\frac{1}{\alpha}\left\Vert x\right\Vert _{w}\left\Vert -t\dot{x}\right\Vert _{w}.
\]

Under certain conditions, we can ensure that there is an $\alpha$
such that $\Phi$ is a contraction on $x_{0}+B_{c}$ w.r.t. the norm
$\left\Vert \cdot\right\Vert _{\alpha}$:
\end{proof}
\begin{lem}
\label{lem:deep_contraction_phi} Let $C_{r}=H+e_{r}$ be a localized
$(k,m)$-approximately homogeneous loss as in\textcolor{magenta}{{}
}\textcolor{black}{\ref{eq: localized cost},} where $H$ is a polynomial,
with $k>2$. Choose a $s_{0}>\frac{k-1}{m-k+1}k\left\Vert H\right\Vert _{\infty}$.
Let $x_{0}\in\mathcal{F}_{H}[s_{0}]$, there exist $r>0$ small enough,
$c>0$ large enough and $T<0$ small enough, such that the map $\Phi$
is a contraction on the set $x_{0}+B_{c,T}=\left\{ x_{0}+b:\left\Vert b(t)\right\Vert \leq ce^{s(m-1)t},\forall t<T\right\} $
w.r.t. to the norm $\left\Vert \cdot\right\Vert _{\alpha}$ for some
well-suited $\alpha$.
\end{lem}

\begin{proof}
(1) \textbf{Self-map: }Let $x_{0}+b\in x_{0}+B_{c,T}$, we first show
that $\Phi(x_{0}+b)\in x_{0}+B_{c,T}$. Let us rewrite
\begin{align*}
\Phi(x_{0}+b) & =-\int_{-\infty}\nabla C_{r}(x_{0}(u)+b(u))du\\
 & =-\int_{-\infty}\nabla H(x_{0}(u))du+\int_{-\infty}\nabla H(x_{0}(u))-\nabla C_{r}(x_{0}(u)+b(u))du\\
 & =x_{0}+b'
\end{align*}
where
\begin{align*}
b'(t) & =\int_{-\infty}^{t}\nabla H(x_{0}(u))-\nabla C_{r}(x_{0}(u)+b(u))du\\
 & =\int_{-\infty}^{t}\nabla H(x_{0}(u))-\nabla C_{r}(x_{0}(u))du\\
 & +\int_{-\infty}^{t}\nabla C_{r}(x_{0}(u))-\nabla C_{r}(x_{0}(u)+b(u))du\\
 & =-\int_{-\infty}^{t}\nabla e_{r}(x_{0}(u))du\\
 & +\int_{-\infty}^{t}\nabla C_{r}(x_{0}(u))-\nabla C_{r}(x_{0}(u)+b(u))du.
\end{align*}
Our goal is to show that $b'(t)\in B_{c}$, i.e. that $\left\Vert b'(t)\right\Vert \leq c(-t)^{\frac{k-m-1}{k-2}}$.
We bound the two terms separately:
\[
\left\Vert b'(t)\right\Vert \leq\left\Vert \int_{-\infty}^{t}\nabla e_{r}(x_{0}(u))du\right\Vert +\left\Vert \int_{-\infty}^{t}\left[\nabla C_{r}(x_{0}(u))-\nabla C_{r}(x_{0}(u)+b(u))\right]du\right\Vert .
\]

The first term $\left\Vert \int_{-\infty}^{t}\nabla e_{r}(x_{0}(u))du\right\Vert $
is bounded by
\begin{align*}
 & \int_{-\infty}^{t}\left\Vert \nabla e_{r}(x_{0}(u))\right\Vert du\\
 & \leq m\left\Vert \partial_{x}^{m}e_{r}\right\Vert \int_{-\infty}^{t}\left\Vert x_{0}(u)\right\Vert ^{m-1}du\\
 & \leq m\kappa s_{0}^{-\frac{m-1}{k-2}}(k-2)^{-\frac{m-1}{k-2}}\int_{-\infty}^{t}(-u)^{-\frac{m-1}{k-2}}du\\
 & =m\kappa s_{0}^{-\frac{m-1}{k-2}}(k-2)^{-\frac{m-1}{k-2}}\frac{k-2}{m-k+1}(-t)^{\frac{k-m-1}{k-2}}\\
 & =m\kappa s_{0}^{-\frac{m-1}{k-2}}\frac{(k-2)^{\frac{k-m-1}{k-2}}}{(m-k+1)}(-t)^{\frac{k-m-1}{k-2}}.
\end{align*}
The second term $\left\Vert \int_{-\infty}^{t}\nabla C_{r}(x_{0}(u))-\nabla C_{r}(x_{0}(u)+b(u))du\right\Vert $
is bounded by
\begin{align*}
 & \int_{-\infty}^{t}\left\Vert \nabla C_{r}(x_{0}(u))-\nabla C_{r}(x_{0}(u)+b(u))\right\Vert du\\
 & \leq\frac{\sup_{z}\left\Vert \partial_{z}^{k}C_{r}(z)\right\Vert _{\mathrm{op}}}{(k-2)!}\int_{-\infty}^{t}\left\Vert b(u)\right\Vert \max\left\{ \left\Vert x_{0}(u)\right\Vert ,\left\Vert x_{0}(u)+b(u)\right\Vert \right\} ^{k-2}du
\end{align*}
by Lemma \ref{lem:dist-gradients-from-kth-derivative}. Let us first
bound $\max\left\{ \left\Vert x_{0}(u)\right\Vert ,\left\Vert x_{0}(u)+b(u)\right\Vert \right\} ^{k-2}$
by
\begin{align*}
\left(\left\Vert x_{0}(u)\right\Vert +\left\Vert b(u)\right\Vert \right)^{k-2} & \leq\left(\left(s_{0}(k-2)(-u)\right)^{-\frac{1}{k-2}}+c(-u)^{-\frac{m-k+1}{k-2}}\right)^{k-2}\\
 & =\left(s_{0}(k-2)(-u)\right)^{-1}+\sum_{i=1}^{k-2}\left(\begin{array}{c}
k-2\\
i
\end{array}\right)\left(s_{0}(k-2)(-u)\right)^{-1+\frac{i}{k-2}}c(-u)^{-\frac{m-k+1}{k-2}i}\\
 & =\left(s_{0}(k-2)(-u)\right)^{-1}+\left(s_{0}(k-2)(-u)\right)^{-1}\sum_{i=1}^{k-2}\left(\begin{array}{c}
k-2\\
i
\end{array}\right)\left(s_{0}(k-2)\right)^{\frac{i}{k-2}}c^{i}(-u)^{-\frac{m-k}{k-2}i}\\
 & \leq\left(s_{0}(k-2)(-u)\right)^{-1}\left[1+\sum_{i=1}^{k-2}\left(\begin{array}{c}
k-2\\
i
\end{array}\right)\left(s_{0}(k-2)\right)^{\frac{i}{k-2}}c^{i}(-T)^{-\frac{m-k}{k-2}i}\right],
\end{align*}
for any $\epsilon$, we can choose $T<0$ small enough so that $\max\left\{ \left\Vert x_{0}(u)\right\Vert ,\left\Vert x_{0}(u)+b(u)\right\Vert \right\} ^{k-2}$
is bounded by $\left(s_{0}(k-2)(-u)\right)^{-1}[1+\epsilon]$.

Using also the bounds $\frac{\sup_{z}\left\Vert \partial_{z}^{k}C_{r}(z)\right\Vert _{\mathrm{op}}}{(k-2)!}\leq k(k-1)\left\Vert H\right\Vert _{\infty}+\frac{\kappa}{(k-2)!}r^{m-k}$
and $\left\Vert b(u)\right\Vert \leq c(-u)^{-\frac{m-k+1}{k-2}}$,
the second term $\left\Vert \int_{-\infty}^{t}\nabla C_{r}(x_{0}(u))-\nabla C_{r}(x_{0}(u)+b(u))du\right\Vert $
can be bounded by
\begin{align*}
 & \frac{k(k-1)\left\Vert H\right\Vert _{\mathrm{op}}+\frac{\kappa}{(k-2)!}r^{m-k}}{s_{0}(k-2)}(1+\epsilon)\int_{-\infty}^{t}c(-u)^{-\frac{m-k+1}{k-2}-1}du\\
 & =\frac{k(k-1)\left\Vert H\right\Vert _{\mathrm{op}}+\frac{\kappa}{(k-2)!}r^{m-k}}{s_{0}(m-k+1)}(1+\epsilon)c(-t)^{-\frac{m-k+1}{k-2}}.
\end{align*}

We choose $r>0$ small enough, $c>0$ large enough, and $T<0$ small
enough so that $\frac{k(k-1)\left\Vert H\right\Vert _{\mathrm{op}}+\frac{\kappa}{(k-2)!}r^{m-k}}{s_{0}(m-k+1)}(1+\epsilon)<1$
and $\left[m\kappa s_{0}^{-\frac{m-1}{k-2}}\frac{(k-2)^{-\frac{m-k+1}{k-2}}}{m-k+1}+\frac{k(k-1)\left\Vert H\right\Vert _{\mathrm{op}}+\frac{\kappa}{(k-2)!}r^{m-k}}{s_{0}(m-k+1)}(1+\epsilon)c\right]\leq c$
so that

\begin{align*}
\left\Vert b'(t)\right\Vert  & \leq\left[m\kappa s_{0}^{-\frac{m-1}{k-2}}\frac{(k-2)^{-\frac{m-k+1}{k-2}}}{m-k+1}+\frac{k(k-1)\left\Vert H\right\Vert _{\mathrm{op}}+\frac{\kappa}{(k-2)!}r^{m-k}}{s_{0}(m-k+1)}(1+\epsilon)c\right](-t)^{-\frac{m-k+1}{k-2}}\\
 & \leq c(-t)^{-\frac{m-k+1}{k-2}}
\end{align*}
and therefore $b'\in B_{c}$.

(2) \textbf{Contraction: }We have, for any $x,y\in x_{0}+B_{c}$

\begin{align*}
\left\Vert \Phi(x)-\Phi(y)\right\Vert _{w} & =\left\Vert \int_{-\infty}^{t}\nabla C_{r}(x(s))-\nabla C_{r}(y(s))ds\right\Vert _{\alpha}\\
 & \leq\frac{1}{\alpha}\left\Vert -t\left(\nabla C_{r}(x)-\nabla C_{r}(y)\right)\right\Vert _{\alpha}\\
 & \leq\frac{\left\Vert \partial_{k}C_{r}(x_{0})\right\Vert _{\infty}}{\alpha(k-2)!}\left\Vert -t\left\Vert x-y\right\Vert \left(\max\{\left\Vert x\right\Vert ,\left\Vert y\right\Vert \}\right)^{k-2}\right\Vert _{\alpha}
\end{align*}
by Lemma \ref{lem:dist-gradients-from-kth-derivative}. Using the
same argument as in point (1) to bound $\left(\max\{\left\Vert x\right\Vert ,\left\Vert y\right\Vert \}\right)^{k-2}$,
we obtain

\begin{align*}
 & \frac{k(k-1)\left\Vert H\right\Vert _{\infty}+\frac{\kappa}{(k-2)!}r^{m-k}}{\alpha}\left\Vert -t\left\Vert x-y\right\Vert (s_{0}(k-2)(-t))^{-1}\right\Vert _{\alpha}\\
 & \leq\frac{k(k-1)\left\Vert H\right\Vert _{\infty}+\frac{\kappa}{(k-2)!}r^{m-k}}{\alpha s_{0}(k-2)}(1+\epsilon)\left\Vert x-y\right\Vert _{\alpha}
\end{align*}
To obtain a contraction, we need to choose $\alpha>\frac{k(k-1)\left\Vert H\right\Vert _{\infty}+\frac{\kappa}{(k-2)!}r^{m-k}}{s_{0}(k-2)}(1+\epsilon)$.
To summarize $\alpha$ must lie within the two bounds:
\[
\frac{k(k-1)\left\Vert H\right\Vert _{\infty}+\frac{\kappa}{(k-2)!}r^{m-k}}{(m-k+1)s}(1-\epsilon)\frac{m-k+1}{k-2}<\alpha<\frac{m-k+1}{k-2}
\]
which is possible since we have chosen $r,c$ and $T$ such that $\frac{k(k-1)\left\Vert H\right\Vert _{\infty}+\frac{\kappa}{(k-2)!}r^{m-k}}{(m-k+1)s}(1-\epsilon)<1$.
\end{proof}

\subsection{Proof of Theorem \ref{thm:first-path}}\label{subsec:proof-first-path}

We have now all the tools to prove the Theorem 4 of the main:
\begin{thm}[Theorem \ref{thm:first-path} of the main text]
Assume that the largest singular value $s_{1}$
of the gradient of $C$ at the origin $\nabla C(0)\in\mathbb{R}^{n_{L}\times n_{0}}$
has multiplicity 1. There is a deterministic gradient flow path $\underline{\theta}^{1}$ in the space
of width-$1$ DLNs such that, with probability $1$ if $L\leq 3$, and probability at least $\nicefrac{1}{2}$ if $L>3$, there exists an escape time $t_{\alpha}^{1}$ and a rotation $R$ such that
\[
\lim_{\alpha\to0}\theta_{\alpha}(t_{\alpha}^{1}+t)=RI^{(1\to w)}\underline{\theta}^{1}(t).
\]
\end{thm}

\begin{proof}
From Proposition \ref{prop:lower-bound-rate-escape-path} we know
that with prob. 1 there is a time horizon $t_{\alpha}^{1}$ and an
escape path such that $\lim_{\alpha\to0}\theta_{\alpha}(t_{\alpha}^{1}+t)=\theta^{1}(t)$
which for any $\epsilon>0$ escapes the origin at a rate of at least
$e^{(s^{*}+\epsilon)t}$ for shallow networks and $\left[(k-2)(s^{*}+\epsilon)t\right]^{\frac{1}{2-k}}$
for deep networks, where $s^{*}=-L^{-\frac{L=2}{2}}s_{1}$.

Since the loss $C^{NN}$ is $(L,2L)$-approximately homogeneous, we
can apply Theorem \ref{thm:bijection-optimal-escape-paths-appendix} to obtain that $\theta^{1}$ must be in bijection
with an escape path of the homogeneous loss $H$ of the same speed.
For small enough $\epsilon$ the only escape path of $H$ of at least
this speed are of the form $\rho^{*}e^{s^{*}(t+T)}$ for shallow networks
and $\rho^{*}\left((k-1)s^{*}(-t-T)\right)^{-\frac{1}{k-1}}$ for
some constant $T$ and an optimal escape direction $\rho^{*}$. We
therefore call $\theta^{1}$ an optimal escape path since it belongs
to the unique set of paths which escape at an optimal speed and are
in bijection to the optimal escape directions.

Assuming that the largest singular value of $s_{1}$ of $\nabla C(0)$
has multiplicity 1, with singular vectors $u_{1},v_{1}$, the optimal
escape directions are of the form
\[
\rho^{*}=RI^{(1\to w)}(\underline{\rho}^{*})=\frac{1}{\sqrt{L}}RI^{(1\to w)}\left(-v_{1}^{T},1,\dots,1,u_{1}\right)
\]
for any rotation $R$. In the width $1$ network, there is an optimal
escape path $\underline{\theta}^{1}$ corresponding to the escape
direction $\underline{\rho}^{*}$, then by the unicity of the bijection
of Theorem \ref{thm:bijection-optimal-escape-paths-appendix}, the escape
path $RI^{(1\to w)}(\underline{\theta}^{1})$ is the unique optimal
escape path escaping along $RI^{(1\to w)}(\underline{\rho}^{*})$,
as a result, we know that $\theta^{1}=RI^{(1\to w)}(\underline{\theta}^{1})$
for some rotation $R$.
\end{proof}

\section{Technical Results}\label{sec:technical-results}

In this section, we state and prove a few technical lemmas used throughout
the appendix.

Let us first prove a generalization of Grönwall's inequality for polynomial
bounds:
\begin{lem}
\label{lem:Gronwall_inequality_polynomial_bounds}Let $x:\mathbb{R^{+}\to\mathbb{R}}$
which satisfy
\[
\partial_{t}x(t)\leq cx(t)^{\alpha},
\]
for some $c>0$ and $\alpha>1$. Then, for all $t<\frac{x(0)^{1-\alpha}}{c(\alpha-1)}$,
\[
x(t)\leq\left[x(0)^{1-\alpha}-c(\alpha-1)t\right]^{-\frac{1}{\alpha-1}}.
\]
\end{lem}

\begin{proof}
Note that the function $y(t)=\left[x(0)^{1-\alpha}+c(1-\alpha)t\right]^{-\frac{1}{\alpha-1}}$
satisfies $y(0)=x(0)$ and for all $t<\frac{x(0)^{1-\alpha}}{c(\alpha-1)}$:
\[
\partial_{t}y(t)=cy(t)^{\alpha}.
\]
We conclude by showing that if $x(t)\leq y(t)$ then $x(s)\leq y(s)$
for all $t\leq s\leq\frac{x(0)^{1-\alpha}}{c(\alpha-1)}$: this follows
from the fact that on the diagonal, i.e. when $x(t)=y(t),$ we have
\[
\partial_{t}x(t)-\partial_{t}y(t)\leq cx(t)^{\alpha}-cy(t)^{a}=0
\]
which implies that the flow points towards the inside of $\left\{ (x,y):x\leq y\right\} $.
\end{proof}
Let us now state a lemma to bound the gradient of a cost $C$ in terms
of its high order derivatives:
\begin{lem}
\label{lem:dist-gradients-from-kth-derivative} Let $C$ be a cost
and $k$ the largest integer such that $\partial_{x}^{n}C(0)=0$ for
all $n<k$ and $\left\Vert \partial_{x}^{k}C\right\Vert _{\infty}<\infty$,
then
\end{lem}

\begin{enumerate}
\item For all $x$, $\left\Vert \nabla C(x)\right\Vert \leq\frac{\left\Vert \partial_{x}^{k}C\right\Vert _{\infty}\left\Vert x\right\Vert ^{k-1}}{(k-1)!}$.
\item For all $x,y$, $\left\Vert \nabla C(x)-\nabla C(y)\right\Vert \leq\frac{1}{(k-2)!}\left\Vert \partial_{x}^{k}C\right\Vert _{\infty}\left\Vert x-y\right\Vert \left(\max\{\left\Vert x\right\Vert ,\left\Vert y\right\Vert \}\right)^{k-2}.$
\end{enumerate}
\begin{proof}
(1)
\begin{align*}
\left\Vert \nabla C(x)\right\Vert  & =\left\Vert \int_{0}^{1}\mathcal{H}C(\lambda x)[x]d\lambda\right\Vert \\
 & =\left\Vert \int_{0}^{1}\int_{0}^{\lambda_{1}}\cdots\int_{0}^{\lambda_{k-2}}\partial_{x}^{k}C(\lambda_{1}\cdots\lambda_{k-1}z_{t})[x,\dots,x]dt_{1}\cdots dt_{1}\right\Vert \\
 & \leq\int_{0}^{1}\int_{0}^{\lambda_{1}}\cdots\int_{0}^{\lambda_{m-2}}\left\Vert \partial_{x}^{k}C\right\Vert _{\infty}\left\Vert x\right\Vert ^{k-1}dt_{1}\cdots dt_{1}\\
 & \leq\frac{\left\Vert \partial_{x}^{k}C\right\Vert _{\infty}\left\Vert x\right\Vert ^{k-1}}{(k-1)!}
\end{align*}

(2) First note that $\nabla C(x)-\nabla C(y)$ is equal to
\[
\int_{0}^{1}\mathcal{H}C(z_{t})[x-y]dt
\]
where $z_{t}=tx+(1-t)y$. This can further be rewritten as
\[
\int_{0}^{1}\int_{0}^{1}\partial_{x}^{3}C(t_{1}z_{t,t_{1}})[x-y,z_{t}]dt_{1}\,dt.
\]
Iterating this procedure, we obtain that $\nabla C(x)-\nabla C(y)$
equals
\[
\int_{0}^{1}\int_{0}^{1}\int_{0}^{t_{1}}\cdots\int_{0}^{t_{k-3}}\partial_{x}^{k}C(t_{1}\cdots t_{k-2}z_{t,t_{1}})[x-y,z_{t},\dots,z_{t}]dt_{k-2}\cdots dt_{2}\,dt_{1}\,dt.
\]

Since the volume of the set $\{(t_{1},\dots,t_{k-2}):0\geq t_{1}\geq\dots\geq t_{k-2}\geq0\}$
is $\frac{1}{(k-2)!}$ we have
\begin{align*}
\left\Vert \nabla C(x)-\nabla C(y)\right\Vert  & \leq\int_{0}^{1}\int_{0}^{1}\int_{0}^{t_{1}}\cdots\int_{0}^{t_{k-3}}\left\Vert \partial_{x}^{k}C\right\Vert \left\Vert x-y\right\Vert \left\Vert z_{t}\right\Vert ^{k-2}dt_{k-2}\cdots dt_{2}\,dt_{1}\,dt\\
 & \leq\frac{1}{(k-2)!}\left\Vert \partial_{x}^{k}C\right\Vert _{\infty}\left\Vert x-y\right\Vert \left(\max\{\left\Vert x\right\Vert ,\left\Vert y\right\Vert \}\right)^{k-2}.
\end{align*}
\end{proof}

\end{document}